\documentclass[12pt]{article}
\usepackage[utf8]{inputenc}

\usepackage{epsfig}
\usepackage{graphicx,epstopdf,epsfig}
\usepackage{subfigure}
\usepackage{amsmath,amsthm,verbatim,amssymb}
\usepackage{verbatim,color,amsmath}
\usepackage{lscape}
\usepackage{threeparttable}
\usepackage{algorithm, algorithmic}
\usepackage{float}
\usepackage{algorithmic}
\usepackage{mathrsfs}
\usepackage{booktabs}
\usepackage{threeparttable}
\usepackage{multirow}
\usepackage{hyperref}
\hypersetup{colorlinks=true,
            linkcolor=blue,
            anchorcolor=blue,
            citecolor=blue}
\usepackage{enumerate}
\usepackage[round]{natbib}

\makeatletter 
\date{}
\usepackage{booktabs}
\usepackage{tabularx}
\usepackage{array}

\newtheorem{theorem}{Theorem}[section]
\newtheorem{lemma}{Lemma}[section]
\newtheorem{lemma*}{Lemma}
\newtheorem{remark}{Remark}[section]
\newtheorem{remark*}{Remark}
\newtheorem{corollary}{Corollary}[section]
\newtheorem{proposition}{Proposition}[section]
\newtheorem{definition}{Definition}[section]
\newtheorem{assumption}{Assumption}[section]

\DeclareMathOperator*{\argmin}{argmin}
\DeclareMathOperator*{\argmax}{argmax}

\def\m0{\mathbf{0}}

\def \cD {{\cal D}}

\def \cU {{\cal U}}

\begin{document}
\title{Schr\"odinger--F\"ollmer Actor--Critic: Diffusion Policy Improvement with Finite-Sample Analysis}

\author{
Yuling Jiao  
\thanks{
School of Artificial Intelligence,
and Hubei Key Laboratory of Computational Science, Wuhan University, Wuhan, 430072, China.
Email: yulingjiaomath@whu.edu.cn
}
\and
Lican Kang 
\thanks{
Institute for Math and AI,  Hubei Key Laboratory of Computational Science, and School of Artificial Intelligence, Wuhan University, Wuhan, 430072, China. 
Email: kanglican@whu.edu.cn
}
\and
Jerry Zhijian Yang
\thanks{
School of Mathematics and Statistics, and Hubei Key Laboratory of Computational Science, Wuhan University, Wuhan, 430072,  China.
Email: zjyang.math@whu.edu.cn
}
\and
Jincheng Ying
\thanks{
School of Mathematics and Statistics,
 Wuhan University, Wuhan, 430072, China.
Email: jinchengying@whu.edu.cn
}
}

\maketitle
\begin{abstract}
Diffusion policies represent multimodal action distributions, but a
$Q$-weighted update does not specify how to sample from the resulting
target distribution. We propose
Schr\"odinger--F\"ollmer Actor--Critic (SFAC), an offline-to-online reinforcement
learning (RL) method for Kullback--Leibler (KL)-regularized policy improvement
through conditional diffusion. A minimax Bellman critic estimates the
action-value function, which defines an exponentially tilted target policy.
A Doob $h$-transform expresses this update as a correction to the reference
diffusion drift. The target drift is a weighted conditional mean of Brownian-bridge
labels. We fit it by weighted regression using endpoints from the current
policy and pointwise action-value evaluations. In the small-update regime, the
KL-regularized update follows the natural policy-gradient direction, and the
Doob correction represents the same local change in the space of diffusion drifts.
Under suitable conditions, we derive finite-sample bounds that separate the
errors from critic estimation, weighted neural drift regression,
and diffusion discretization through a reverse-KL actor bound.
The average-performance bound depends explicitly on the comparison
coefficients over the policy iterations. Across six D4RL
offline-to-online tasks, the final scores of the $Q$-guided
bridge-regression implementation exceed its offline initialization on
every task and Cal-QL on five tasks. We also compare with offline RL
and online-only baselines.
\end{abstract}
\par
\section{Introduction}
\label{sec:introduction}

Offline RL learns a policy from previously collected transitions without
further interaction with the environment. The accuracy of its value estimates
depends on the state--action coverage of these data
\citep{levine2020offline,kumar2020conservative,kostrikov2022offline}.
Offline-to-online (O2O) RL combines offline data or pretrained policies
with subsequent online learning
\citep{nair2020awac,lee2022balanced,nakamoto2023calql,
ball2023efficient,luo2023finetuning}.
This transition requires adapting to new experience while retaining useful
information from the offline data. In the early online stages, value
estimates remain influenced by offline coverage: favoring poorly covered
actions can amplify estimation error, whereas restrictive policy updates
can limit adaptation. Policy representation introduces a further challenge:
offline datasets may combine several behavior policies and contain multiple
plausible actions at the same state. Diffusion actors accommodate the
resulting multimodal conditional action distributions
\citep{wang2022diffusion,chen2022offline,hansen2023idql}.
They define policies through a generative process without requiring an
explicit conditional density. Policy improvement therefore requires
modifying this process to sample from the action distribution prescribed
by estimated value information.
To formulate this problem, we follow advantage-weighted and product-policy
approaches to policy improvement
\citep{peng2019advantage,nair2020awac,frans2025diffusion} and use
$Q^\pi(s,a)$ to evaluate an action under the current policy.
For a reference policy $\pi$ and regularization parameter $\lambda>0$,
the improved policy at state $s$ is defined by
\[
\bar\pi(\cdot\mid s)
\in
\argmax_{\pi'(\cdot\mid s)}
\left\{
\mathbb E_{a\sim\pi'(\cdot\mid s)}[Q^\pi(s,a)]
-\lambda\,\mathrm{KL}\!\left(
\pi'(\cdot\mid s)\,\middle\|\,\pi(\cdot\mid s)
\right)
\right\},
\]
whose solution is the exponential tilt
\begin{equation}
\label{eq:intro-policy-tilt}
\bar\pi(a\mid s)
=\frac{\pi(a\mid s)\exp\{Q^\pi(s,a)/\lambda\}}
{\int_{\mathcal A}\pi(a'\mid s)\exp\{Q^\pi(s,a')/\lambda\}\,\mathrm da'}.
\end{equation}
Shifting $Q^\pi(s,\cdot)$ by a state-dependent constant leaves the
normalized optimizer unchanged. We use $Q^\pi$ directly.
For an implicit diffusion actor, \eqref{eq:intro-policy-tilt} specifies the
target distribution but does not provide the diffusion dynamics needed to
generate it. Realizing this KL-regularized update therefore requires a
corresponding change in the diffusion drift. Existing generative-policy
methods incorporate value information through
critic action gradients, actor-parameter optimization, or reweighting of
sampled candidate actions
\citep{wang2022diffusion,hansen2023idql,
lu2023contrastive,park2025flow}.
A complementary approach is to estimate the drift implied by the prescribed
tilt from policy samples and pointwise action-value evaluations. Such a
construction would permit supervised actor fitting without differentiating
the critic with respect to actions. Its statistical accuracy, however,
depends on both the estimated action value and the approximation of the target
diffusion. Errors from critic estimation, drift learning, and numerical
sampling can propagate through successive policy updates. An exact tilting
identity alone does not control these effects. This motivates a joint
construction and analysis that connects the prescribed target, the learned
diffusion drift, and the resulting policy performance.

To address these issues, we propose Schr\"odinger--F\"ollmer Actor--Critic
(SFAC), an offline-to-online method for KL-regularized policy improvement
through conditional diffusion. SFAC represents the current policy as the
terminal law of a conditional Schr\"odinger--F\"ollmer diffusion
\citep{follmer1985,follmer1986,follmer1988}.
Reweighting this terminal law induces a change of measure on path space,
which a Doob $h$-transform realizes through a drift correction
\citep{doob1984classical,baudoin2002conditioned}. Here $h$ is the positive
conditional expectation, under the reference diffusion, of the terminal
reweighting factor given the current diffusion state.
Applying this construction to the $Q$ tilt gives a target diffusion
whose terminal law is the prescribed policy update. The learned actor and
its numerical sampler approximate this exact representation.
Specifically, at iteration $k$, let $\pi_k$ denote the current practical policy and
$\widehat Q_k$ the minimax Bellman critic's estimate of its action-value
function $Q^{\pi_k}$. The resulting plug-in target $\bar\pi_{k+1}$ satisfies
$\bar\pi_{k+1}(a\mid s)\propto
\pi_k(a\mid s)\exp\{\widehat Q_k(s,a)/\lambda\}$.
Here $\propto$ denotes equality up to a state-dependent normalizing factor.

The Doob transform expresses the target drift as a conditional
mean of Brownian-bridge labels under the $Q$-tilted endpoint law.
This identity gives a $Q$-weighted squared-loss regression objective.
We draw endpoints from the frozen current policy and sample bridge
positions to construct the regression pairs.
 Euler--Maruyama simulation of this actor defines the practical
policy $\pi_{k+1}$, which approximates $\bar\pi_{k+1}$. The actor update
therefore uses supervised drift regression without critic action gradients.
In the small-update regime, the ideal exponential tilt follows the natural
policy-gradient direction, and the Doob correction represents the same
first-order change in the space of diffusion drifts.
We then analyze how the approximation errors affect policy performance.
We bound action-value estimation error under dependent transition data
and use weighted bridge regression to control the actor's reverse-KL
error. The performance bound separates sampling discrepancy, neural
approximation, and time discretization.
We evaluate the $Q$-guided bridge-regression implementation on six
continuous-control tasks from Datasets for Deep Data-Driven Reinforcement
Learning (D4RL) \citep{fu2020d4rl}.
Its final online scores are compared with its offline initialization
and with offline RL, offline-to-online RL, and online-only baselines.

\subsection{Contributions}
\label{sec:contributions}

Our contributions are as follows.
\begin{itemize}
    \item
    \textbf{A conditional-diffusion formulation of policy improvement.}
    We formulate KL-regularized policy improvement for diffusion actors through
    a conditional Schr\"odinger--F\"ollmer process. The exponential $Q$
    tilt defines a terminal change of measure, and the corresponding Doob
    $h$-transform identifies the drift correction that realizes the improved
    policy. In the small-update regime, the ideal KL step follows the natural
    policy-gradient direction, while the Doob correction represents the same
    local change in the space of diffusion drifts.

    \item
    \textbf{Sample-based actor learning without critic action gradients.}
    The Doob transform gives the target drift as a weighted conditional
    mean of Brownian-bridge labels. This identity yields a regression
    objective using policy endpoints and sampled bridge noise, with
    weights computed from pointwise action-value evaluations.

    \item
    \textbf{Finite-sample guarantees.}
We bound the actor's reverse-KL error through weighted regression
and Euler--Maruyama discretization, using weak action derivatives of
the tilted density and the second moment of the bridge labels at each
time. Combined with the critic error bound, this gives an
average-performance guarantee relative to a covered comparison policy.
The bound retains the sampling, approximation, and coverage terms
over the policy iterations.
\end{itemize}

\subsection{Related Work}
\label{sec:related-work}

\paragraph{Offline and offline-to-online RL.}
Offline RL addresses distribution shift through conservative action-value
estimation, behavior-regularized policy learning, or implicit policy
extraction \citep{kumar2020conservative,fujimoto2021minimalist,
kostrikov2022offline}. The transition to online learning also requires preserving useful
offline behavior while adapting to new experience. Advantage-weighted
actor--critic (AWAC) uses weighted maximum-likelihood actor updates,
calibrated Q-learning (Cal-QL) controls the scale of conservative action-value
estimates, and policy expansion (PEX) combines a frozen offline policy with
a learnable policy
\citep{nair2020awac,nakamoto2023calql,zhang2023policy}.
RL with prior data (RLPD) instead incorporates offline transitions directly
into online off-policy training, without requiring offline pretraining
\citep{ball2023efficient}.
Diffusion models can also augment the training distribution: synthetic
experience replay (SynthER) generates additional replay data, whereas
energy-guided diffusion sampling (EDIS) guides synthetic transitions toward
the current online distribution
\citep{lu2023synthetic,liu2024energy}.
These methods generate transitions; SFAC instead updates conditional
action distributions. Fitted policy iteration and minimax Bellman
methods relate estimation error to policy performance under approximation
and coverage conditions \citep{antos2008learning,xie2020q}.
Deep approximate policy iteration further combines neural approximation,
Bellman-residual estimation under dependent data, and error propagation
\citep{jiao2025deep}. Minimax off-policy evaluation separately studies the
statistical estimation of policy returns
\citep{uehara2020minimax,uehara2021finite}, while policy-finetuning theory
quantifies the role of reference-policy coverage in sample-efficient online
adaptation \citep{xie2021policyfinetuning}.
SFAC addresses the additional actor-learning and numerical-sampling errors
introduced by conditional diffusion policy updates.

\paragraph{Generative policies.}
Denoising and score-based diffusion models provide flexible distributional
representations
\citep{ho2020denoising,song2019generative,song2021scorebased}.
Diffuser and Decision Diffuser generate trajectories for planning, whereas
Diffusion Policy generates observation-conditioned action sequences
\citep{janner2022planning,ajay2023conditional,chi2023diffusion}.
For action-level policy improvement, Diffusion-QL combines denoising
regression with action-value maximization \citep{wang2022diffusion}.
Implicit diffusion Q-learning (IDQL) reweights candidates from a diffusion
behavior model and also considers online fine-tuning \citep{hansen2023idql}.
Contrastive energy prediction (CEP) learns intermediate energy guidance
for Q-guided policy optimization (QGPO), with population consistency under
unlimited model capacity and data \citep{lu2023contrastive}.
Diffusion actor--critic (DAC) expresses behavior-constrained policy
improvement as noise regression with a critic-action-gradient correction
\citep{fang2025diffusion}. Classifier-free guidance for RL (CFGRL) instead
trains an optimality-conditioned generative model and adjusts its guidance
strength at sampling time \citep{frans2025diffusion}.
Flow Q-Learning trains a one-step actor using a flow behavior model and
includes O2O evaluation \citep{park2025flow}; diffusion policy policy
optimization (DPPO) fine-tunes diffusion policies by applying policy
optimization to coupled denoising and environment dynamics
\citep{ren2024dppo}.
These methods differ in how action-value information enters the policy
update. SFAC uses the Brownian-bridge conditional mean to derive a $Q$-weighted
regression objective for the Doob drift. The same endpoint weights
define both the target policy and the loss.

\paragraph{Regularized updates and controlled diffusion.}
Relative-entropy policy search and KL-constrained policy optimization
motivate exponential reweighting relative to a reference distribution
\citep{peters2010relative,abdolmaleki2018maximum,peng2019advantage}.
Maximum-entropy RL promotes policy entropy, while regularized Bellman
theory provides a broader framework for penalty-based updates and their
error propagation \citep{haarnoja2018soft,geist2019regularized}.
On path space, F\"ollmer's entropy construction and the Doob $h$-transform
connect terminal reweighting to a change in diffusion drift
\citep{follmer1985,follmer1986,follmer1988,doob1984classical}.
Schr\"odinger--F\"ollmer samplers realize a target law at a finite terminal
time, without an ergodic limit; their nonasymptotic analyses allow
non-log-concave targets under suitable regularity conditions
\citep{huang2021schrodinger,jiao2021convergence}.
Diffusion Schr\"odinger bridges address the related problem of
entropy-regularized transport between endpoint distributions
\citep{debortoli2021diffusion}.
For pretrained diffusion models, entropy-regularized control yields
reward-based fine-tuning objectives \citep{uehara2024finetuning}, and
feedback-efficient methods combine reward learning, exploration, and
generative updates with regret guarantees \citep{uehara2024feedback}.
Recent theory establishes regularity preservation and convergence of
control-policy iteration \citep{han2025stochastic}, regression-based
statistical guarantees for diffusion fine-tuning \citep{mou2025regression},
and nonasymptotic guidance and distributional bounds for learned Doob
$h$-transforms \citep{chang2026doob}.
Finite-particle approximations can nevertheless bias the intended tilt,
even in tractable Gaussian examples \citep{dandapanthula2026tilting}.
An exact change-of-measure identity therefore does not establish the
accuracy of an estimated and discretized actor. SFAC complements these
results by propagating Bellman-critic, weighted drift-regression,
and discretization errors through successive policy updates. The resulting
bound controls average performance relative to a covered comparison
policy and depends explicitly on the comparison coefficients and
regression errors.
Table~\ref{tab:generative-policy-positioning} compares representative
update mechanisms. For a critic action-value function $Q(s,a)$,
$\nabla_a Q$ denotes its gradient with respect to the action $a$.
The columns $\nabla_a Q$, Refit, and O2O indicate the use of this gradient,
actor refitting after policy updates, and an explicit offline-to-online
study, respectively; $\checkmark$ and $\times$ denote yes and no.
The Refit column excludes initial behavior-model training.
IDQL considers both frozen and updated behavior models during online
fine-tuning; CFGRL can vary guidance strength without refitting its
conditional model. An O2O $\times$ does not rule out an extension to that
setting. The gradient column refers only to critic action gradients;
QGPO, for example, differentiates a learned energy model.

\begin{table}[H]
    \centering
  \caption{Comparison of representative generative-policy update mechanisms.}
    \label{tab:generative-policy-positioning}
    \scriptsize
    \setlength{\tabcolsep}{4pt}
    \renewcommand{\arraystretch}{1.12}
    \begin{tabularx}{\textwidth}{
        @{}
        l
        >{\raggedright\arraybackslash}p{3.2cm}
        ccc
        >{\raggedright\arraybackslash}X
        @{}
    }
        \toprule
        Method
        & Update object
        & $\nabla_a Q$
        & Refit
        & O2O
      & Characterization / analysis
        \\
        \midrule

        Diffusion-QL \citep{wang2022diffusion}
        & Diffusion actor parameters
        & $\checkmark$
        & $\checkmark$
        & $\times$
        & Denoising and action-value objective
        \\

        IDQL \citep{hansen2023idql}
        & Candidate reweighting
        & $\times$
        & Optional
        & $\checkmark$
        & Implicit-policy characterization
        \\

        CEP/QGPO \citep{lu2023contrastive}
        & Intermediate energy guidance
        & $\times$
        & $\times$
        & $\times$
        & Population guidance consistency
        \\

        DAC \citep{fang2025diffusion}
        & Q-guided noise predictor
        & $\checkmark$
        & $\checkmark$
        & $\times$
        & KL/noise-regression correspondence
        \\

        CFGRL \citep{frans2025diffusion}
        & Guided sampling distribution
        & $\times$
        & $\times$
        & $\times$
        & Product-policy improvement
        \\

        \textbf{SFAC (ours)}
        & Weighted bridge-drift regression
        & $\times$
        & $\checkmark$
        & $\checkmark$
        & Critic, bridge-regression, and discretization errors in average performance
        \\
        \bottomrule
    \end{tabularx}

\end{table}

\subsection{Outline}\label{sec:po}
Section~\ref{sec:preliminaries} introduces the notation and reviews the relevant RL and
Schr\"odinger--F\"ollmer preliminaries.
Section~\ref{sec:method} develops action-value estimation with a minimax
Bellman critic and conditional Schr\"odinger--F\"ollmer policy improvement.
Section~\ref{sec:theory} analyzes action-value estimation, actor approximation,
and their propagation through policy iteration.
Section~\ref{sec:proof sketch} outlines the main proofs, and
Section~\ref{sec:experiment} presents the numerical experiments.
Section~\ref{sec:conclusion} concludes. Complete proofs and auxiliary
results appear in the appendix.

\section{Preliminaries}\label{sec:preliminaries}

This section introduces the notation, RL framework, and function classes
and reviews the Schr\"odinger--F\"ollmer representation underlying the
diffusion actor.

\subsection{Notation}
\label{sec:notations-organization}
Let $\mathbb N_0$ and $\mathbb N$ denote the nonnegative and positive
integers, respectively. For a vector $x=(x_1,\ldots,x_d)^\top$ and
$q\in[1,\infty)$, write
$
\|x\|_q=\left(\sum_{i=1}^d|x_i|^q\right)^{1/q}.
$
The quantity $\|x\|_0$ counts the nonzero entries of $x$.
For a subset $E$ of a Euclidean space, let $\mathcal B(E)$ be its Borel
$\sigma$-algebra and $\mathcal P(E)$ the set of probability measures on
$(E,\mathcal B(E))$. For a measurable function $f$ and a probability
measure $\mu$ on its domain, define
$
\|f\|_{L^q(\mu)}
=\left(\int |f(x)|^q\,\mathrm d\mu(x)\right)^{1/q}.
$ 
We use $\|f\|_\infty$ for the supremum norm on the stated domain and
$\|f\|_{L^\infty(\mu)}$ for the $\mu$-essential supremum norm.
The notation $\mathrm{Law}(X)$ denotes the distribution of a random
variable $X$, and $I_d$ is the $d\times d$ identity matrix. We write
$\gamma_{d_{\mathcal A}}=\mathcal N(0,I_{d_{\mathcal A}})$ for the standard
Gaussian measure on $\mathbb R^{d_{\mathcal A}}$, where
$d_{\mathcal A}$ is the action dimension specified below.
Finally, $a\lesssim b$ means that $a\le Cb$ for a constant $C>0$,
and $\widetilde O$ suppresses logarithmic factors.

\subsection{Markov Decision Process}
A discounted Markov decision process (MDP) is a tuple
 $(\mathcal X,\mathcal A,P,\mathcal R,\gamma)$,
where
$\mathcal X\subseteq\mathbb R^{d_{\mathcal X}}$ and
$\mathcal A\subseteq\mathbb R^{d_{\mathcal A}}$ are the state and action
spaces, respectively. We write $d:=d_{\mathcal X}+d_{\mathcal A}$.
For each state-action pair $(s,a)\in\mathcal X\times\mathcal A$,
$P(\cdot\mid s,a)$ is a probability measure on
$(\mathcal X,\mathcal B(\mathcal X))$, and
$(s,a)\mapsto P(D\mid s,a)$ is measurable for every
$D\in\mathcal B(\mathcal X)$. The reward kernel
$\mathcal R(\cdot\mid s,a)$ specifies the conditional distribution of the
immediate reward, and $\gamma\in[0,1)$ is the discount factor.

Let $\pi(\cdot\mid s)$ be a stochastic policy, that is, a probability
kernel from states to actions, and let $P_1$ denote the initial-state
distribution. A trajectory under $\pi$ is generated by
\[
X_1\sim P_1,\qquad
A_i\sim\pi(\cdot\mid X_i),\qquad
R_i\sim\mathcal R(\cdot\mid X_i,A_i),\qquad
X_{i+1}\sim P(\cdot\mid X_i,A_i).
\]
The expected one-step reward is
\[
r(s,a)
:=
\int_{\mathbb R} z\,\mathcal R(\mathrm dz\mid s,a).
\]
We write $\mathbb P^\pi$ for the trajectory law induced by
$P_1$, $\pi$, $P$, and $\mathcal R$, and $\mathbb E^\pi$ for expectation
under this law. The conditional expectation below refers to a trajectory
starting from $(s,a)$, with subsequent actions sampled from $\pi$.
The action-value function is
\[
Q^\pi(s,a)
:=
\mathbb E^\pi
\left[
    \sum_{i=1}^{\infty}\gamma^{i-1}R_i
    \,\middle|\,
    X_1=s,\ A_1=a
\right].
\]
The policy-averaged state transition kernel is
\[
P^\pi(D\mid s)
:=
\int_{\mathcal A}
P(D\mid s,a)\,\pi(\mathrm da\mid s),
\qquad
D\in\mathcal B(\mathcal X).
\]
The induced state--action transition operator is
\[
  P^\pi Q(s,a)
:=
\int_{\mathcal X}\int_{\mathcal A}
Q(s',a')\,\pi(\mathrm da'\mid s')\,P(\mathrm ds'\mid s,a).
\]
We assume that rewards lie in $[0,R_{\max}]$. Consequently, $Q^\pi$
takes values in $[0,R_{\max}/(1-\gamma)]$.

For an action-value function $Q$, the policy-specific Bellman operator is
\begin{equation}
\label{eq:policy-bellman-operators}
\mathcal T^\pi Q(s,a):=r(s,a)+\gamma P^\pi Q(s,a).
\end{equation}
Conditioning on the first transition gives the Bellman equation 
\begin{equation}
\label{eq:policy-bellman-fixed-points}
Q^\pi=\mathcal T^\pi Q^\pi.
\end{equation}
In particular,
\[
Q^\pi(s,a)=r(s,a)+\gamma
\int_{\mathcal X}\int_{\mathcal A}
Q^\pi(s',a')\,\pi(\mathrm da'\mid s')\,P(\mathrm ds'\mid s,a).
\]
The identity in \eqref{eq:policy-bellman-fixed-points} underlies
the critic objective in Section~\ref{sec:mabo-action-value-learning}.
For a state--action joint distribution $\mu$, the distribution after one
transition under $\pi$ is
\begin{equation*}
(\mu P^\pi)(B)
:=
\int_{\mathcal X\times\mathcal A}
\int_{\mathcal X\times\mathcal A}
\mathbf 1_B(s',a')
P(\mathrm ds'\mid s,a)
\pi(\mathrm da'\mid s')
\mu(\mathrm ds,\mathrm da).
\end{equation*}
Here, $B\in\mathcal B(\mathcal X\times\mathcal A)$ and $\mathbf 1_B$
is its indicator function.
The normalized discounted state occupancy measure of $\pi$ is
\[
d^\pi(S)
:=
(1-\gamma)
\sum_{t=1}^{\infty}
\gamma^{t-1}
\mathbb P^\pi(X_t\in S),
\qquad
S\in\mathcal B(\mathcal X).
\]
We measure the  policy performance by
\[
J(\pi)
:=
\mathbb E^\pi_{X_1\sim P_1}
\left[
    \sum_{t=1}^{\infty}\gamma^{t-1}R_t
\right],
\]
and an optimal policy, when one exists, is denoted by
$\pi^\star\in\argmax_\pi J(\pi)$.
In the performance analysis, this symbol can also denote a fixed
comparison policy. The bound gives an optimality guarantee when an
optimal policy satisfies the comparison conditions.

We describe the  temporal dependence by  using $\beta$-mixing
\citep{yu1994rates,antos2008learning}.
\begin{definition}[$\beta$-mixing]
Let $\{W_t\}_{t\ge 1}$ be a stochastic process.
For $1\le i\le j\le\infty$, write $W^{i:j}=(W_i,\ldots,W_j)$,
interpreted as the infinite tail when $j=\infty$, and let
$\sigma(W^{i:j})$ be the $\sigma$-algebra it generates. The
$\beta$-mixing coefficient at lag $m\in\mathbb N_0$ is
\[
\beta_m = \sup_{t\ge 1} \mathbb{E}\left[\,
\sup_{B\in\sigma(W^{t+m:\infty})}
\left|
\mathbb P(B\mid W^{1:t})
-
\mathbb P(B)
\right|\,
\right].
\]
The process is \emph{$\beta$-mixing} if $\beta_m\to0$ as $m\to\infty$.
It mixes at an exponential rate with parameters
$\bar\beta,b,\eta>0$ if
$\beta_m\le\bar\beta\exp(-bm^\eta)$ for all $m\in\mathbb N_0$.
\end{definition}
For any positive integer block length $a_n$ satisfying $2a_n\le n$, set
$\xi_n:=\lfloor n/(2a_n)\rfloor$; the resulting $2\xi_n$ blocks are used in
the independent-block argument for critic estimation.

Admissible distributions are state--action laws reached from the
initial-state distribution under a sequence of policies
\citep{munos2003variable,antos2008learning,xie2020minimax}.
\begin{definition}[Admissible distributions]\label{def:admissible distribution}
A probability distribution
$\nu\in\mathcal P(\mathcal X\times\mathcal A)$ is admissible if there exist
$h\in\mathbb N$ and a sequence of policies from the reference, target, practical,
and comparator families considered in this paper,
$\boldsymbol\pi=(\pi_1,\ldots,\pi_h)$, such that
\[
\nu(B)
=
\mathbb P
\left(
    (X_h,A_h)\in B
    \,\middle|\,
    X_1\sim P_1,\;
    A_t\sim\pi_t(\cdot\mid X_t),\;
    X_{t+1}\sim P(\cdot\mid X_t,A_t)
\right)
\]
for every
$B\in\mathcal B(\mathcal X\times\mathcal A)$.
The class $\mathfrak M$ contains these distributions and their probability
mixtures. 
\end{definition}

\begin{definition}[KL divergence and $2$-Wasserstein distance]
Let $\mu$ and $\nu$ be probability measures on a measurable space. If
$\mu$ is absolutely continuous with respect to $\nu$, written
$\mu\ll\nu$, their KL divergence is
\[
\mathrm{KL}(\mu\|\nu)
:=
\int \log\!\left(\frac{\mathrm d\mu}{\mathrm d\nu}\right)\mathrm d\mu;
\]
otherwise, $\mathrm{KL}(\mu\|\nu):=+\infty$.
Let $\mathcal P_2(\mathbb R^d)$ denote the probability measures on
$\mathbb R^d$ with finite second moments. For
$\mu,\nu\in\mathcal P_2(\mathbb R^d)$, their $2$-Wasserstein distance is
\[
W_2(\mu,\nu)
:=
\left(
\inf_{\gamma\in\Pi(\mu,\nu)}
\int_{\mathbb R^d\times\mathbb R^d}
\|x-y\|_2^2\,\mathrm d\gamma(x,y)
\right)^{1/2},
\]
where $\Pi(\mu,\nu)$ denotes the set of couplings of $\mu$ and $\nu$.
\end{definition}
\subsection{Deep Neural Networks and H\"older Classes}

\paragraph{ReLU networks.}
Let $\mathcal F_{\mathcal D,\mathcal W,\mathcal S,B_{\rm net}}$ denote
the class of scalar-valued deep neural networks (DNNs) with rectified
linear unit (ReLU) activation, $\mathcal D\ge1$ hidden layers, hidden-layer
width at most $\mathcal W$, at most $\mathcal S$ nonzero weights and
biases, and output bound $B_{\rm net}>0$. Each network has the form
\[
f_\theta
=
\mathcal L_{\mathcal D}
\circ\sigma
\circ\mathcal L_{\mathcal D-1}
\circ\cdots
\circ\sigma
\circ\mathcal L_0,
\]
where the affine maps and activation function are
\[
\mathcal L_i(x)=W_i x+b_i,
\qquad
\sigma(x)=\max\{x,0\}
\]
and $\sigma$ is applied componentwise. Set $d_0=d_{\rm in}$ (input dimension),
$d_{\mathcal D+1}=d_{\rm out}$ (output dimension), and $1\le d_i\le\mathcal W$ for
$i=1,\ldots,\mathcal D$. Here
$W_i\in\mathbb R^{d_{i+1}\times d_i}$ and
$b_i\in\mathbb R^{d_{i+1}}$, and $\theta$ collects all weights and biases.
The class constraints are
\[
\sum_{i=0}^{\mathcal D}
\left(
    \|W_i\|_0+\|b_i\|_0
\right)\le\mathcal S,
\qquad
\|f_\theta\|_\infty\le B_{\rm net},
\qquad \|\theta\|_\infty\le B_{\rm par},
\]
where $\|W_i\|_0$ counts the nonzero entries of $W_i$.
The fixed output bound $B_{\rm net}$ is distinct from the parameter
bound $B_{\rm par}\ge1$. 
%In the rate results, $B_{\rm par}$ and
%$\mathcal W$ grow at most polynomially with the relevant training sample
%size and are large enough to include the approximating networks.
The total number of weights and biases satisfies
$
\sum_{i=0}^{\mathcal D}d_{i+1}(d_i+1)
\lesssim
\mathcal W^2\mathcal D.
$
 
\paragraph{H\"older classes.}
For a smoothness parameter $\zeta>0$, write
\[
\zeta=r_\zeta+\alpha,
\qquad
r_\zeta\in\mathbb N_0,
\qquad
\alpha\in(0,1].
\]
For a multi-index $\iota=(\iota_1,\ldots,\iota_d)\in\mathbb N_0^d$,
write $|\iota|=\sum_{j=1}^d\iota_j$ and
$D^\iota f=\partial^{|\iota|}f/
(\partial x_1^{\iota_1}\cdots\partial x_d^{\iota_d})$, with $D^0f=f$.
For $f:\mathbb R^d\to\mathbb R$ with continuous partial derivatives
up to order $r_\zeta$, define
\[
\|f\|_{\mathcal H^\zeta(\mathbb R^d)}
:=
\max_{|\iota|\le r_\zeta}
\|D^\iota f\|_{L^\infty(\mathbb R^d)}
+
\max_{|\iota|=r_\zeta}
\sup_{\substack{z,z'\in\mathbb R^d\\z\ne z'}}
\frac{
    |D^\iota f(z)-D^\iota f(z')|
}{
    \|z-z'\|_2^\alpha
}.
\]
The H\"older ball of radius $B_{\rm H}>0$ is
\[
\mathcal H^\zeta(\mathbb R^d,B_{\rm H})
:=
\left\{
f:\mathbb R^d\to\mathbb R:
\|f\|_{\mathcal H^\zeta(\mathbb R^d)}
\le B_{\rm H}
\right\}.
\]
In Section~\ref{sec:method}, the critic population classes
$\mathcal U_1,\mathcal U_2$ are H\"older classes.
The ReLU networks classes $\mathcal G_1,\mathcal G_2$ are used for critic
estimation, and $\mathcal G_3$ is used  for drift regression.

\subsection{Schr\"odinger--F\"ollmer Process and Doob 
$h$-Transforms}
\label{sec:prelim-sf-process}

The Schr\"odinger--F\"ollmer (SF) process
\citep{follmer1985,follmer1986,follmer1988} provides a finite-time diffusion
representation of a target distribution. Let $\mu$ be a probability
measure on $\mathbb R^{d_{\mathcal A}}$ that is absolutely continuous
with respect to $\gamma_{d_{\mathcal A}}$, and define its density ratio by
\[
f(x)
:=
\frac{\mathrm d\mu}
     {\mathrm d\gamma_{d_{\mathcal A}}}(x).
\]
For $t\geq 0$, let $Q_t$ denote the Gaussian heat semigroup
\begin{equation*}
    Q_t f(x)
    :=
    \mathbb E_{Z\sim\mathcal N(0, I_{d_{\mathcal{A}}})}
    \left[
        f\left(x+\sqrt{t}\,Z\right)
    \right].
\end{equation*}
The associated F\"ollmer drift is defined by
\begin{equation*}
    b_f(x,t)
    :=
    \nabla_x\log Q_{1-t}f(x),
    \qquad
    t\in[0,1).
\end{equation*}
Here $\nabla_x$ denotes the gradient with respect to $x$.
The SF process satisfies the stochastic differential
equation (SDE)
\begin{equation}
\label{eq:sf-sde-preliminary}
    \mathrm dX_t
    =
    b_f(X_t,t)\,\mathrm dt+\mathrm dB_t,
    \qquad
    X_0=0,
    \qquad
    t\in[0,1),
\end{equation}
where $\{B_t\}_{t\in[0,1]}$ is a standard $d_{\mathcal{A}}$-dimensional Brownian
motion. Under suitable regularity conditions on $f$, the process extends
to $t=1$ and has the target terminal law 
$
    X_1\sim\mu.
$
Thus, the process transports the point mass $\delta_0$ at the origin to
$\mu$ over a finite time horizon.
The drift also admits a conditional-mean representation. 
Conditional on $X_1=x$, the process is a Brownian bridge
from $0$ to $x$, with marginal distribution
\begin{equation*}
    X_t\mid X_1=x
    \sim
    \mathcal N\left(
        t x,\,
        t(1-t) I_{d_{\mathcal{A}}}
    \right).
\end{equation*}
Whenever the conditional first moment is finite, the drift satisfies,
for $0<t<1$,
\begin{equation*}
    b_f(x,t)
    =
    \frac{
        \mathbb E[X_1 \mid X_t=x]-x
    }{
        1-t
    }.
\end{equation*}
This conditional-mean representation enables sample-based drift estimation
from terminal actions.

\paragraph{Doob $h$-transform of the SF process.}
Let $\mathcal{R}:\mathbb R^{d_{\mathcal A}}\to\mathbb R$ be a terminal reward,
let $\lambda>0$, and write $g=f\exp(\mathcal{R}/\lambda)$.  Define the normalizing constant and the
exponentially tilted target by
\[
 C:=\int_{\mathbb R^{d_{\mathcal A}}}
       \exp \left(\frac{\mathcal{R}(x)}{\lambda}\right)\mu(\mathrm d x)
   =Q_1g(0)\in(0,\infty),
 \qquad
 \mu_R(\mathrm d x):=
 \frac{\exp(\mathcal{R} (x)/\lambda)}{C}\,\mu(\mathrm dx).
\]
Let $\mathbb P$ be the path law of the SF process in
\eqref{eq:sf-sde-preliminary}. Its Doob $h$-transform
\citep{doob1984classical,baudoin2002conditioned}  is determined by
\[
 h(x,t)
 :=\mathbb E_{\mathbb P}
       \left[\left.\exp \left(\frac{\mathcal{R}(X_1)}{\lambda}\right)
       \right|X_t=x\right]
 =\frac{Q_{1-t}g(x)}{Q_{1-t}f(x)},
 \qquad 0\le t<1.
\]
In particular, $h(0,0)=C$, and $h(x,1)=\exp(\mathcal{R} (x)/\lambda)$.
Define the transformed path law by
\[
 \frac{\mathrm d\mathbb P^R}{\mathrm d\mathbb P}
 =\frac{\exp(\mathcal{R}(X_1)/\lambda)}{C}.
\]
With respect to the natural filtration $\mathcal F_t$ of the process $X$, the density process of the transformed law is
$h(X_t,t)/C
=\mathbb E_{\mathbb P}[\exp(\mathcal{R}(X_1)/\lambda)/C\mid\mathcal F_t]$.
This is a positive martingale, and the Doob change of measure gives the
following SDE for the coordinate process $Y$ under $\mathbb P^{\mathcal{R}}$:
\[
 \mathrm dY_t
 =\left[b_f(Y_t,t)+\nabla_y\log h(Y_t,t)\right]\mathrm dt
   +\mathrm dB_t^{\mathcal{R}},
 \qquad Y_0=0,\qquad 0\le t<1,
\]
where $B^{\mathcal{R}}$ is a standard Brownian motion under $\mathbb P^{\mathcal{R}}$.
Indeed, the transformed drift satisfies
\[
 b_f(y,t)+\nabla_y\log h(y,t)
 =\nabla_y\log Q_{1-t}(g/C)(y),
 \qquad
 \frac{\mathrm d\mu_R}{\mathrm d\gamma_{d_{\mathcal A}}}=\frac{g}{C},
\]
so $Y$ is itself the SF process associated with $\mu_R$. Moreover, for
every bounded measurable test function $\psi$,
\[
 \mathbb E_{\mathbb P^R}[\psi(Y_1)]
 =\frac{1}{C}\mathbb E_{\mathbb P}
       [\psi(X_1)\exp(\mathcal{R}(X_1)/\lambda)]
 =\int\psi(a)\,\mu_R(\mathrm da).
\]
Thus the transformed process retains the deterministic initial value
and has the desired terminal distribution:
\[
 Y_0=0,
 \qquad
 Y_1\sim\mu_R,
 \qquad
 \mu_R(\mathrm dx)
 =\frac{\exp(\mathcal{R}(x)/\lambda)}{C}\,\mu(\mathrm dx).
\]

For conditional generation, the target depends on an additional variable
$s$. Given a family of conditional distributions
$\{\mu(\cdot\mid s)\}_{s\in\mathcal X}$ satisfying
$\mu(\cdot\mid s)\ll\gamma_{d_{\mathcal A}}$ for each $s$, define
\[
    f_s(x)
    :=
    \frac{\mathrm d\mu(\cdot\mid s)}{\mathrm d \gamma_{d_{\mathcal{A}}}}(x),
    \qquad
    b_s(x,t)
    :=
    \nabla_x\log Q_{1-t}f_s(x).
\]
For each fixed $s$, this conditional process has terminal law
$\mu(\cdot\mid s)$ under the corresponding regularity conditions.
The same Doob transform applies statewise with the terminal reward function $\mathcal{R}(s,x)$ and
normalizer $C(s)=\int\exp(\mathcal{R}(s,x)/\lambda)\mu(\mathrm d x \mid s)$,
giving the terminal law
$\mu_R(\mathrm d x \mid s)=C(s)^{-1}\exp(\mathcal{R} (s,x )/\lambda)
\mu(\mathrm d x \mid s)$.
In SFAC, $s$ is the state, and the terminal law is the conditional action
distribution. Section~\ref{sec:method} uses this representation to
construct the policy update.

\section{Method}
\label{sec:method}
SFAC alternates between action-value estimation and conditional diffusion
policy improvement. Section~\ref{sec:mabo-action-value-learning} develops
the minimax Bellman critic, and Section~\ref{sec:conditional-sfs-method}
constructs the updated diffusion actor. 
The population sequence
$\{\pi_k^*\}_{k=0}^K$ uses exact action-value functions and policy updates,
whereas the practical sequence $\{\pi_k\}_{k=0}^K$ uses estimated
action values, a fitted drift network, and numerical sampling.
We first define the population policy-improvement sequence
\begin{equation*}
    \pi_0^*
    \longrightarrow
    Q^{\pi_0^*}
    \longrightarrow
    \pi_1^*
    \longrightarrow
    Q^{\pi_1^*}
    \longrightarrow
    \cdots
    \longrightarrow
    \pi_{K-1}^*
    \longrightarrow
    Q^{\pi_{K-1}^*}
    \longrightarrow
    \pi_K^* .
\end{equation*}
For a fixed state $s$, the population policy update solves the optimization problem
\begin{equation*}
\pi_{k+1}^*(\cdot\mid s)\in\argmax_{\pi(\cdot\mid s)}
\left\{
\mathbb E_{a\sim\pi(\cdot\mid s)}[Q^{\pi_k^*}(s,a)]
-\lambda\,\mathrm{KL}\!\left(\pi(\cdot\mid s)\,\middle\|\,\pi_k^*(\cdot\mid s)\right)
\right\},
\end{equation*}
 whose optimizer is the improved policy
\begin{align}\label{eq:ideal_policy_improvement}
\pi_{k+1}^*(a\mid s)
=\frac{\pi_k^*(a\mid s)\exp\{Q^{\pi_k^*}(s,a)/\lambda\}}{C},
\end{align}
where $\lambda>0$ controls the deviation from
$\pi_k^*(\cdot\mid s)$ and $Q^{\pi_k^*}$ is its action-value function.
By the policy-improvement result of \citet{frans2025diffusion},
$J(\pi_{k+1}^*) \ge J(\pi_k^*)$, ensuring monotonic improvement of the ideal policy sequence.

In practice, each iteration collects online data under the current policy and
estimates its action-value function. Starting from the same initial policy as the
population sequence, $\pi_0=\pi_0^*$, the practical sequence is
\begin{equation*}
    \pi_0
    \longrightarrow
    \widehat Q_0
    \longrightarrow
    \pi_1
    \longrightarrow
    \widehat Q_1
    \longrightarrow
    \cdots
    \longrightarrow
    \pi_{K-1}
    \longrightarrow
    \widehat Q_{K-1}
    \longrightarrow
    \pi_K .
\end{equation*}
More precisely, at iteration $k$, the minimax Bellman critic estimates the action-value
function of the current policy $\pi_k$ from the available transition data,
and we denote the resulting estimate by $\widehat Q_k$.
The plug-in target policy uses this critic directly:
\begin{equation}
\label{eq:empirical-target-policy}
\bar\pi_{k+1}(a\mid s)
=
\frac{
    \pi_k(a\mid s)
    \exp\left(
        \widehat Q_k(s,a)/\lambda
    \right)
}{
 C
}.
\end{equation}
The normalizing constant $C$  in \eqref{eq:empirical-target-policy} need not
have a closed form, and direct samples from
$\bar\pi_{k+1}(\cdot\mid s)$ are not assumed to be available. SFAC represents this target through a conditional SF process whose full
drift is fitted by $Q$-weighted Brownian-bridge regression.
Euler--Maruyama simulation of the fitted network defines the practical policy
$\pi_{k+1}(\cdot\mid s)$.
The update is summarized by
\begin{equation*}
\bigl(
    \pi_k,\mathcal D_k
\bigr)
\xrightarrow{\text{Action-Value Function Learning}}
\widehat Q_k
\xrightarrow{\eqref{eq:empirical-target-policy}}
\bar\pi_{k+1}
\xrightarrow{\mbox{Algorithm}~\ref{alg:conditional-sf-policy-improvement}}
\pi_{k+1}.
\end{equation*}
Algorithm~\ref{alg1} summarizes the proposed SFAC method.
\begin{algorithm}[H]
\caption{Schr\"odinger--F\"ollmer Actor--Critic}
\label{alg1}
\begin{algorithmic}[1]
\STATE
\textbf{Input:}
initial policy $\pi_0$,
initial transition
data $\cD_0$, regularization
parameter $\lambda$, number of diffusion steps $T$, and number
of policy iterations $K$.

\FOR{$k=0,1,\ldots,K-1$}

\STATE
Fit the action-value estimate $\widehat Q_k$ using the empirical
minimax Bellman objective in \eqref{eq:empirical-critic}.
\STATE
Fit the updated actor and define $\pi_{k+1}$ using
Algorithm~\ref{alg:conditional-sf-policy-improvement}; collect new
transitions to update $\cD_{k+1}$.
\ENDFOR

\STATE
\textbf{Output:}
$\pi_K$.

\end{algorithmic}
\end{algorithm}

\subsection{Action-Value Function Learning}
\label{sec:mabo-action-value-learning}

We estimate the action-value function of the current policy through a
minimax Bellman objective. 
The critic targets the fixed point $Q^{\pi_k}$ in
\eqref{eq:policy-bellman-fixed-points} by controlling the Bellman residual
$Q-\mathcal T^{\pi_k}Q$.
Let $\mu_k$ denote the state--action sampling distribution at iteration
$k$. For a candidate action-value function $Q \in \cU_1$ and an auxiliary
function $O \in \cU_2$, consider the population minimax objective
\begin{equation}
\label{eq:population-mabo-objective}
\begin{aligned}
       \min_{Q\in\mathcal U_1}
    \max_{O\in\mathcal U_2}
    \mathcal L_k^{\mathrm{crit}}(Q,O).
\end{aligned}
\end{equation}
The population loss is
\begin{align*}
    \mathcal L_k^{\mathrm{crit}}(Q,O) :=
    \mathbb E_{(S,A)\sim\mu_k}
    \Big[
        \bigl(
            Q(S,A)-\mathcal T^{\pi_k}Q(S,A)
        \bigr)^2 -
        \bigl(
            O(S,A)-\mathcal T^{\pi_k}Q(S,A)
         \bigr)^2
    \Big],
\end{align*}
where $\mathcal T^{\pi_k}$ is the state--action Bellman operator in
\eqref{eq:policy-bellman-operators}, evaluated at the current policy $\pi_k$.
Under the corresponding realizability and completeness conditions,
$Q^{\pi_k}$ is a population minimizer of
\eqref{eq:population-mabo-objective}.

Let
$   \mathcal D_k
    =
    \left\{
        (s_i,a_i,r_i,s_i')
    \right\}_{i=1}^{n}
$ 
denote the transition data available at iteration $k$.
For the finite-sample analysis, we model $\mathcal D_k$ as a length-$n$
transition sequence with state--action marginal $\mu_k$ and dependence
controlled by the $\beta$-mixing condition in
Section~\ref{sec:theory of action-value learning}. The analysis does not
separately quantify the effects of replay-buffer management or neural
critic optimization.
For each next state $s_i'$, sample
$
    a_i'
    \sim
    \pi_k(\cdot\mid s_i'),
$ 
and define the temporal-difference target
 $
    y_i 
    :=
    r_i+\gamma Q(s_i',a_i')
 $ for the  candidate $Q$. 
We estimate $Q^{\pi_k}$ by solving the empirical minimax problem
\begin{equation}
\label{eq:empirical-critic}
    \widehat Q_k
    \in
    \argmin_{Q\in\mathcal G_1}
    \max_{O\in\mathcal G_2}
    \widehat{\mathcal L}_k^{\mathrm{crit}}(Q,O).
\end{equation}
where
\begin{equation*}
\begin{aligned}
    \widehat{\mathcal L}_k^{\mathrm{crit}}(Q,O)
    :=
    \frac{1}{n}
    \sum_{i=1}^{n}
    \Big[
        \bigl(
            Q(s_i,a_i)-y_i 
        \bigr)^2-
        \bigl(
            O(s_i,a_i)-y_i 
        \bigr)^2
    \Big].
\end{aligned}
\end{equation*}
For a fixed $Q$, the auxiliary function $O$ approximates the conditional
mean of the one-step target, $\mathcal T^{\pi_k}Q$.
Because both squared losses use the same target $y_i$, their conditional
variance terms cancel in expectation. This yields the population
objective above without requiring two independent next-state samples
for the same state--action pair, although finite-sample fluctuations remain.
We use $\widehat Q_k$ directly in the endpoint weight
$\omega_k(s,a)=\exp\{\widehat Q_k(s,a)/\lambda\}$.
Section~\ref{sec:conditional-sfs-method} constructs a learned diffusion
sampler for the resulting target $\bar\pi_{k+1}$.

\subsection{Conditional Schr\"odinger--F\"ollmer Policy Improvement}
\label{sec:conditional-sfs-method}
Fix a state $s$ and represent $\pi_k(\cdot\mid s)$ as the terminal law
of a conditional SF process. The density ratio
\begin{align*} 
    \frac{
    \mathrm d\bar\pi_{k+1}(\cdot\mid s)
}{
    \mathrm d\pi_k(\cdot\mid s)
}(a)
=
\frac{
    \exp\left(\widehat Q_k(s,a)/\lambda\right)
}{
   C
},
\end{align*}
defines a terminal change of measure, where $C$ is the
normalizing constant in \eqref{eq:empirical-target-policy}.
The associated Doob $h$-transform has terminal law
$\bar\pi_{k+1}(\cdot\mid s)$. We express the target drift as a conditional mean, derive its weighted
regression loss, and use Euler--Maruyama sampling to generate actions
from the fitted actor.

\paragraph{Reference diffusion.}
We use $\pi_k(\cdot\mid s)$ as the reference action distribution. Suppose
that $\pi_k(\cdot\mid s)$ is absolutely continuous with respect to the
standard Gaussian measure $\gamma_{d_{\mathcal A}}$, and define
\begin{equation*}
    f_{k,s}(a)
    :=
    \frac{\mathrm d\pi_k(\cdot\mid s)}{\mathrm d \gamma_{d_{\mathcal{A}}}}(a),
    \qquad a\in\mathbb R^{d_{\mathcal{A}}}.
\end{equation*}
The reference diffusion satisfies \eqref{eq:sf-sde-preliminary} with drift $b_k (Y_t,t;s)$.
As stated in Section~\ref{sec:prelim-sf-process}, the exact reference drift is
$b_k(y,t;s)=\nabla_y\log Q_{1-t}f_{k,s}(y)$,  
and 
$Y_1\sim\pi_k(\cdot\mid s)$.
The same process admits the representation
\begin{equation}\label{eq:reference-sf}
\begin{aligned}
\mathrm dY_t
&=
\left[
\frac{Y_t}{t}
+\nabla_y\log p_t(Y_t\mid s)
\right]\mathrm dt+\mathrm dB_t,
\qquad t\in(0,1),
\\
&Y_0=0,
\qquad
Y_1\sim\pi_k(\cdot\mid s).
\end{aligned}
\end{equation}
Here, for $t\in(0,1)$,
\[
p_t(y\mid s):=
\int
(2\pi t(1-t))^{-d_{\mathcal{A}}/2}
\exp\!\left(
-\frac{\| y-ty_1\|^2}{2t(1-t)}
\right)
\pi_k(\mathrm dy_1\mid s)
\]
is the density of
$tY_1+\sqrt{t(1-t)}Z$, where
$Y_1\sim\pi_k(\cdot\mid s)$ and $Z\sim \mathcal N(0, I_{d_{\mathcal{A}}})$ are independent.
With the endpoint weight defined above,
\eqref{eq:empirical-target-policy} gives
\begin{equation*}
\frac{
    \mathrm d\bar\pi_{k+1}(\cdot\mid s)
}{
    \mathrm d\pi_k(\cdot\mid s)
}(a)
=
\frac{\omega_k(s,a)}{C}.
\end{equation*}

\paragraph{$Q$-guided Doob $h$-transform.}
Let $\mathbb P_k^s$ denote the path measure of the reference process.
Reweighting this path measure by
$\omega_k(s,Y_1)/C$ yields a path measure whose terminal law is
$\bar\pi_{k+1}(\cdot\mid s)$. Define
\begin{equation*}
h_k(y,t;s)
=
\mathbb E_{\mathbb P_k^s}
\left[
    \omega_k(s,Y_1)
    \mid
    Y_t=y
\right].
\end{equation*}
Since $\pi_k(\cdot\mid s)\ll\gamma_{d_{\mathcal A}}$,
the target density ratio with respect to the Gaussian measure is
\begin{equation*}
    \bar f_{k+1,s}(a)
    :=
    \frac{
        \rm d\bar\pi_{k+1}(\cdot\mid s)
    }{
        \rm d \gamma_{d_{\mathcal{A}}}
    }(a)= \frac{
        \exp\left(
            \widehat Q_k(s,a)/\lambda
        \right)
     f_{k,s}(a)}{C}
   ,
    \qquad
    a\in\mathbb R^{d_{\mathcal{A}}}.
\end{equation*}
As described in Section \ref{sec:prelim-sf-process}, the Doob $h$-transformed process $\{Z_t\}_{t\in[0,1]}$ has terminal law
$\bar\pi_{k+1}(\cdot\mid s)$ and satisfies
\begin{align} \label{eq:doobs-sde}
    \mathrm dZ_t
    =
   \bar b_{k+1,s}(Z_t,t)\mathrm dt+\mathrm dB_t,
    \qquad
    Z_0=0,
    \qquad
    t\in[0,1],
\end{align}
where
\begin{align}
\label{eq:doob-drift-density-ratio}
    \bar b_{k+1,s}(y,t)
    =
    \nabla_y
    \log Q_{1-t}\bar f_{k+1,s}(y)
    =b_k (y,t;s)+\nabla_y\log h_k(y,t;s).
\end{align}

\paragraph{Conditional-mean representation.}
The target drift $\bar b_{k+1,s}$ need not equal the SF drift
$b_{k+1,s}$ of the fitted policy $\pi_{k+1}$ used as the next reference.
To express the target in a form suitable for regression, fix $s$ and
$0<t<1$. Under the reference path law, $Y_1\sim\pi_k(\cdot\mid s)$ and
$Y_t\mid Y_1=a,s\sim\mathcal N(ta,t(1-t)I_{d_{\mathcal A}})$.
The Doob transform reweights endpoints by $\omega_k(s,Y_1)/C$ and
preserves this conditional bridge. The target drift in
\eqref{eq:doob-drift-density-ratio} is therefore
\[
\bar b_{k+1,s}(y,t)
=\frac{\mathbb E\!\left[
\omega_k(s,Y_1)\frac{Y_1-y}{1-t}\,\middle|\,Y_t=y,s\right]}
{\mathbb E\!\left[\omega_k(s,Y_1)\mid Y_t=y,s\right]},
\]
where both expectations are under the reference law. Multiplying by
the denominator gives
\[
\mathbb E\!\left[
\omega_k(s,Y_1)
\left(\frac{Y_1-y}{1-t}-\bar b_{k+1,s}(y,t)\right)
\,\middle|\,Y_t=y,s\right]=0.
\]
For any candidate field $b$, the difference
$b(s,y,t)-\bar b_{k+1,s}(y,t)$ is fixed conditional on $(s,y,t)$.
Its weighted inner product with the bridge-label residual thus has
zero conditional expectation.

\paragraph{Weighted bridge regression.}
Average over $s\sim\mu_{k,\mathcal X}$, the state marginal of the
population replay law $\mu_k$, endpoints $a\sim\pi_k(\cdot\mid s)$,
uniform grid times $t\in\{j/T:0\le j\le T-1\}$, and bridge positions
$y\sim\mathcal N(ta,t(1-t)I_{d_{\mathcal A}})$.
Expanding the square and using the zero cross term gives
\begin{align*}
\mathcal L_{\mathrm{actor}}^k(b)
&:=\mathbb E\!\left[
\frac{\omega_k(s,a)}{d_{\mathcal A}}
\left\|b(s,y,t)-\frac{a-y}{1-t}\right\|_2^2\right]\\
&=\mathbb E\!\left[
\frac{\omega_k(s,a)}{d_{\mathcal A}}
\left\|b(s,y,t)-\bar b_{k+1,s}(y,t)\right\|_2^2\right]+\mathbb E\left[
\frac{\omega_k(s,a)}{d_{\mathcal A}}
\left\|\bar b_{k+1,s}(y,t)-\frac{a-y}{1-t}\right\|_2^2\right].
\end{align*}
%When these second moments are finite, 
Note that the last term is independent of
$b$. Minimizing the regression loss therefore minimizes the weighted
squared error to the target drift. Since $\omega_k>0$, the unrestricted
minimizer equals $\bar b_{k+1,s}$ almost surely under the regression
design. 
The factor $1/d_{\mathcal A}$ averages over action coordinates without
changing the minimizer. Appendix~\ref{sec:proof of lemma drift mc}
derives the conditional-mean identity and the loss decomposition.

At each state, dividing by $C$ converts the weighted endpoint
expectation into an expectation under $\bar\pi_{k+1}(\cdot\mid s)$.
Omitting $C$ preserves the statewise minimizer but changes the relative
weight of states in a shared neural network. We use the action-value
weights directly.
To construct the training samples, draw
$a_i\sim\pi_k(\cdot\mid s_i)$ from the frozen policy,
$t_i$ uniformly from $\{j/T:0\le j\le T-1\}$, and
$\varepsilon_i\sim\mathcal N(0,I_{d_{\mathcal A}})$ independently.
Set
\[
y_i=t_i a_i+\sqrt{t_i(1-t_i)}\,\varepsilon_i
\]
and use $(a_i-y_i)/(1-t_i)$ as the regression label.
At $t_i=0$, $y_i=0$ and the label is $a_i$; the same conditional-mean
identity holds. The endpoint $t_i=1$ is excluded from the training grid.

For online learning, let $\mathcal G_3$ be a class of vector-valued ReLU
networks with input $(s,y,t)$ and a coordinatewise output bound $B_{\rm net}$.
For the statistical analysis, we take an empirical risk minimizer, i.e., 
\[
b_{\phi_{k+1}}\in\argmin_{b_\phi\in\mathcal G_3}
\widehat{\mathcal L}_{\mathrm{actor}}^k(\phi),
\]
where the empirical loss is
\begin{align}
\label{eq:empirical-drift-correction}
\widehat{\mathcal L}_{\mathrm{actor}}^k(\phi)
:=\frac1n\sum_{i=1}^{n}\omega_k(s_i,a_i)
\frac{1}{d_{\mathcal A}}
\left\|
b_\phi(s_i,y_i,t_i)-\frac{a_i-y_i}{1-t_i}
\right\|_2^2.
\end{align}
As in the population loss, $1/d_{\mathcal A}$ averages over coordinates.
We draw states from replay data and initialize the actor at $b_{\phi_k}$.
During each update block, the endpoint policy and critic remain frozen.
Each training step resamples endpoints, times, and bridge noise, and
updates only $b_\phi$. Here $n$ counts regression observations and $B$
is the minibatch size; the sampling-discrepancy term accounts for reused
replay states. The fitted drift and its Euler--Maruyama sampler define
$\pi_{k+1}$, as summarized in
Algorithm~\ref{alg:conditional-sf-policy-improvement}.
Notice that we pretrain $b_{\phi_0}$ by bridge regression on the offline
state--action dataset $\mathcal D_{\mathrm{off}}$; see
Algorithm~\ref{alg:offline-pretraining} in Appendix~\ref{sec:offline_alg}.

\paragraph{Euler--Maruyama sampling.}
After weighted bridge regression, we hold $b_{\phi_{k+1}}$ fixed
during action generation. For an integer $T\ge2$, set $\Delta t=\tau=1/T$ and let
$
\mathcal T_T:=\{t_j=j/T:0\le j\le T-1\}.
$ Starting from the deterministic
initial value $\widetilde Z_{t_0}=0$, we apply the Euler--Maruyama
scheme
\begin{equation}
\label{eq:EM-for-Z}
    \widetilde Z_{t_{j+1}}
    =
    \widetilde Z_{t_j}
    +
        b_{\phi_{k+1}}
    \left(s,
        \widetilde Z_{t_j},t_j
    \right)\Delta t
    +
    \sqrt{\Delta t}\,\varepsilon_{j+1},
\end{equation}
where
$\varepsilon_{j+1}\overset{\mathrm{i.i.d.}}{\sim}
\mathcal N(0, I_{d_{\mathcal{A}}})$.
On the terminal interval $[1-1/T,1]$, the scheme holds the last evaluated
drift value fixed. The output bound on $\mathcal G_3$ controls this
continuation in the analysis.
The terminal law $
    \pi_{k+1}(\cdot\mid s)
    :=
    \mathrm{Law}
    \left(
        \widetilde Z_{t_T}\mid s
    \right)
$
defines the practical updated policy.

 \begin{algorithm}[H]
\caption{Conditional  Schr\"odinger--F\"ollmer Policy Improvement}
\label{alg:conditional-sf-policy-improvement}
\begin{algorithmic}[1]
\STATE \textbf{Input:} Regularization parameter $\lambda$,
number of diffusion steps $T$, minibatch size $B$, and state $s$
for action generation.
    \STATE At iteration $k$, Compute $\omega_k(s_i,a_i)$ and update $\phi$ using
    the minibatch version of \eqref{eq:empirical-drift-correction}.
\STATE Set $b_{\phi_{k+1}}\leftarrow b_\phi$,
$\widetilde Z_{t_0}=0$.
\FOR{$j=0,1,\ldots,T-1$}
    \STATE Update $\widetilde Z_{t_{j+1}}$ by
    \eqref{eq:EM-for-Z} using the fixed network $b_{\phi_{k+1}}$.
\ENDFOR
\STATE \textbf{Output:} action
$A=\widetilde Z_{t_T} \sim \pi_{k+1}(\cdot\mid s)$.
\end{algorithmic}
\end{algorithm}
Algorithm~\ref{alg:conditional-sf-policy-improvement} summarizes this actor
update. The update fits the full target drift without critic action
derivatives and generates actions by Euler--Maruyama simulation. As $\lambda^{-1}\to0$, the ideal exponential tilt follows the natural
policy-gradient direction to first order, while the Doob correction represents
the same local change in diffusion-drift space. Appendix~\ref{sec:npg-connection}
derives these expansions.

\section{Theoretical Analysis}\label{sec:theory}
This section analyzes action-value estimation and conditional diffusion
sampling and then relates their errors to policy performance.
Section~\ref{sec:theory of action-value learning} bounds the mean-square
action-value estimation error under dependent observations.
Section~\ref{sec:conditional-sfs-theory} uses weighted regression
and a martingale argument under the target path law to bound reverse KL
from the plug-in target to the practical policy.
Section~\ref{sec:Final convergence rate} then bounds average performance
relative to a covered comparison policy, with explicit dependence on
the policy sequence.

\subsection{Convergence Rate of Action-Value Estimation}\label{sec:theory of action-value learning}
Assumptions~\ref{assump2}--\ref{ass:completeness} specify the coverage,
tail, and Bellman regularity conditions. Together with the mixing and
smoothness conditions below, they yield the mean-square action-value error
bound in Theorem~\ref{thm:action-value-error}.

\begin{assumption}\label{assump2}
There exists a constant $C<\infty$ such that, for every $k$ and every
admissible distribution $\nu\in\mathfrak M$, $\nu\ll\mu_k$ and
\[
\left\|\frac{\mathrm d\nu}{\mathrm d\mu_k}\right\|_\infty \le C.
\]
\end{assumption}
\begin{assumption}
\label{assump:critic-subGaussian-tail}
There exist constants
$c_\mu,C_\mu>0$
such that
\[
\sup_{0\le k<K}
\max
\left\{
    \int_{\mathbb R^d}
    \exp\!\left(c_\mu\|z\|_2^2\right)
    \mu_k(\mathrm dz),
    \int_{\mathbb R^d}
    \exp\!\left(c_\mu\|z\|_2^2\right)
    (\mu_kP^{\pi_k})(\mathrm dz)
\right\}
\le C_\mu,
\]
where $z=(s,a)\in\mathbb R^{d_{\mathcal X}+d_{\mathcal A}}$.
\end{assumption}
 
\begin{assumption}
\label{ass:completeness}
For every $k$ and each $Q\in\mathcal G_1$,
$
\mathcal T^{\pi_k}Q
\in
\mathcal H^\zeta
\left(
    \mathbb R^d,B_{\rm H}
\right),
$ and the output bound for the network class $\mathcal G_1$ satisfies
$B_{\rm net}\ge B_{\rm H}$.
\end{assumption}
 
\begin{theorem}[Convergence rate of action-value estimation]\label{thm:action-value-error}
Suppose that
 Assumptions~\ref{assump2}--\ref{ass:completeness} hold and  $Q^{\pi_k}\in\mathcal H^\zeta(\mathbb R^d,B_{\rm H})$
uniformly in $k$. Conditional on the history that determines $\pi_k$, suppose that
the critic sample is strictly stationary and $\beta$-mixing with
$\beta_r\le\bar\beta e^{-br^\eta}$.
Choose
\begin{align*}
a_n
&=\left\lceil(2\log n/b)^{1/\eta}\right\rceil,
&
\xi_n
&=\left\lfloor n/(2a_n)\right\rfloor,\\
\mathcal S
&\asymp\xi_n^{d/(d+2\zeta)}\log \xi_n,
&
\mathcal D
&\asymp\log \xi_n.
\end{align*}
Then, uniformly over $k$ and $\nu\in\mathfrak M$,
\[
\mathbb E\|\widehat Q_k-Q^{\pi_k}\|_{L^2(\nu)}^2
\le \frac{1}{(1-\gamma)^2}
\widetilde{\mathcal O}\left(
n^{-\frac{\zeta}{d+2\zeta}}\right).
\]
\end{theorem}
\begin{remark}
Assumption~\ref{assump2} imposes uniform concentrability: the sampling
law covers every admissible state--action distribution with a bounded
density ratio. A similar condition appears in
Assumption~1 of \citet{chen2019information}; related concentrability
coefficients are used by \citet{antos2008learning} and \citet{xie2020q}.
Assumption~\ref{assump:critic-subGaussian-tail} controls the tails of the
current and one-step state--action laws, making errors outside expanding
compact sets negligible even on unbounded domains. %
Bounded-design approximation is used in Section~4.2 of
\citet{feng2023over}, while Section~2.1 of \citet{jiao2025deep} assumes
bounded state and action spaces. On a common bounded state--action domain,
both exponential-moment bounds hold automatically. The explicit tail
condition here extends the analysis to unbounded support; it differs from
the bounded-domain conditions in those papers.
Assumption~\ref{ass:completeness} requires H\"older regularity of Bellman
images so that the auxiliary neural network class can approximate them.
It is a smoothness-based counterpart of the Bellman-completeness
condition in Assumption~3 of \citet{chen2019information}.
H\"older regularity of Bellman targets is used in Theorem~21 and
Remark~22 of \citet{feng2023over} and in the deep policy-iteration analysis
of \citet{jiao2025deep}. The condition does not require exact closure of a
finite network class under the Bellman operator; approximation error is
controlled through network size.
The smoothness condition on $Q^{\pi_k}$ controls approximation of
the action-value function itself.
Theorem~\ref{thm:action-value-error} establishes mean-square consistency of
$\widehat Q_k$, uniformly over iterations and admissible distributions.
Its rate balances neural approximation and statistical error under
dependent observations. Blocking yields the effective sample size
$\xi_n\asymp n/(\log n)^{1/\eta}$, so exponential mixing affects the
stated polynomial rate only through logarithmic factors.
The factor $(1-\gamma)^{-2}$ reflects amplification through the
discounted Bellman equation. The policy-performance guarantee also requires
the actor and error-propagation analyses that follow.
\end{remark}

\subsection{Error Bound for the Amortized Conditional SF Actor}
\label{sec:conditional-sfs-theory}
We bound the discrepancy between the practical policy $\pi_{k+1}$
and the plug-in target $\bar\pi_{k+1}$ by drift-estimation and
time-discretization errors. The drift error consists of statistical
and neural approximation errors.
Condition on the frozen bounded finite ReLU actor $b_{\phi_k}$
and bounded finite ReLU critic $\widehat Q_k$, the regression design draws
$S\sim\mu_{k,\mathcal X}$, $A\mid S\sim\pi_k(\cdot\mid S)$,
$t$ uniformly from $\mathcal T_T$, and an independent standard Gaussian
$\varepsilon$, and sets
\[
X=(S,Y,t),\qquad
Y=tA+\sqrt{t(1-t)}\,\varepsilon,\qquad
U_k=\frac{A-Y}{1-t}.
\]
Define the weighted design norm by
\[
\|g\|_{\chi_k}^2
:=\mathbb E\left[\frac{\omega_k(S,A)}{d_{\mathcal A}}
\|g(S,Y,t)\|_2^2\right].
\]
The weighted conditional mean of $U_k$ is $\bar b_{k+1}(X)$, so
\begin{equation}
\mathcal L_{\rm actor}^k(b)-\mathcal L_{\rm actor}^k(\bar b_{k+1})
=\|b-\bar b_{k+1}\|_{\chi_k}^2.
\label{eq:bridge-regression-drift-error}
\end{equation}
Weighted bridge regression therefore estimates the full target drift
in this norm.
We impose the following two assumptions on the regularity of
$f_{k,s}$ and $\bar f_{k+1,s}$ and the uniform coverage of
the relevant distributions.  
\begin{assumption} 
\label{assump:actor-target-regularity}
There exists $\xi_f>0$ such that, for every $k$, the Gaussian density
ratios $f_{k,s}$ and $\bar f_{k+1,s}$ are bounded below by $\xi_f$ and,
as functions of $(s,a)$, belong to
$\mathcal H^{2\zeta+2}
(\mathbb R^{d_{\mathcal X}+d_{\mathcal A}},B_f)$.
\end{assumption}

 \begin{assumption}
 %[Uniform policy--occupancy coverage]
\label{ass:uniform-coverage-assumption}
There exists  a uniform constant $C<\infty$
%, independent of $n$
%and $K$, 
such that for every realized sequence of $K$ updates, 
\begin{align}
 C_K=\max_{0\le k<K}\max\Biggl\{1,
&\left\|\frac{\mathrm d d^{\pi^\star}}
{\mathrm d d^{\bar\pi_{k+1}}}\right\|_\infty,
\left\|\frac{\mathrm d d^{\bar\pi_{k+1}}}
{\mathrm d d^{\pi^\star}}\right\|_\infty,\left\|\frac{\mathrm d d^{\pi^\star}}
{\mathrm d\mu_{k,\mathcal X}}\right\|_\infty,
\mathop{\mathrm{ess\,sup}}_{s\sim d^{\pi^\star}}
\left\|\frac{\mathrm d\pi^\star(\cdot\mid s)}
{\mathrm d\bar\pi_{k+1}(\cdot\mid s)}\right\|_\infty
\Biggr\}\le C.
\label{ass:uniform-coverage}
\end{align}
\end{assumption}

\begin{theorem}[SF actor error bounds]
\label{thm:conditional-sfs-sampling-error}
Suppose Assumptions~\ref{assump:critic-subGaussian-tail},
\ref{assump:actor-target-regularity}, and
\ref{ass:uniform-coverage-assumption} hold. Let $b_{\phi_{k+1}}$
minimize the empirical actor loss over $\mathcal G_3$ using $n$
conditionally independent draws from the frozen design. %If
%$\mathcal D=\mathcal O(\log n)$
%{\color{red}
%Give detailed rate e.g. $T=n^d$:
%and
%$T,\mathcal W,B_{\rm par},B_{\rm net}$ grow at most polynomially in $n$,
%}
For the class in Lemma~\ref{lem:actor-approximation-error},  by taking
$p=d_{\mathcal X}+d_{\mathcal A}+1$,
$\mathcal S\asymp n^{p/(p+4\zeta)}\log n$,
$\mathcal D\asymp\log n$, $T\asymp n^{2\zeta/(d+1+4\zeta)}$, and fixing  $B_{\rm net}\ge B_{\rm H}$,
we then have 
\begin{equation}
\mathbb E \left[
\mathbb E_{S\sim d^{\pi^\star}}
\mathrm{KL}(\bar\pi_{k+1}(\cdot\mid S)\|\pi_{k+1}(\cdot\mid S))
\right]
\lesssim \widetilde{\mathcal O} \left(n^{-2\zeta/(p+4\zeta)}\right).
%+\frac1T
\label{eq:actor-holder-sample-rate}
\end{equation}
%for polynomially growing $T\asymp n^{2\zeta/(d+1+4\zeta)}$.
 
\end{theorem}

\begin{remark}
Assumption~\ref{assump:actor-target-regularity} controls the Gaussian
density ratios of the reference and plug-in target policies. Their
joint smoothness in state and action, together with the positive lower
bound, controls the logarithmic derivatives of the heat-semigroup
density that define the SF drift. Lemma~\ref{lem:actor-approximation-error}
uses this regularity to obtain the explicit network-size bound
\eqref{eq:sketch-actor-approximation-bound}; it is an additional condition
on the target densities, not a consequence of using a finite ReLU actor
or critic. The condition concerns the diffusion representation of the
policy update. In Section~\ref{sec:npg-connection},
Proposition~\ref{prop:local-npg} describes the local direction of the
ideal KL update in a smooth parametric policy family, while
Proposition~\ref{prop:local-doob-gradient} describes the first-order
Doob correction for the plug-in $Q$ tilt. The two local statements do
not replace the target regularity required for the actor's explicit
approximation rate. Assumption~\ref{assump:actor-target-regularity} provides the regularity
needed for this diffusion representation: positivity prevents degeneracy
of the logarithmic derivatives defining the drift, while smoothness
controls the density-to-drift map and supports neural approximation and
time discretization. Related density-ratio conditions are standard in
nonasymptotic SF analyses
\citep{jiao2021convergence,huang2021schrodinger}; here they are imposed
jointly in the state and action variables, with additional smoothness to
control time derivatives of the heat semigroup. From an actor--critic
viewpoint \citep{wu2020finite,chen2023finite,wang2024non,chen2026finite},
this is analogous to assuming bounded or Lipschitz policy scores for a
parametric actor: both ensure that a local policy update induces a
controlled change in its representation. The difference is that classical
actor--critic theory regularizes the map
$\theta\mapsto\pi_\theta$, whereas SFAC regularizes the induced
density-to-drift map. This is an additional condition on the plug-in
target and is not implied by the ReLU critic specification.

Assumption~\ref{ass:uniform-coverage-assumption} provides the change-of-measure
conditions used in the policy-performance analysis. Its state-occupancy
bounds relate the current, plug-in target, and comparator distributions,
allowing critic and policy-approximation errors to be evaluated under a
common reference distribution, as in concentrability-based RL analyses
\citep{chen2019information,xie2020q,xie2021policy,zhan2022offline}. The action-level bound complements
this state coverage by requiring the plug-in target to dominate the
comparator policy at each state; state-occupancy coverage alone does not
imply this action-level domination, and the assumption does not assert
monotonic policy improvement.
Theorem~\ref{thm:conditional-sfs-sampling-error} controls expected
reverse KL from the plug-in target to the learned Euler policy.  The explicit rate
\eqref{eq:actor-holder-sample-rate} additionally requires
Assumption~\ref{assump:actor-target-regularity} and conditionally
independent observations from the frozen design. For dependent or
reused observations, Lemma~\ref{lem:learned-drift-error} retains an
empirical design discrepancy; repeated draws from a fixed replay
buffer do not supply independent population observations, since 
the actor regresses the full drift directly.
\end{remark}

\subsection{Overall Policy-Performance Guarantee}
\label{sec:Final convergence rate}
We use the reverse-KL actor bound to control performance over policy
iterations. Here,  $\pi^\star$ is a fixed comparison policy. The result
gives an optimality guarantee when an optimal policy satisfies the
comparison conditions. Define
\[
D(\pi):=\mathbb E_{s\sim d^{\pi^\star}}
\mathrm{KL}(\pi^\star(\cdot\mid s)\|\pi(\cdot\mid s)),
\qquad D_0=D(\pi_0).
\]

\begin{theorem}[Overall convergence rate]
\label{thm:overall-rate}
Suppose Assumptions~\ref{assump2}--\ref{ass:uniform-coverage-assumption}
and the conditions of Theorems~\ref{thm:action-value-error} and
\ref{thm:conditional-sfs-sampling-error} hold uniformly over
$0\le k<K$.
%, with constants independent of $n$ and $K$.
Let $D_0=D(\pi_0)<\infty$ for a fixed initial policy.
Under the stated network-size choices, with
$T\asymp n^{2\zeta/(d+1+4\zeta)}$,
\begin{align}
&\mathbb E\left[\frac1K\sum_{k=0}^{K-1}
\bigl(J(\pi^\star)-J(\pi_k)\bigr)\right]\notag\\
& \lesssim\frac{R_{\max}+\lambda D_0}{K(1-\gamma)}+\frac1{(1-\gamma)^2}
\widetilde{\mathcal O} \left(n^{-\frac{\zeta}{2d+4\zeta}}\right)
+\frac1{(1-\gamma)^2}
\widetilde{\mathcal O}\left(n^{-\frac{\zeta}{d+1+4\zeta}}\right).
\label{eq:overall-explicit-rate}
\end{align}
Choosing $K\asymp n^{\zeta/(2d+4\zeta)}$ gives
\begin{equation}
\mathbb E\left[\frac1K\sum_{k=0}^{K-1}
\bigl(J(\pi^\star)-J(\pi_k)\bigr)\right]
\lesssim\frac1{(1-\gamma)^2}
\widetilde{\mathcal O}\!\left(n^{-\frac{\zeta}{2d+4\zeta}}\right).
\label{eq:overall-final-rate}
\end{equation}
\end{theorem}

\begin{remark}
The three contributions in \eqref{eq:overall-explicit-rate} are the
$K^{-1}$ policy-iteration term, the critic estimation rate, and the
square root of the actor's reverse-KL rate. The square root arises
when Pinsker's inequality transfers policy error to return and
comparator-KL differences. With
$d=d_{\mathcal X}+d_{\mathcal A}$ and actor input dimension $p=d+1$,
the actor contribution has order
$\widetilde{\mathcal O}(n^{-\zeta/(d+1+4\zeta)})$, while the critic
contribution has order
$\widetilde{\mathcal O}(n^{-\zeta/(2d+4\zeta)})$.
For $d\ge1$, the critic rate is no faster than the actor rate;
balancing it with $K^{-1}$ gives \eqref{eq:overall-final-rate}.
This guarantee concerns the average gap to the fixed comparator,
not a last-iterate or monotonic-improvement claim. It applies to
an optimal comparator only when that policy satisfies the stated
initial-KL and coverage conditions. This error decomposition is related to both approximate policy-iteration
theory \citep{xie2020q,zanette2021provable,jiao2025deep} and finite-sample
actor--critic theory
\citep{wu2020finite,xu2020improving,wang2024non,
gaur2024closing,ganesh2025order,chen2025approximate}.
\end{remark}
\begin{remark}
Finite-time actor--critic analyses commonly track critic and approximation
errors in parameterized policy-gradient updates. Under Markovian sampling,
two-timescale AC reaches a first-order stationary point with
$\widetilde{\mathcal O}(\epsilon^{-5/2})$ samples
\citep{wu2020finite}, while a single-timescale analysis obtains
$\widetilde{\mathcal O}(\epsilon^{-2})$ \citep{chen2023finite}.
Further results address deep neural critics, compatible function
approximation, and global performance in continuous state--action spaces
\citep{chen2026finite,wang2024non,gaur2024closing,ganesh2025order}.
In a policy-space analysis, \citet{chen2025approximate} view natural
policy gradient as approximate policy iteration and derive geometric
actor convergence and $\widetilde{\mathcal O}(\epsilon^{-2})$ sample
complexity up to function-approximation error. SFAC also analyzes
policy improvement: it defines a KL-regularized target and fits its
conditional Schr\"odinger--F\"ollmer drift. Theorem~\ref{thm:overall-rate}
bounds the average performance gap to a covered comparison policy by
propagating critic estimation, weighted drift-regression, and diffusion
discretization errors. The critic bound uses geometric $\beta$-mixing;
the actor rate assumes conditionally independent observations from the
frozen regression design. Its separate policy-iteration, regression,
and diffusion-step budgets do not directly yield a comparable
$\epsilon$-sample complexity for environment interactions.
\end{remark}

\section{Proof Sketches}\label{sec:proof sketch}
This section presents proof sketches for
Theorems~\ref{thm:action-value-error},
\ref{thm:conditional-sfs-sampling-error}, and~\ref{thm:overall-rate};
complete proofs appear in the appendix.

\subsection{Proof Sketch of Theorem~\ref{thm:action-value-error}}
The action-value  
estimation 
error is controlled through the Bellman residual,
%\subsubsection{Action-Value Function Estimation Error}\label{sec:bound Q error}
as illustrated in 
Lemma~\ref{lem:policy-evaluation-residual}. %bounds the action-value
%function error in terms of the Bellman residual.
Lemma~\ref{lem:excess-risk} then provides an excess-risk decomposition
of the Bellman residual into statistical and approximation error terms.
We control these two terms using empirical-process bounds for dependent
data (Lemma~\ref{lem:statistical-error}) and neural network approximation
bounds (Lemma~\ref{lem:app error}), respectively. 
 
\begin{lemma}
\label{lem:policy-evaluation-residual}
Fix a stochastic policy $\pi$ and suppose that $Q^\pi$ is the unique
fixed point of $\mathcal{T}^\pi$. Let $\mu$ be the sampling distribution
over $\mathcal X\times\mathcal A$. Suppose there exists
$C<\infty$ such that, for every admissible distribution
$\nu$ and every $t\ge 0$, the $t$-step state--action law starting from
$\nu$ and following $\pi$ is absolutely continuous with respect to $\mu$,
with
\begin{align*}
    \left\|
    \frac{\mathrm d(\nu (P^\pi)^t)}{\mathrm d\mu}
    \right\|_\infty
    \le C,
    \qquad \forall t\ge 0.
\end{align*}
Then, for any bounded measurable function $\widehat Q:\mathcal X\times\mathcal A\to\mathbb R$,
\begin{align*}
    \|\widehat Q-Q^\pi\|_{L^2(\nu)}^2
    \lesssim
    \frac{1}{(1-\gamma)^2}
    \|\widehat Q-\mathcal{T}^\pi \widehat Q\|_{L^2(\mu)}^2.
\end{align*}
\end{lemma}
Define the population and empirical critic risks by
\begin{align*}
\mathcal{L}_{k,\mathcal{U}_2}^{\mathrm{crit}}(Q)
:=\sup_{O\in \mathcal{U}_2}\mathcal{L}_k^{\mathrm{crit}}(Q,O),~
\mathcal{L}_{k,\mathcal{G}_2}^{\mathrm{crit}}(Q)
:=\sup_{O\in \mathcal{G}_2}\mathcal{L}_k^{\mathrm{crit}}(Q,O),~
\widehat{\mathcal{L}}_{k,\mathcal{G}_2}^{\mathrm{crit}}(Q)
:=\sup_{O\in \mathcal{G}_2}\widehat{\mathcal{L}}_k^{\mathrm{crit}}(Q,O),
\end{align*}
and define the excess risk by
\begin{align*}
   \mathcal{R}(\widehat Q_k)
:=
\mathcal{L}_{k,\mathcal{U}_2}^{\mathrm{crit}}(\widehat Q_k)-\mathcal{L}_{k,\mathcal{U}_2}^{\mathrm{crit}}(Q^{\pi_k})= \|\widehat Q_k-\mathcal{T}^{\pi_k} \widehat Q_k\|_{L^2(\mu_k)}^2.
\end{align*}
Lemma~\ref{lem:policy-evaluation-residual} therefore bounds the
action-value function error in terms of $\mathcal{R}(\widehat Q_k)$.

The following lemma separates the excess risk into statistical and
approximation errors.
\begin{lemma}[Excess-risk decomposition]\label{lem:excess-risk}
The excess risk satisfies
\begin{align*}
\mathcal{R}(\widehat Q_k)
\leq
\mathcal E_{\rm sta}
+\mathcal E_{\mathcal G_1}
+2\mathcal E_{\mathcal G_2},
\end{align*}
where
\begin{align*}
\mathcal E_{\mathcal G_1}
&:=
\inf_{\phi\in \mathcal{G}_1}\mathcal{L}_{k,\mathcal{U}_2}^{\mathrm{crit}}(\phi)-\mathcal{L}_{k,\mathcal{U}_2}^{\mathrm{crit}}(Q^{\pi_k}),\\
\mathcal E_{\mathcal G_2}
&:=
\sup_{\phi\in \mathcal{G}_1}\left|\mathcal{L}_{k,\mathcal{U}_2}^{\mathrm{crit}}(\phi)-\mathcal{L}_{k,\mathcal{G}_2}^{\mathrm{crit}}(\phi)\right|,\\
\mathcal E_{\rm sta}
&:=2 \sup_{\phi \in \mathcal{G}_1} |\widehat{\mathcal{L}}_{k,\mathcal{G}_2}^{\mathrm{crit}}(\phi) - \mathcal{L}_{k,\mathcal{G}_2}^{\mathrm{crit}}(\phi)|.
\end{align*}
\end{lemma}
%%%
We next bound the statistical component $\mathcal E_{\rm sta}$ and
the approximation components $\mathcal E_{\mathcal G_1}$ and
$\mathcal E_{\mathcal G_2}$.

\paragraph{Statistical Error
\texorpdfstring{$\mathcal E_{\rm sta}$.}{}}
\label{para:bound stat error} 
We use the independent-block technique for $\beta$-mixing sequences with
$2\xi_n$ consecutive full blocks of length $a_n$, where
$
    \xi_n=\left\lfloor\frac{n}{2a_n}\right\rfloor.
$
The odd and even blocks are treated separately and coupled with
independent blocks having the same blockwise marginals. The coupling
error is controlled by
$\xi_n\beta_{a_n}$.
Combining the independent-block inequality
(Lemma~\ref{lem:5 of antos}) with the network covering-number bound
(Lemma~\ref{lemma:covering number upperbound}) yields the following result.
\begin{lemma}[Statistical error bound]\label{lem:statistical-error}
Suppose that $\mathcal{G}_1$ and $\mathcal{G}_2$ are ReLU network classes
with the output and parameter bounds specified in Section~\ref{sec:preliminaries},
that Assumption~\ref{assump:critic-subGaussian-tail} holds, and that the data process
is strictly stationary and $\beta$-mixing.
Let $a_n$ be a positive integer satisfying $2a_n\le n$.
Set $\xi_n=\lfloor n/(2a_n)\rfloor$ and
$\mathfrak L_n=\log(e+B_{\rm par}B_{\rm net}\mathcal W\mathcal D\xi_n)$. Then,
\begin{align}
\mathbb{E}\sup_{\phi\in \mathcal{G}_1} \left| \widehat{\mathcal{L}}_{k,\mathcal{G}_2}^{\mathrm{crit}}(\phi) -\mathcal{L}_{k,\mathcal{G}_2}^{\mathrm{crit}}(\phi) \right| \lesssim %\widetilde{M} \left[ 
\sqrt{\frac{\mathcal{S}\mathcal{D}\mathfrak L_n}{\xi_n}} +\xi_n\beta_{a_n}. %\right]. 
\end{align}
%Here, $\widetilde{M}=4B_{\rm net}(2B_{\rm net}+R_{\max})$.
\end{lemma}

\paragraph{Approximation Error
\texorpdfstring{$\mathcal E_{\mathcal G_1},\mathcal E_{\mathcal G_2}$}
{approximation error}.}
\label{para:bound app error}

The two approximation terms satisfy
\begin{align*}
    \mathcal E_{\mathcal G_1}
&\le
\inf_{\phi \in \mathcal{G}_1} 2 \left \{ \|Q^{\pi_k} - \phi\|_{L^2({\mu_k})}^2 + \gamma^2\|Q^{\pi_k} - \phi\|_{L^2({\mu_k}P^{\pi_k})}^2  \right\}
\\
\mathcal E_{\mathcal G_2}
&\le
(4B_{\rm net}+2R_{\max})
\sup_{v\in\mathcal U_2}
\inf_{O\in\mathcal G_2}
\|O-v\|_{L^1(\mu_k)}.
\end{align*}
Applying Lemma~\ref{lem:th5-of-feng} on expanding compact sets and
controlling the tails yields the following bounds.
\begin{lemma}[Approximation error over an unbounded state--action space]
\label{lem:app error}
Suppose that Assumptions \ref{assump:critic-subGaussian-tail} and \ref{ass:completeness} hold and
$Q^{\pi_k}\in\mathcal H^\zeta(\mathbb R^d,B_{\rm H})$. Define
$
    \varepsilon_{\mathcal S}
    :=
    \left(
        \frac{\log(e\mathcal S)}{\mathcal S}
    \right)^{\zeta/d},
   ~
    d=d_{\mathcal X}+d_{\mathcal A}.
$
Suppose that the ReLU network classes $\mathcal G_1$ and $\mathcal G_2$
are constructed with
\[
    \mathcal D
    \lesssim
    \log\left(
        \frac{e}{\varepsilon_{\mathcal S}}
    \right),
    \qquad
    \mathcal S
    \lesssim
    \varepsilon_{\mathcal S}^{-d/\zeta}
    \log\left(
        \frac{e}{\varepsilon_{\mathcal S}}
    \right),
\]
and that their outputs are bounded by $B_{\rm net}$. Then
\begin{align}
    \mathcal E_{\mathcal G_1}
    &\lesssim
    (1+\gamma^2)B_{\rm H}^2
    \left(
        \frac{\log(e\mathcal S)}{\mathcal S}
    \right)^{2\zeta/d}
    \bigl(\log(e\mathcal S)\bigr)^\zeta,
    \label{eq:G1-unbounded-approximation}
\end{align}
and
\begin{align}
    \mathcal E_{\mathcal G_2}
    &\lesssim
    (4B_{\rm net}+2R_{\max})B_{\rm H}
    \left(
        \frac{\log(e\mathcal S)}{\mathcal S}
    \right)^{\zeta/d}
    \bigl(\log(e\mathcal S)\bigr)^{\zeta/2}.
    \label{eq:G2-unbounded-approximation}
\end{align}
\end{lemma}
The bounds in \eqref{eq:G1-unbounded-approximation}--\eqref{eq:G2-unbounded-approximation}, together with Lemmas~\ref{lem:excess-risk}--\ref{lem:statistical-error},
therefore yield a bound on the excess risk.

Combining the excess-risk bound with
Lemma~\ref{lem:policy-evaluation-residual} and the network sizes in
Theorem~\ref{thm:action-value-error} gives  
the convergence rate  of the
action-value estimation.
 
%the
%action-value estimation rate.

\subsection{Proof Sketch of Theorem~\ref{thm:conditional-sfs-sampling-error}}
\label{sec:proof-sketch-amortized-actor}
We first decompose the $\rm{KL}$ divergence between $\bar{\pi}_{k+1}$ and $\pi_{k+1}$ into drift estimation error and time discretization error. Then, we bound the time discretization error in Lemma~\ref{lem:kl-time-discretization}. For the drift estimation, we decompose the drift error into statistical and approximation errors in  Lemmas~\ref{lem:bridge-label-moment}-\ref{lem:learned-drift-error} and bound them by Lemma~\ref{lem:statistical-error} and Lemma~\ref{lem:actor-approximation-error}.
\paragraph{Error decomposition.}
Let $Z_t^s$ denote the exact target SF process and use the continuous
interpolation of the learned Euler scheme,
\begin{equation}
\mathrm d\widetilde Z_t^s
=b_{\phi_{k+1}}(s,\widetilde Z_{\underline t}^s,\underline t)
\,\mathrm dt+\mathrm dB_t,\qquad\widetilde Z_0^s=0,
\label{eq:continuous version  of sampling process}
\end{equation}
where $\underline t=t_j$ on $[t_j,t_{j+1})$.
The target drift along $Z_t^s$ is a square-integrable martingale.
Girsanov's formula, terminal data processing, and orthogonality of its
increments therefore give
\begin{align}
&\mathbb E_{S\sim d^{\pi^\star}}
\mathrm{KL}(\bar\pi_{k+1}(\cdot\mid S)\|\pi_{k+1}(\cdot\mid S))
\notag\\
&\quad\le\frac12\underbrace{\mathbb E_{S\sim d^{\pi^\star}}
\mathbb E\int_0^1
\|\bar b_{k+1,S}(Z_t^S,t)
-\bar b_{k+1,S}(Z_{\underline t}^S,\underline t)\|_2^2
\,\mathrm dt}_{\text{time discretization}}
\notag\\
&\qquad+\frac12\underbrace{\Delta t\sum_{j=0}^{T-1}
\mathbb E_{S\sim d^{\pi^\star}}\mathbb E
\|b_{\phi_{k+1}}(S,Z_{t_j}^S,t_j)
-\bar b_{k+1,S}(Z_{t_j}^S,t_j)\|_2^2}_{\text{drift estimation}}.
\label{eq:actor-drift-time-decomposition}
\end{align}
The inner expectations are taken under the exact target path law.

\paragraph{Time discretization error.}
\begin{lemma}[Time discretization error]
\label{lem:kl-time-discretization}
Under the frozen-network conditions of
Theorem~\ref{thm:conditional-sfs-sampling-error},
\begin{equation}
\mathbb E_{S\sim d^{\pi^\star}}\mathbb E
\left[\int_0^1
\|\bar b_{k+1,S}(Z_t^S,t)
-\bar b_{k+1,S}(Z_{\underline t}^S,\underline t)\|_2^2
\,\mathrm dt
\right]
\lesssim \frac{1}{T}.
\label{eq:kl-time-discretization}
\end{equation}
%The constant depends on the frozen actor and critic bounds, but not
%on the new grid size $T$.
\end{lemma}
\paragraph{Drift estimation: statistical and approximation errors.}
The target and reference paths have the same conditional Brownian
bridge given their endpoint. Changing only the endpoint and state laws
therefore gives, for any fitted $b$,
\[
\Delta t\sum_j\mathbb E_{S\sim d^{\pi^\star}}\mathbb E
\|b(S,Z_{t_j}^S,t_j)-\bar b_{k+1,S}(Z_{t_j}^S,t_j)\|_2^2
\le C\|b-\bar b_{k+1}\|_{\chi_k}^2.
\]
We next bound the regression error in this norm.
Let $P$ be expectation under the frozen regression design, $P_n$ its
sample average, and define the centered loss
\[
\ell_b(X,A):=\frac{\omega_k(S,A)}{d_{\mathcal A}}
\{\|b(X)\|_2^2-2b(X)\cdot U_k\}.
\]
Subtracting the label-only term preserves all excess losses.

\begin{lemma}[Bridge-label second moment]
\label{lem:bridge-label-moment}
For the regression design with finite endpoint second moment,
\[
P\|U_k\|_2^2=P\|A\|_2^2+d_{\mathcal A}(H_T-1),
\qquad H_T:=\sum_{r=1}^T\frac1r.
\]
Consequently, all centered losses are integrable for bounded networks,
and the population target risk is finite.
\end{lemma}

\begin{lemma}[Drift-regression error decomposition]
\label{lem:learned-drift-error}
For every regression sample, let $b_{\phi_{k+1}}$ minimize
$\widehat{\mathcal L}_{\rm actor}^k$ over $\mathcal G_3$. 
 Then
\begin{align}
\mathbb E\|b_{\phi_{k+1}}-\bar b_{k+1}\|_{\chi_k}^2
\le{}&\underbrace{2\mathbb E\sup_{b\in\mathcal G_3}
|(P_n-P)\ell_b|}_{\text{statistical error}}
+\underbrace{\inf_{b\in\mathcal G_3}
\|b-\bar b_{k+1}\|_{\chi_k}^2}_{\text{approximation error}}.
\label{eqe1}
\end{align}
\end{lemma}

For the statistical term, the centered-loss covering bound in
Lemma~\ref{lem:actor-statistical-error}, together with the bridge-label
moment, gives
\[
\mathbb E\sup_{b\in\mathcal G_3}|(P_n-P)\ell_b|
\lesssim
\widetilde{\mathcal O}\!\left(\sqrt{\frac{\mathcal S\mathcal D}{n}}\right).
\]
 
\paragraph{Approximation error.}
The regularity of $\bar b_{k+1}$ gives the neural approximation bound below.
\begin{lemma}[Actor approximation error]
\label{lem:actor-approximation-error}
Under Assumption~\ref{assump:actor-target-regularity} and the state-tail
and bounded frozen-design conditions above, let
$p=d_{\mathcal X}+d_{\mathcal A}+1$. There exists a ReLU DNN class
$\mathcal G_3$ of size at most $\mathcal S$, depth
$\mathcal D\lesssim\log(e\mathcal S)$, and fixed coordinatewise
output bound $B_{\rm net}\ge B_{\rm H}$ such that
\begin{equation}
\mathbb E\inf_{b\in\mathcal G_3}
\|b-\bar b_{k+1}\|_{\chi_k}^2
\lesssim d_{\mathcal A}B_{\rm H}^2
\left(\frac{\log(e\mathcal S)}{\mathcal S}\right)^{2\zeta/p}
\bigl(\log(e\mathcal S)\bigr)^\zeta.
\label{eq:sketch-actor-approximation-bound}
\end{equation}
\end{lemma}
The weighted norm satisfies
\[
\|g\|_{\chi_k}^2
\le\frac{e^{\|\widehat Q_k\|_\infty/\lambda}}{d_{\mathcal A}}
\mathbb E\|g(S,Y,t)\|_2^2.
\]
Bounded weights affect only the constant. The approximation exponent
is $2\zeta/p$, with no additional dependence on the time grid.

Substituting the statistical bound in  Lemma~\ref{lem:actor-statistical-error},
Lemma~\ref{lem:actor-approximation-error},
and \eqref{eq:kl-time-discretization} into
\eqref{eq:actor-drift-time-decomposition} completes the error bound.
Full proofs are in
Appendix~\ref{sec:proofs-amortized-actor}.

\subsection{Proof Sketch of Theorem~\ref{thm:overall-rate}}
\label{sec:proof-sketch-overall-revised}
Define the errors for the realized policies by
\begin{align*}
\mathcal E_k&:=
\max_{\substack{\nu\in\{d^{\pi^\star},d^{\bar\pi_{k+1}}\}\\
\pi\in\{\pi^\star,\pi_k,\bar\pi_{k+1}\}}}
\|\widehat Q_k-Q^{\pi_k}\|_{L^2(\nu\otimes\pi)},\\
\mathcal A_k&:=\mathbb E_{S\sim d^{\pi^\star}}
\mathrm{KL}(\bar\pi_{k+1}(\cdot\mid S)\|\pi_{k+1}(\cdot\mid S)).
\end{align*}
The KL projection identity and the performance-difference lemma give
\begin{align}
J(\pi^\star)-J(\pi_k)
\le{}&C_K\{J(\bar\pi_{k+1})-J(\pi_k)\}
+\frac{\lambda}{1-\gamma}\{D(\pi_k)-D(\bar\pi_{k+1})\}
 +\frac{2(C_K+1)}{1-\gamma}\mathcal E_k.
\label{eq:decpo}
\end{align}
We transfer the nonnegative estimated-improvement term between
occupancy measures, keeping $C_K$ fixed in the sum over $k$.
Pinsker's inequality, bounded rewards, and comparison coverage imply
\begin{align*}
|J(\bar\pi_{k+1})-J(\pi_{k+1})|
&\le\frac{\sqrt2 R_{\max}C_K}{(1-\gamma)^2}\sqrt{\mathcal A_k},\\
D(\pi_{k+1})-D(\bar\pi_{k+1})
&\le C_K\left(\mathcal A_k+\sqrt{\mathcal A_k/2}\right).
\end{align*}
Inserting $\pi_{k+1}$ into both differences in \eqref{eq:decpo}
produces telescoping sums. Use
$J(\pi_K)-J(\pi_0)\le2R_{\max}/(1-\gamma)$ and
$D(\pi_0)-D(\pi_K)\le D_0$ to obtain the pathwise bound
\begin{align}
&\frac1K\sum_{k=0}^{K-1}\{J(\pi^\star)-J(\pi_k)\}
\le \frac{2C_KR_{\max}+\lambda D_0}{K(1-\gamma)}\notag\\
&\quad+\frac1{K(1-\gamma)}\sum_{k=0}^{K-1}
\left[2(C_K+1)\mathcal E_k
+\frac{\sqrt2 R_{\max}C_K^2}{1-\gamma}\sqrt{\mathcal A_k}
+\lambda C_K\left(\mathcal A_k+\sqrt{\frac{\mathcal A_k}{2}}\right)
\right].
\label{eq:overall-amortized-rate}
\end{align}
 Assumption~\ref{ass:uniform-coverage-assumption} bounds
$C_K$ uniformly, and the conditions of Theorems~\ref{thm:action-value-error}
and~\ref{thm:conditional-sfs-sampling-error} hold at each iteration.
Theorem~\ref{thm:action-value-error}
and Cauchy--Schwarz give
\[
\mathbb E\mathcal E_k
\le\frac1{1-\gamma}
\widetilde{\mathcal O}\left(n^{-\zeta/(2d+4\zeta)}\right).
\]
Theorem~\ref{thm:conditional-sfs-sampling-error}, with the stated
choice of $T$, gives
\[
\mathbb E\mathcal A_k
\le\widetilde{\mathcal O}\!\left(n^{-2\zeta/(d+1+4\zeta)}\right),
\qquad
\mathbb E\sqrt{\mathcal A_k}
\le\widetilde{\mathcal O}\!\left(n^{-\zeta/(d+1+4\zeta)}\right).
\]
Substitution gives \eqref{eq:overall-explicit-rate}. The linear
$\mathcal A_k$ term decays faster and is absorbed into the actor term.
Since $d\ge1$, the actor rate is at least as fast as the critic rate.
The stated choice of $K$ then gives \eqref{eq:overall-final-rate}.

\section{Numerical Experiments}
\label{sec:experiment}

We evaluate SFAC on six offline-to-online control tasks and compare its
final scores after $500K$ online interactions with its offline
initialization and baseline results.
Sections~\ref{sec:experimental-setup} and~\ref{sec:offline-to-online-results}
give the setup and main results. Section~\ref{sec:ablation-studies}
examines the role of value weighting and the strength of the prescribed
policy tilt.

\subsection{Experimental Setup}
\label{sec:experimental-setup}

\paragraph{Tasks and metrics.}
We evaluate SFAC on six D4RL continuous-control tasks
\citep{fu2020d4rl}: the medium (M) and medium-replay (MR) variants of
HalfCheetah, Hopper, and Walker2d. We report D4RL normalized returns. The three
environments are illustrated in Figure~\ref{fig:evaluation-environments}, and
the evaluated dataset variants are listed in Table~\ref{tab:main-compare}.

\begin{figure}[H]
    \centering
    \subfigure[Walker2d]{%
        \includegraphics[height=3cm]{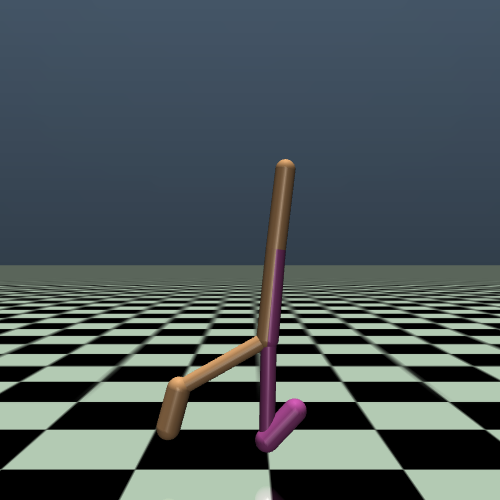}}
    \hfill
    \subfigure[Hopper]{%
        \includegraphics[height=3cm]{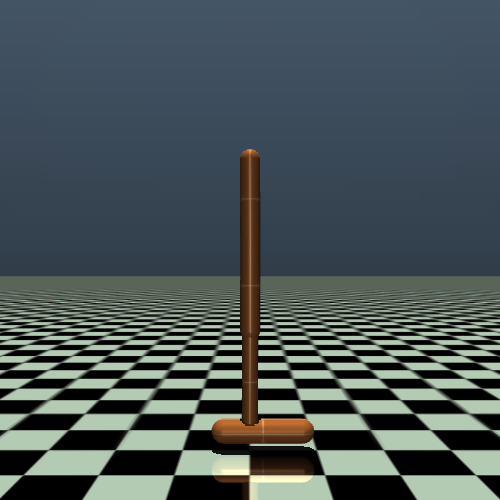}}
    \hfill
    \subfigure[HalfCheetah]{%
        \includegraphics[height=3cm]{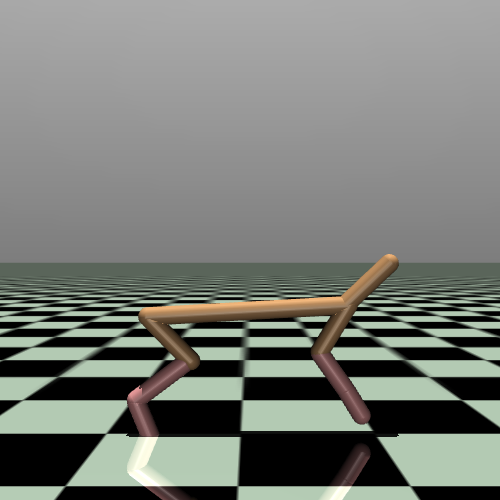}}
    \caption{Representative D4RL evaluation environments.}
    \label{fig:evaluation-environments}
\end{figure}

\paragraph{Reference methods.}
Table~\ref{tab:main-compare} includes results for
Cal-QL, a method developed for offline-to-online fine-tuning
\citep{nakamoto2023calql,liu2024energy,ball2023efficient}, together with the
offline final-policy scores of IQL, AWAC, Diffusion-QL, and IDQL
\citep{kostrikov2022offline,nair2020awac,wang2022diffusion,hansen2023idql}.
The table also includes DPMD as an online-only reference; its result for
each environment spans the medium and medium-replay rows because it does
not use either offline dataset. Baseline values come from the cited
studies and follow their respective evaluation protocols.
We separately compare SFAC's final score with its own offline
initialization under the protocol below.

\paragraph{Model and optimization.}
The conditional SF actor is an MLP with two hidden layers of width $256$
and ReLU activations. It is initialized by behavior cloning on the offline
dataset and updated online by $Q$-guided bridge regression. Actions are
generated with $T=32$ Euler--Maruyama steps. The actor learning rate is
$10^{-4}$. The critic is an ensemble of $10$ action-value functions trained
with a learning rate of $3\times10^{-5}$ and minibatches of size $1024$.
The discount factor is $\gamma=0.99$, and the KL temperature is
$\lambda=3.0$. The update uses the critic directly.

\paragraph{Online evaluation protocol.}
We run each online seed for $500K$ environment steps. We evaluate
the policy every $10K$ steps over $10$ episodes and report the mean D4RL
normalized return at each checkpoint. The primary online metric is the
score at the final $500K$ checkpoint. We report the mean and sample standard
deviation across 3 seeds initialized from the same task-specific offline
checkpoint.  

\subsection{Offline-to-Online Performance}
\label{sec:offline-to-online-results}

Table~\ref{tab:main-compare} reports SFAC's final-checkpoint scores
alongside its offline initialization and the baseline results. SFAC
exceeds its initialization and every listed offline final-policy score
on all six tasks. Its equally weighted mean increases from $40.20$
at initialization to $87.76$ after online learning. Compared with
Cal-QL, SFAC is higher on five tasks and lower on
HalfCheetah-MR ($55.26\pm3.35$ versus $58.70\pm0.93$). The
online-only DPMD results provide a further reference.
 
\begin{table}[H]
\centering
\caption{D4RL normalized returns for SFAC and baselines.}
\label{tab:main-compare}
\scriptsize
\setlength{\tabcolsep}{2.4pt}
\renewcommand{\arraystretch}{1.10}

\resizebox{\linewidth}{!}{%
\begin{tabular}{llcccccccc}
\toprule
& & \multicolumn{6}{c}{baselines}
& \multicolumn{2}{c}{SFAC} \\
\cmidrule(lr){3-8}
\cmidrule(lr){9-10}
Suite & Dataset
& IQL & AWAC & Cal-QL & Diffusion-QL & IDQL & DPMD (online)
& Offline init. & \textbf{Final score} \\
\midrule

\multirow{6}{*}{MuJoCo}
& HC-M
& $47.89\pm0.37$ & $49.58\pm0.36$ & $61.85\pm1.61$
& $49.04\pm0.77$ & $46.35\pm1.57$
& \multirow{2}{*}{$68.10$}
& $43.27$ & $\mathbf{70.24\pm2.05}$ \\
& HC-MR
& $42.70\pm1.12$ & $45.36\pm0.95$ & $58.70\pm0.93$
& $47.59\pm0.60$ & $41.86\pm0.51$
&
& $35.56$ & $\mathbf{55.26\pm3.35}$ \\

\addlinespace
& Hop-M
& $63.18\pm7.30$ & $64.20\pm5.61$ & $97.19\pm1.85$
& $70.86\pm19.07$ & $59.49\pm6.14$
& \multirow{2}{*}{$13.49$}
& $55.56$ & $\mathbf{104.96\pm0.18}$ \\
& Hop-MR
& $83.61\pm14.89$ & $98.83\pm1.32$ & $96.36\pm5.87$
& $95.11\pm19.11$ & $51.45\pm10.79$
&
& $19.62$ & $\mathbf{100.39\pm4.88}$ \\

\addlinespace
& Wal-M
& $77.75\pm5.18$ & $77.40\pm12.65$ & $80.38\pm14.56$
& $81.79\pm12.12$ & $67.23\pm2.72$
& \multirow{2}{*}{$90.14$}
& $81.15$ & $\mathbf{99.75\pm1.37}$ \\
& Wal-MR
& $82.39\pm10.06$ & $77.23\pm6.57$ & $92.52\pm1.38$
& $73.86\pm23.61$ & $69.80\pm4.96$
&
& $6.06$ & $\mathbf{95.95\pm3.80}$ \\

\midrule
& \textbf{Average}
& $66.25$ & $68.77$ & $81.17$ & $69.71$
& $56.03$ & -- & $40.20$ & $\mathbf{87.76}$ \\
\bottomrule
\end{tabular}%
}
\end{table}
 
\subsection{Ablation Studies}
\label{sec:ablation-studies}

\paragraph{Role of value weighting.}
We compare exponential endpoint weights with unit weights in two
settings. Unit weights assign every sampled action the same coefficient
in the bridge-regression loss, regardless of its value; minibatch
normalization gives $w_i=1/B$. At the population level, unit-weight
regression fits the SF drift of the reference policy, whereas value
weighting fits the drift of its prescribed exponential tilt. The
known-target experiment tests this distinction without critic error;
the online experiment tests its effect during policy iteration.

For $L\in\{2,4,6,8,10\}$, we use the same two-dimensional ring-mixture
family as in the tilt-strength experiment below. Its reference policy and
the known target at $\beta=1$ are
\[
\pi_{\mathrm{ref}}^{(L)}(a)
=\frac1L\sum_{j=1}^{L}\mathcal N(a;\mu_j^{(L)},0.18I_2),
\qquad
\bar\pi^{(L)}(a)\propto
\pi_{\mathrm{ref}}^{(L)}(a)e^{G(a)},
\]
where the means are
$\mu_j^{(L)}=2.2(\cos(2\pi(j-1)/L),\sin(2\pi(j-1)/L))^\top$.
The temperature is $\lambda=1$, and the common bounded score is
\[
G(a)=1.4e^{-\|a-(2.2,0)^\top\|_2^2/[2(0.55)^2]}
-0.9e^{-\|a-(-2.2,0)^\top\|_2^2/[2(0.70)^2]}.
\]
For each $L$, the weighted and unit-weight actors share their
initialization and training random stream. Each receives $1{,}500$
updates with minibatches of $512$ and $T=8$ diffusion steps. The
comparison uses one paired training seed and no learned critic.
Figure~\ref{fig:ablation-mechanism-seed1} places the five mode counts
from left to right. The upper row shows the tilted target and the two
actors' terminal probabilities in nearest-mean regions, estimated from
$5{,}000$ generated actions per actor. The lower row shows mean-square
error relative to the target SF drift at six diffusion times, evaluated
under the tilted target bridge; its vertical axis is logarithmic.
Across the five constructions, the weighted actor's terminal region
probabilities are closer to the target, and its drift error is lower at
the displayed times. This known-target comparison shows how value
weighting steers bridge regression toward the prescribed tilt.

\begin{figure}[H]
    \centering
    \includegraphics[width=0.94\linewidth]{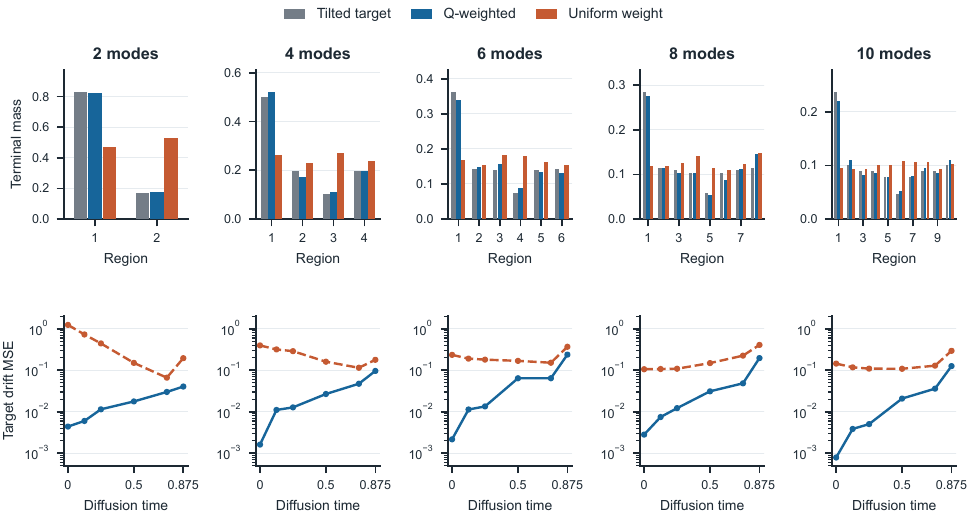}
    \caption{Role of value weighting in bridge regression.}
    \label{fig:ablation-mechanism-seed1}
\end{figure}

We then compare $Q$-weighted and unit-weight bridge regression on
HalfCheetah-medium-v2. The paired runs share the offline initialization
and online training settings, including continued critic updates; only
the actor weights differ. Figure~\ref{fig:ablation-online-seed1} reports
unsmoothed means over $10$ evaluation episodes at each checkpoint.
Both policies start at a D4RL normalized return of $43.40$. After
$500K$ online steps, the $Q$-weighted policy reaches $45.84$, a gain
of $2.44$ points, whereas the unit-weight policy reaches $0.014$.
The difference is consistent with the distinct population regression
targets.  

\begin{figure}[H]
    \centering
    \includegraphics[width=0.96\linewidth]{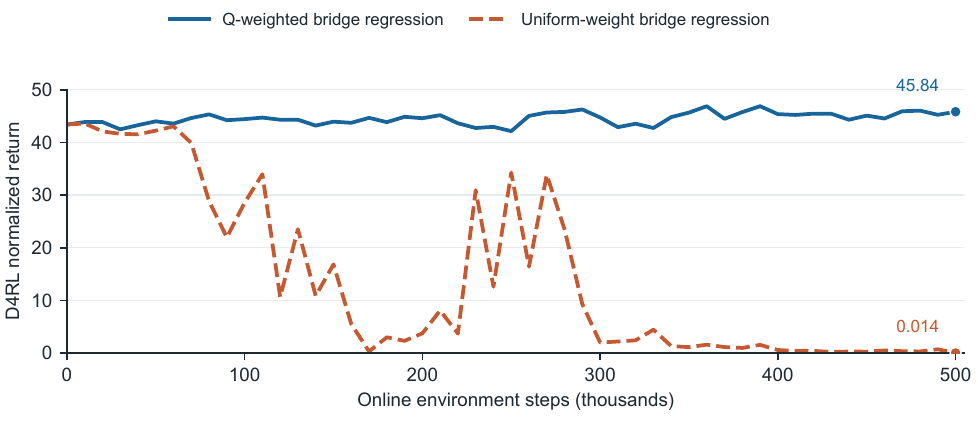}
    \caption{Online bridge-weight comparison on HalfCheetah-medium-v2.}
    \label{fig:ablation-online-seed1}
\end{figure}

\paragraph{Tilt strength and mode count.}
Using the same ring-mixture family and score $G$, we prescribe
\[
\bar\pi_{\beta}^{(L)}(a)\propto
\pi_{\mathrm{ref}}^{(L)}(a)\exp\{\beta G(a)\}.
\]
We vary $\beta\in\{0,0.5,1,2,4\}$ with
$T=8$, minibatches of $B=512$, and $1{,}500$ actor updates per
condition. At $\beta=0$, the weights are uniform and the target equals
the reference policy.

Figure~\ref{fig:ablation-strength-seed1} reports mode-mass total
variation (TV) and normalized weight effective sample size (ESS).
For endpoint scores $G_i$, the batch ESS is
$\mathrm{ESS}=(\sum_i e^{\beta G_i})^2/\sum_i e^{2\beta G_i}$.
Its mean fraction $\mathrm{ESS}/B$ falls as $\beta$ increases for
every mode count, ranging from $0.241$ for two modes to $0.064$ for
ten at $\beta=4$. The mode-mass TV has no common monotone pattern:
for ten modes it rises from $0.0257$ at $\beta=0$ to $0.1039$ at
$\beta=4$, whereas for six modes it falls from $0.0537$ at
$\beta=2$ to $0.0429$ at $\beta=4$. At the strongest tilt, the
target assigns $99.9\%$ and $97.1\%$ of its mass to one region for
two and four modes, respectively. Their small region-level TV cannot
establish accurate fitting within that region. Thus, stronger tilts
concentrate endpoint weights, while coarse mode-mass error shows no
common monotone trend.  Panel (a)
    reports TV between target and terminal probabilities in nearest-mean
    regions, estimated with $5{,}000$ terminal samples per condition.
    Panel (b) reports mean ESS divided by $B=512$ over $100$
    reference-endpoint minibatches.

\begin{figure}[H]
    \centering
    \includegraphics[width=0.98\linewidth]{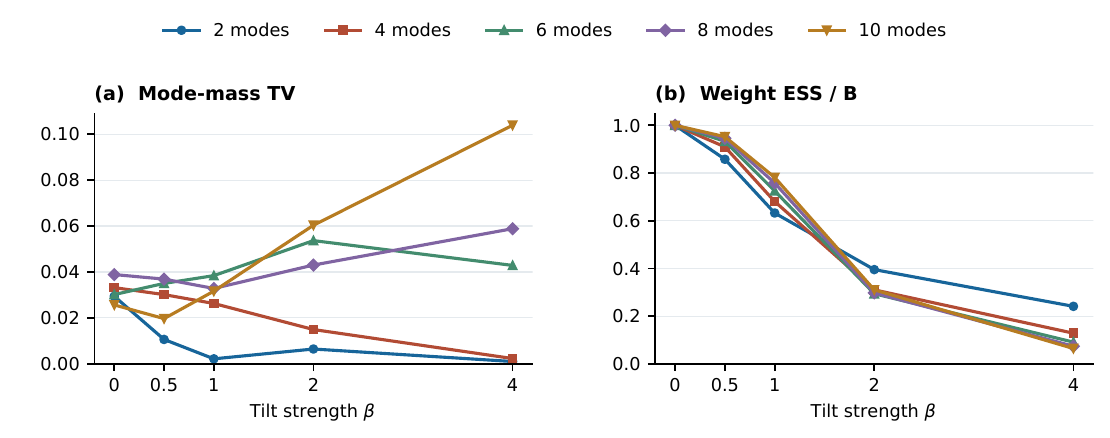}
    \caption{Tilt-strength ablation for equal-weight  mixtures with
    $L\in\{2,4,6,8,10\}$ modes.}
    \label{fig:ablation-strength-seed1}
\end{figure}
 
%%%%%%
\section{Conclusion}
\label{sec:conclusion}
We formulate KL-regularized policy improvement for diffusion actors using
conditional Schr\"odinger--F\"ollmer dynamics. A Doob $h$-transform converts
the exponential $Q$ tilt into a reference-drift correction, whose
conditional-mean representation yields a weighted bridge-regression
loss for the full drift, giving an actor update without directly approximating critic action
gradients. We also reveal that the update follows the natural
policy-gradient direction locally. Our finite-sample bound separates the
critic, weighted-regression, and Euler--Maruyama errors through
reverse KL. The   bound retains the statistical
and approximation errors, target Fisher information, and comparison
coefficients over policy iterations. After $500K$ online interactions,
the $Q$-guided bridge-regression runs exceed their offline initialization
and the listed offline RL scores on all six D4RL tasks. SFAC's
final score is also higher than Cal-QL on five of
the six tasks.

\section*{Appendix} \label{append}
The appendix provides the derivations and proofs supporting  
the theoretical results in the main text.
 Appendix~\ref{sec:npg-connection} illustrates the interpretation of SFAC in terms of a diffusion-based natural policy-gradient direction.
Appendix~\ref{sec:proof of Section conditional-sfs-method} derives the
conditional SF drift and its bridge-regression objective.
Appendices~\ref{sec:proof of Action-Value-Function Learning}
and~\ref{sec:proofs-amortized-actor} analyze action-value estimation and actor
approximation, respectively. Appendix~\ref{sec:proof-overall-revised}
combines these bounds in the policy-performance analysis under its stated
conditions. Appendix~\ref{app:auxiliary_lemmas} collects the
auxiliary results, and Appendix~\ref{sec:offline_alg} gives the offline
pretraining algorithm.

\appendix
 
 \renewcommand{\theequation}{\thesection.\arabic{equation}}
\numberwithin{equation}{section}
 
 \section{Connection to Natural Policy Gradient}
\label{sec:npg-connection}
The policy update in Section~\ref{sec:conditional-sfs-method}
is defined as a finite KL-regularized improvement step rather than as a
gradient step in the actor parameters. However, its local behavior
admits a direct actor--critic interpretation. The first-order variation of the ideal policy update follows
the natural policy-gradient direction. The corresponding Doob correction provides
a diffusion-based representation of the same update.

We next express the ideal policy improvement in
\eqref{eq:ideal_policy_improvement} as a natural policy-gradient update.
\begin{proposition}[Natural policy-gradient form of ideal policy improvement]
\label{prop:local-npg}

Suppose that $\pi_k^*=\pi_{\theta_k^*}, ~\pi_{k+1}^*=\pi_{\theta_{k+1}^*}$ for some $\theta_k^*,\theta_{k+1}^*\in\Theta$, with
$\theta_{k+1}^*\to\theta_k^*$ as $\lambda^{-1}\to0$, where
$\{\pi_\theta:\theta\in\Theta\}$ is a smooth parametric policy family.
Assume that $\theta\mapsto\log\pi_\theta(a\mid s)$ is $C^3$ in a
neighborhood of $\theta_k^*$, with the required derivatives
integrable under $d^{\pi_k^*}(s)\pi_k^*(a\mid s)$.
Assume further that
\[
F_k^*
:=
\mathbb E_{d^{\pi_k^*}\pi_k^*}
\left[
\nabla_\theta\log\pi_{\theta_k^*}(A\mid S)
\nabla_\theta\log\pi_{\theta_k^*}(A\mid S)^\top
\right]
\]
is positive definite. Then, as $\lambda^{-1}\to0$,
\[
\theta_{k+1}^*
=
\theta_k^*
+
\frac{1}{\lambda}
(F_k^*)^{-1}
\mathbb E_{d^{\pi_k^*}\pi_k^*}
\left[
Q^{\pi_k^*}(S,A)
\nabla_\theta
\log\pi_{\theta_k^*}(A\mid S)
\right]
+
O(\lambda^{-2}).
\]
\end{proposition}
Therefore, the finite KL-regularized update has the natural
policy-gradient direction as its first-order limit.
The preceding result interprets Doob's correction term as an approximation to the policy gradient.
\begin{proposition}[Local form of the Doob policy update]
\label{prop:local-doob-gradient}
Fix an iteration $k$ and a state $s$, and
suppose that the
conditional moments below are finite. Then, for every $t<1$,
as $\lambda^{-1}\to0$,
\[
\bar b_{k+1}(y,t;s)-b_k(y,t;s)
=
\frac{1}{\lambda(1-t)}
\operatorname{Cov}_{\mathbb P_k^s}
\left(
Y_1,\widehat Q_k(s,Y_1)
\mid Y_t=y
\right)
+
O(\lambda^{-2}).
\]
Whenever differentiation under the conditional expectation is valid,
\[
\bar b_{k+1}(y,t;s)-b_k(y,t;s)
=
\frac{1}{\lambda}
\nabla_y
\mathbb E_{\mathbb P_k^s}
\left[
\widehat Q_k(s,Y_1)
\mid Y_t=y
\right]
+
O(\lambda^{-2}).
\]
\end{proposition}
Proposition~\ref{prop:local-npg} shows that the
natural policy-gradient direction induced by the exact KL update is
represented in the SF process through a first-order Doob drift
correction.
\subsection{Proofs of the Connections to Natural Policy Gradient}
\begin{proof}[Proof of Proposition~\ref{prop:local-npg}]
Fix an iteration $k$. The $Q$-weighted update is
\[
\pi_{k+1}^*(a\mid s)
=
\frac{
\pi_k^*(a\mid s)
\exp\{Q^{\pi_k^*}(s,a)/\lambda\}
}{
Z_{k+1}^*(s)
},
\]
where
\[
Z_{k+1}^*(s)
=
\mathbb E_{A\sim\pi_k^*(\cdot\mid s)}
\left[
\exp\{Q^{\pi_k^*}(s,A)/\lambda\}
\right].
\]
Since $Q^{\pi_k^*}$ is uniformly bounded,
\[
\exp\{Q^{\pi_k^*}(s,a)/\lambda\}
=
1+\lambda^{-1}Q^{\pi_k^*}(s,a)+O(\lambda^{-2}).
\]
Hence
\[
Z_{k+1}^*(s)
=
1+\lambda^{-1}\mathbb E_{A\sim\pi_k^*(\cdot\mid s)}
[Q^{\pi_k^*}(s,A)]+O(\lambda^{-2}).
\]
Therefore,
\[
\frac{
\pi_{k+1}^*(a\mid s)
}{
\pi_k^*(a\mid s)
}
=
1+\lambda^{-1}\left(Q^{\pi_k^*}(s,a)
-\mathbb E_{A\sim\pi_k^*(\cdot\mid s)}Q^{\pi_k^*}(s,A)\right)
+O(\lambda^{-2}),
\]
and hence
\[
\pi_{k+1}^*(a\mid s)
=
\pi_k^*(a\mid s)
\left[
1+\lambda^{-1}\left(Q^{\pi_k^*}(s,a)
-\mathbb E_{A\sim\pi_k^*(\cdot\mid s)}Q^{\pi_k^*}(s,A)\right)
+O(\lambda^{-2})
\right].
\]
Let
\[
g_k^*(s,a)
:=
\nabla_\theta
\log\pi_{\theta_k^*}(a\mid s),
\qquad
\delta_k^*
:=
\theta_{k+1}^*-\theta_k^*.
\]
The first-order condition for the KL projection gives
\[
\mathbb E_{d^{\pi_k^*}\pi_{k+1}^*}
\left[
\nabla_\theta
\log\pi_{\theta_{k+1}^*}(A\mid S)
\right]
=0.
\]
Expanding around $\theta_k^*$ gives
\[
\nabla_\theta
\log\pi_{\theta_k^*+\delta_k^*}(A\mid S)
=
g_k^*(S,A)
+
H_k^*(S,A)\delta_k^*
+
O(\|\delta_k^*\|_2^2),
\]
where
$
H_k^*(S,A)
:=
\nabla_\theta^2
\log\pi_{\theta_k^*}(A\mid S).
$
%%%
The expansion of $\pi_{k+1}^*$ gives
\[
\mathbb E_{d^{\pi_k^*}\pi_{k+1}^*}
[g_k^*(S,A)]
=
\lambda^{-1}
\mathbb E_{d^{\pi_k^*}\pi_k^*}
\left[
Q^{\pi_k^*}(S,A)g_k^*(S,A)
\right]
+
O(\lambda^{-2}),
\]
since
\[
\mathbb E_{\pi_k^*(\cdot\mid s)}
[g_k^*(s,A)]
=0.
\]
Similarly,
\[
\mathbb E_{d^{\pi_k^*}\pi_{k+1}^*}
[H_k^*(S,A)]
=
\mathbb E_{d^{\pi_k^*}\pi_k^*}
[H_k^*(S,A)]
+
O(\lambda^{-1}).
\]
By the information identity,
\[
\mathbb E_{d^{\pi_k^*}\pi_k^*}
[H_k^*(S,A)]
=
-F_k^*.
\]
Hence,
\[
\begin{aligned}
0
={}&
\lambda^{-1}
\mathbb E_{d^{\pi_k^*}\pi_k^*}
\left[
Q^{\pi_k^*}(S,A)g_k^*(S,A)
\right]
-
F_k^*\delta_k^*
%\\
%&\quad
+
O(\lambda^{-2})
+
O(\lambda^{-1}\|\delta_k^*\|_2)
+
O(\|\delta_k^*\|_2^2).
\end{aligned}
\]
Since $F_k^*$ is positive definite and
$\delta_k^*\to0$, the local expansion yields
$
\|\delta_k^*\|_2=O(\lambda^{-1}).
$
Therefore,
\[
F_k^*\delta_k^*
=
\lambda^{-1}
\mathbb E_{d^{\pi_k^*}\pi_k^*}
\left[
Q^{\pi_k^*}(S,A)
\nabla_\theta\log\pi_{\theta_k^*}(A\mid S)
\right]
+
O(\lambda^{-2}),
\]
and thus
\[
\delta_k^*
=
\lambda^{-1}
(F_k^*)^{-1}
\mathbb E_{d^{\pi_k^*}\pi_k^*}
\left[
Q^{\pi_k^*}(S,A)
\nabla_\theta\log\pi_{\theta_k^*}(A\mid S)
\right]
+
O(\lambda^{-2}).
\]

Finally, since $d^{\pi_k^*}$ is the normalized discounted occupancy
measure, the policy-gradient theorem gives
\[
\nabla_\theta J(\pi_{\theta_k^*})
=
\frac{1}{1-\gamma}
\mathbb E_{d^{\pi_k^*}\pi_k^*}
\left[
Q^{\pi_k^*}(S,A)
\nabla_\theta\log\pi_{\theta_k^*}(A\mid S)
\right].
\]
Substituting this identity yields
\[
\theta_{k+1}^*
=
\theta_k^*
+
\frac{1-\gamma}{\lambda}
(F_k^*)^{-1}
\nabla_\theta J(\pi_{\theta_k^*})
+
O(\lambda^{-2}),
\]
which proves the claim.
\end{proof}
\begin{proof}[Proof of Proposition~\ref{prop:local-doob-gradient}]
Fix $(s,y,t)$ with $t<1$, and let
\[
P_{k}^{s,y,t}(da)
:=
\mathbb P_k^s(Y_1\in da\mid Y_t=y)
\]
denote the reference posterior law of the terminal action. Write
\[
\mu_0
:=
\mathbb E_{P_k^{s,y,t}}[Y_1].
\]
Under the Doob transform, the posterior law of $Y_1$ is exponentially
tilted:
\[
P_{k,\lambda}^{s,y,t}(da)
=
\frac{
    \exp\{\widehat Q_k(s,a)/\lambda\}
}{
    \mathbb E_{P_k^{s,y,t}}
    [\exp\{\widehat Q_k(s,Y_1)/\lambda\}]
}
P_k^{s,y,t}(da).
\]
Let
\[
\mu_\lambda
:=
\mathbb E_{P_{k,\lambda}^{s,y,t}}[Y_1].
\]
By the posterior-mean representation of the SF drift,
\[
b_k(y,t;s)
=
\frac{\mu_0-y}{1-t},
\qquad
\bar b_{k+1}(y,t;s)
=
\frac{\mu_\lambda-y}{1-t}.
\]
Therefore,
\[
\bar b_{k+1}(y,t;s)-b_k(y,t;s)
=
\frac{\mu_\lambda-\mu_0}{1-t}.
\]

We now expand $\mu_\lambda$. Since $\widehat Q_k$ is uniformly bounded,
\[
e^{\widehat Q_k(s,Y_1)/\lambda}
=
1+\lambda^{-1}\widehat Q_k(s,Y_1)+O(\lambda^{-2}),
\]
and consequently
\[
\mathbb E_{P_k^{s,y,t}}
\left[
Y_1 e^{\widehat Q_k(s,Y_1)/\lambda}
\right]
=
\mu_0
+
\lambda^{-1}
\mathbb E_{P_k^{s,y,t}}
\left[
Y_1 \widehat Q_k(s,Y_1)
\right]
+
O(\lambda^{-2}).
\]
Similarly,
\[
\mathbb E_{P_k^{s,y,t}}
\left[
e^{\widehat Q_k(s,Y_1)/\lambda}
\right]
=
1
+
\lambda^{-1}
\mathbb E_{P_k^{s,y,t}}
[\widehat Q_k(s,Y_1)]
+
O(\lambda^{-2}).
\]
Using
\[
(1+x)^{-1}=1-x+O(x^2),
\]
we obtain
\[
\begin{aligned}
\mu_\lambda
&=
\frac{
\mu_0
+
\lambda^{-1}
\mathbb E[Y_1\widehat Q_k(s,Y_1)]
+
O(\lambda^{-2})
}{
1
+
\lambda^{-1}
\mathbb E[\widehat Q_k(s,Y_1)]
+
O(\lambda^{-2})
}
\\
&=
\mu_0
+
\lambda^{-1}
\left\{
\mathbb E[Y_1\widehat Q_k(s,Y_1)]
-
\mu_0\,
\mathbb E[\widehat Q_k(s,Y_1)]
\right\}
+
O(\lambda^{-2})
\\
&=
\mu_0
+
\lambda^{-1}
\operatorname{Cov}
\left(
Y_1,\widehat Q_k(s,Y_1)
\mid Y_t=y
\right)
+
O(\lambda^{-2}),
\end{aligned}
\]
where all conditional expectations are under $\mathbb P_k^s$.
Substituting into the drift identity gives
\[
\bar b_{k+1}(y,t;s)-b_k(y,t;s)
=
\frac{\lambda^{-1}}{1-t}
\operatorname{Cov}_{\mathbb P_k^s}
\left(
Y_1,
\widehat Q_k(s,Y_1)
\mid Y_t=y
\right)
+
O(\lambda^{-2}).
\]
We express this covariance through the gradient of the conditional
action value. Under the Brownian-bridge representation,
the posterior density of $Y_1$ given $Y_t=y$ is proportional to
\[
\pi_k(a\mid s)
\exp\left\{
-\frac{\|y-ta\|_2^2}{2t(1-t)}
\right\}.
\]
Let $p_k(a\mid y,t,s)$ denote this posterior density. Differentiating
its logarithm gives
\[
\nabla_y
\log p_k(a\mid y,t,s)
=
\frac{
a-\mathbb E[Y_1\mid Y_t=y,s]
}{1-t}.
\]
Hence, for any integrable scalar function $\varphi$,
\[
\nabla_y
\mathbb E[
\varphi(Y_1)\mid Y_t=y,s
]
=
\frac{1}{1-t}
\operatorname{Cov}
\left(
Y_1,\varphi(Y_1)
\mid Y_t=y,s
\right),
\]
provided differentiation under the integral is valid. Taking
$
\varphi(a)=\widehat Q_k(s,a)
$
yields
\[
\nabla_y
\mathbb E
\left[
\widehat Q_k(s,Y_1)
\mid Y_t=y
\right]
=
\frac{1}{1-t}
\operatorname{Cov}
\left(
Y_1,
\widehat Q_k(s,Y_1)
\mid Y_t=y
\right),
\]
which proves the equivalent representation.
\end{proof}

\section{Proofs for Conditional Schr\"odinger--F\"ollmer Policy Improvement}
\label{sec:proof of Section conditional-sfs-method}

\subsection{Derivation of the Reference SDE Drift Term in \texorpdfstring{\eqref{eq:reference-sf}}{(\ref{eq:reference-sf})}}
\label{sec:proof of sf-sde1}
\begin{proof}
Fix $k$ and $s$, and let $f_{k,s}$ be the Gaussian density ratio of
$\pi_k(\cdot\mid s)$. The heat-semigroup representation of the reference
drift is
\[
b_k(y,t;s)=\nabla_y\log Q_{1-t}f_{k,s}(y),
\qquad 0<t<1,
\]
where
\[
Q_u f_{k,s}(y)
=\mathbb E_{Z\sim\mathcal N(0,I_{d_{\mathcal A}})}
\left[f_{k,s}(y+\sqrt{u}Z)\right].
\]
Let $\varphi_{y,u}$ denote the density of
$\mathcal N(y,uI_{d_{\mathcal A}})$ for $u>0$. The Brownian-bridge
representation \citep{mikulincer2024brownian} gives
\begin{align*}
p_t(y\mid s)
&=\varphi_{0,t}(y)Q_{1-t}f_{k,s}(y)\\
&=\int (2\pi t(1-t))^{-d_{\mathcal A}/2}
\exp\!\left(-\frac{\|y-ty_1\|_2^2}{2t(1-t)}\right)
\pi_k(y_1\mid s)\,\mathrm dy_1.
\end{align*}
This is the density of $tY_1+\sqrt{t(1-t)}Z$, where
$Y_1\sim\pi_k(\cdot\mid s)$ and $Z$ are independent. Taking the
logarithmic gradient of the first equality yields
\[
\nabla_y\log p_t(y\mid s)
=-\frac{y}{t}+\nabla_y\log Q_{1-t}f_{k,s}(y),
\]
which proves
\eqref{eq:reference-sf}.
\end{proof}
\subsection{Derivation of the Doob h-transform SDE Drift Term in \texorpdfstring{\eqref{eq:doob-drift-density-ratio}}{(\ref{eq:doob-drift-density-ratio})}}
\label{sec:proof of lemma doobs-drift}
\begin{proof}
The $Q$-tilted policy has the Gaussian density ratio
\begin{equation*}
    \bar f_{k+1,s}(y)
    =
    \frac{\mathrm d\bar\pi_{k+1}(\cdot\mid s)}{\mathrm d \gamma_{d_{\mathcal{A}}}}(y)
    =
    \frac{\omega_k(s,y)f_{k,s}(y)}{C},
\end{equation*}
where $\omega_k(s,y)=\exp\{\widehat Q_k(s,y)/\lambda\}$ and $C$ is the
normalizing constant.
For $0<t<1$, the conditional expectation defining $h_k$ can therefore
be written as
\begin{align*}
h_k(y,t;s)
&=\frac{Q_{1-t}\!\left[\omega_k(s,\cdot)f_{k,s}\right](y)}
{Q_{1-t}f_{k,s}(y)}=C
\frac{Q_{1-t}\bar f_{k+1,s}(y)}{Q_{1-t}f_{k,s}(y)}.
\end{align*}
Since $C$ does not depend on $y$, taking logarithmic
gradients gives
\[
\bar b_{k+1,s}(y,t)
=\nabla_y\log Q_{1-t}\bar f_{k+1,s}(y)
=b_k(y,t;s)+\nabla_y\log h_k(y,t;s),
\]
which proves \eqref{eq:doob-drift-density-ratio}.
\end{proof}

\subsection{Derivation of the bridge-regression objective}
\label{sec:proof of lemma drift mc}
\begin{proof}
Fix $s$ and $0<t<1$. Let $Y_1\sim\pi_k(\cdot\mid s)$ be the
reference endpoint. The Brownian-bridge representation gives
\[
Y_t\mid Y_1=y_1,s
\sim\mathcal N(ty_1,t(1-t)I_{d_{\mathcal A}}).
\]
Write its likelihood, up to a factor independent of $y_1$, as
\[
K(y,y_1)
:=\exp\!\left(-\frac{\|y-ty_1\|_2^2}{2t(1-t)}\right).
\]
Bayes' rule gives the reference endpoint posterior
\[
p(y_1\mid Y_t=y,s)
=\frac{K(y,y_1)\pi_k(y_1\mid s)}
{\int K(y,u)\pi_k(u\mid s)\,\mathrm du}.
\]
Let $\mathbb E_p$ denote expectation under this posterior for fixed
$s,y,t$. Differentiating the Gaussian likelihood yields
\[
\nabla_y\log K(y,y_1)
=-\frac{y-ty_1}{t(1-t)},
\qquad
\nabla_y\log p_t(y\mid s)
=\mathbb E_p[\nabla_y\log K(y,Y_1)].
\]
Consequently,
\begin{equation}
b_k(y,t;s)
=\frac{y}{t}+\nabla_y\log p_t(y\mid s)
=\frac{\mathbb E_p[Y_1]-y}{1-t}.
\label{eq:reference-tweedie}
\end{equation}

For the $Q$-tilted posterior, define
\[
q(y_1\mid y,s)
:=\frac{\omega_k(s,y_1)p(y_1\mid Y_t=y,s)}
{\mathbb E_p[\omega_k(s,Y_1)]},
\]
and write $\mathbb E_q$ for expectation under $q$.
Since $h_k(y,t;s)=\mathbb E_p[\omega_k(s,Y_1)]$,
\begin{equation}
\nabla_y\log h_k(y,t;s)
=\frac{\int\omega_k(s,y_1)
\nabla_y p(y_1\mid Y_t=y,s)\,\mathrm dy_1}
{\mathbb E_p[\omega_k(s,Y_1)]}.
\label{nabla h_1}
\end{equation}
Differentiating the posterior density gives
\[
\nabla_y p(y_1\mid Y_t=y,s)
=p(y_1\mid Y_t=y,s)
\left\{\nabla_y\log K(y,y_1)
-\mathbb E_p[\nabla_y\log K(y,Y_1)]\right\}.
\]
Substitution into \eqref{nabla h_1} yields
\begin{align*}
\nabla_y\log h_k(y,t;s)
&=\mathbb E_q[\nabla_y\log K(y,Y_1)]
-\mathbb E_p[\nabla_y\log K(y,Y_1)]\\
&=\frac{\mathbb E_q[Y_1]-\mathbb E_p[Y_1]}{1-t}.
\end{align*}
Combining this identity with \eqref{eq:reference-tweedie} gives
\begin{equation*}
\begin{aligned}
\bar b_{k+1,s}(y,t)
&=\frac{\mathbb E_p[Y_1]-y}{1-t}
+\frac{\mathbb E_q[Y_1]-\mathbb E_p[Y_1]}{1-t}\\
&=\frac{\mathbb E_q[Y_1]-y}{1-t}.
\end{aligned}
\end{equation*}
Thus, the Doob $h$-transform replaces the reference posterior mean
with its $Q$-tilted counterpart.
Bayes' rule expresses the two posterior means as
\begin{align*}
\mathbb E_q[Y_1]
&=\frac{\mathbb E_{Y_1\sim\pi_k(\cdot\mid s)}
[Y_1\omega_k(s,Y_1)K(y,Y_1)]}
{\mathbb E_{Y_1\sim\pi_k(\cdot\mid s)}
[\omega_k(s,Y_1)K(y,Y_1)]},\\
\mathbb E_p[Y_1]
&=\frac{\mathbb E_{Y_1\sim\pi_k(\cdot\mid s)}
[Y_1K(y,Y_1)]}
{\mathbb E_{Y_1\sim\pi_k(\cdot\mid s)}
[K(y,Y_1)]}.
\end{align*}
Substituting the definition of $q$ into the target drift gives
\[
\bar b_{k+1,s}(y,t)
=\mathbb E_q\!\left[\frac{Y_1-y}{1-t}\right]
=\frac{\mathbb E_p\!\left[
\omega_k(s,Y_1)\frac{Y_1-y}{1-t}\right]}
{\mathbb E_p[\omega_k(s,Y_1)]}.
\]
Multiplication by the denominator yields
\[
\mathbb E_p\!\left[
\omega_k(s,Y_1)
\left\{\frac{Y_1-y}{1-t}-\bar b_{k+1,s}(y,t)\right\}
\right]=0.
\]
For any $b$, its difference from the target drift is fixed under this
conditional expectation. Hence
\begin{align*}
&\mathbb E_p\!\left[\omega_k(s,Y_1)
\left\langle b(s,y,t)-\bar b_{k+1,s}(y,t),\,
\frac{Y_1-y}{1-t}-\bar b_{k+1,s}(y,t)\right\rangle\right]\\
&\quad=\left\langle b(s,y,t)-\bar b_{k+1,s}(y,t),\,
\mathbb E_p\!\left[\omega_k(s,Y_1)
\left\{\frac{Y_1-y}{1-t}-\bar b_{k+1,s}(y,t)\right\}\right]
\right\rangle=0.
\end{align*}
Expanding the square, dividing by $d_{\mathcal A}$, and averaging over
the reference states, bridge positions, and grid times therefore gives
\[
\mathcal L_{\mathrm{actor}}^k(b)
-\mathcal L_{\mathrm{actor}}^k(\bar b_{k+1})
=\mathbb E\!\left[
\frac{\omega_k(s,a)}{d_{\mathcal A}}
\|b(s,y,t)-\bar b_{k+1,s}(y,t)\|_2^2\right].
\]
The term $\mathcal L_{\mathrm{actor}}^k(\bar b_{k+1})$ is independent
of $b$, so the weighted regression loss has the target drift as its
unrestricted minimizer. Individual bridge labels provide the training
targets; their conditional mean need not be estimated at each input.

At $t=0$, $y=0$, the reference endpoint posterior is
$\pi_k(\cdot\mid s)$, and its weighted version is
$\bar\pi_{k+1}(\cdot\mid s)$. The same identity holds with bridge label
$Y_1$. The endpoint $t=1$ is excluded from the regression grid.
The second-moment condition needed for the squared-loss expansion on
this finite grid is established in Lemma~\ref{lem:bridge-label-moment}.
\end{proof}

\section{Proofs for Action-Value-Function Learning}\label{sec:proof of Action-Value-Function Learning}
\subsection{Proof of Lemma~\ref{lem:excess-risk}}
\begin{proof}
For any $\phi \in \mathcal{G}_1$, decompose the excess risk as
\begin{align*}
\mathcal{L}_{k,\mathcal{U}_2}^{\mathrm{crit}}(\widehat Q_k)-\mathcal{L}_{k,\mathcal{U}_2}^{\mathrm{crit}}(Q^{\pi_k})
=\;&
\bigl[\mathcal{L}_{k,\mathcal{U}_2}^{\mathrm{crit}}(\widehat Q_k)-\mathcal{L}_{k,\mathcal{G}_2}^{\mathrm{crit}}(\widehat Q_k)\bigr]
+\bigl[\mathcal{L}_{k,\mathcal{G}_2}^{\mathrm{crit}}(\widehat Q_k)-\widehat{\mathcal{L}}_{k,\mathcal{G}_2}^{\mathrm{crit}}(\widehat Q_k)\bigr] \\
&+\bigl[\widehat{\mathcal{L}}_{k,\mathcal{G}_2}^{\mathrm{crit}}(\widehat Q_k)-\widehat{\mathcal{L}}_{k,\mathcal{G}_2}^{\mathrm{crit}}(\phi)\bigr]
+\bigl[\widehat{\mathcal{L}}_{k,\mathcal{G}_2}^{\mathrm{crit}}(\phi)-\mathcal{L}_{k,\mathcal{G}_2}^{\mathrm{crit}}(\phi)\bigr] \\
&+\bigl[\mathcal{L}_{k,\mathcal{G}_2}^{\mathrm{crit}}(\phi)-\mathcal{L}_{k,\mathcal{U}_2}^{\mathrm{crit}}(\phi)\bigr]
+\bigl[\mathcal{L}_{k,\mathcal{U}_2}^{\mathrm{crit}}(\phi)-\mathcal{L}_{k,\mathcal{U}_2}^{\mathrm{crit}}(Q^{\pi_k})\bigr].
\end{align*}
Here, the third term is nonpositive by empirical minimax optimality,
and the second and fourth terms are each bounded by
\[
\sup_{\psi\in \mathcal{G}_1}
\left|
\widehat{\mathcal{L}}_{k,\mathcal{G}_2}^{\mathrm{crit}}(\psi)-\mathcal{L}_{k,\mathcal{G}_2}^{\mathrm{crit}}(\psi)
\right|,
\]
while the first and fifth terms are each bounded by
\[
\sup_{\psi\in \mathcal{G}_1}
\left|
\mathcal{L}_{k,\mathcal{U}_2}^{\mathrm{crit}}(\psi)-\mathcal{L}_{k,\mathcal{G}_2}^{\mathrm{crit}}(\psi)
\right|.
\]
Taking the infimum over $\phi\in\mathcal G_1$ therefore gives
\begin{align*}
\mathcal{L}_{k,\mathcal{U}_2}^{\mathrm{crit}}(\widehat Q_k)-\mathcal{L}_{k,\mathcal{U}_2}^{\mathrm{crit}}(Q^{\pi_k})
\le\;&
2\sup_{\psi\in \mathcal{G}_1}
\left|
\widehat{\mathcal{L}}_{k,\mathcal{G}_2}^{\mathrm{crit}}(\psi)-\mathcal{L}_{k,\mathcal{G}_2}^{\mathrm{crit}}(\psi)
\right| \\
&+
2\sup_{\psi\in \mathcal{G}_1}
\left|
\mathcal{L}_{k,\mathcal{U}_2}^{\mathrm{crit}}(\psi)-\mathcal{L}_{k,\mathcal{G}_2}^{\mathrm{crit}}(\psi)
\right| \\
&+
\inf_{\phi\in \mathcal{G}_1}
\left\{
\mathcal{L}_{k,\mathcal{U}_2}^{\mathrm{crit}}(\phi)-\mathcal{L}_{k,\mathcal{U}_2}^{\mathrm{crit}}(Q^{\pi_k})
\right\},
\end{align*}
which is the claimed decomposition.
\end{proof}
\subsection{Proof of Lemma~\ref{lem:statistical-error}}
\begin{proof}
Write $Z_i:=(s_i,a_i,r_i,s_i',a_i')$, where
$a_i'\sim\pi_k(\cdot\mid s_i')$ is sampled independently conditional on
the observed transitions. These independent conditional draws preserve
stationarity and do not increase the mixing coefficients. Define
\[
\ell_{Q,O}(s,a,r,s',a')
:=
\bigl(Q(s,a)-r-\gamma Q(s',a')\bigr)^2
-
\bigl(O(s,a)-r-\gamma Q(s',a')\bigr)^2.
\]
Set
 
\[
\mathbb{Z}_n:=
\sup_{Q\in\mathcal G_1,\,O\in\mathcal G_2}
\left|
\frac1n\sum_{i=1}^n\ell_{Q,O}(Z_i)
-\mathbb E\ell_{Q,O}(Z_1)
\right|.
\]
The inequality
$|\sup_O u_O-\sup_O v_O|\le\sup_O|u_O-v_O|$ gives
\[
\sup_{\phi\in\mathcal G_1}
\left|
\widehat{\mathcal L}_{k,\mathcal G_2}^{\mathrm{crit}}(\phi)
-\mathcal L_{k,\mathcal G_2}^{\mathrm{crit}}(\phi)
\right|
\le \mathbb{Z}_n.
\]
The loss envelope is
$\widetilde M=4B_{\rm net}(2B_{\rm net}+R_{\max})$. Moreover, with
$\lambda_{\rm lip}=12B_{\rm net}+6R_{\max}$,
\[
\|\ell_{Q_1,O_1}-\ell_{Q_2,O_2}\|_\infty
\le
\lambda_{\rm lip}
\bigl(\|Q_1-Q_2\|_\infty+\|O_1-O_2\|_\infty\bigr).
\]
To apply the compact-domain covering bound in
Lemma~\ref{lemma:covering number upperbound}, let
$\mathcal Z_R=[-R,R]^d$ and truncate the loss:
\[
\ell^R_{Q,O}(Z)
=
\ell_{Q,O}(Z)
\boldsymbol 1\{(s,a),(s',a')\in\mathcal Z_R\}.
\]
Write $\mathcal L_R:=\{\ell^R_{Q,O}:Q\in\mathcal G_1,\ O\in\mathcal G_2\}$.
Assumption~\ref{assump:critic-subGaussian-tail} and a union bound give
\[
\mathbb P\bigl((s,a)\notin\mathcal Z_R
\ \text{or}\ (s',a')\notin\mathcal Z_R\bigr)
\le 2C_\mu e^{-c_\mu R^2}.
\]
Choose
$
R_n=
\max\left\{1,\,
\left[c_\mu^{-1}\log(2eC_\mu\xi_n^2)\right]^{1/2}\right\}.
$
Thus the truncation probability is at most $\xi_n^{-2}$.

Set $m=2a_n\xi_n$ and $r=n-m$, so that $0\le r<2a_n$.
For
\[
Z_{m,R}
:=
\sup_{Q,O}
\left|
\frac1m\sum_{i=1}^m\ell^R_{Q,O}(Z_i)
-\mathbb E\ell^R_{Q,O}(Z_1)
\right|,
\]
boundedness and stationarity imply
\begin{align*}
\mathbb E \mathbb{Z}_n
&\le
\mathbb E Z_{m,R_n}
+\frac{2\widetilde M r}{n}
+2\widetilde M\xi_n^{-2}\le
\mathbb E Z_{m,R_n}
+\frac{2\widetilde M}{\xi_n}
+\frac{2\widetilde M}{\xi_n^2}.
\end{align*}
The last two terms account for the incomplete block and the truncated tails.

Rescaling $\mathcal Z_R$ to $[0,1]^d$ multiplies the first-layer
parameter bound by at most $C_dR$. The Lipschitz bound for the loss and a product
cover for the two network classes therefore yield
\begin{align*}
\log\mathcal N(\varepsilon,\mathcal L_R,\|\cdot\|_\infty)
&\le
C_0\mathcal S\mathcal D
\log\left(
e+\frac{C_1R B_{\rm par}\mathcal W\mathcal D
\lambda_{\rm lip}}{\varepsilon}
\right).
\end{align*}
Here, the supremum norm is global for the truncated losses: outside the
common truncation set, every member is zero. 

Apply Lemma~\ref{lem:5 of antos} to the $2\xi_n$ complete blocks.
For $0<\varepsilon\le2\widetilde M$,
\begin{align}
\mathbb P(Z_{m,R_n}>\varepsilon)
&\le
16\exp\Biggl\{
C_0\mathcal S\mathcal D
\log\left(
e+\frac{C_1R_n B_{\rm par}\mathcal W\mathcal D
\lambda_{\rm lip}}{\varepsilon}
\right)
-\frac{c_0\xi_n\varepsilon^2}{\widetilde M^2}
\Biggr\}
+2\xi_n\beta_{a_n}.
\label{eq:blocking-tail-app}
\end{align}
Let
$
\delta_n=C_2\widetilde M
\sqrt{\frac{\mathcal S\mathcal D\mathfrak L_n}{\xi_n}},
$
where $\mathfrak L_n=
\log(e+B_{\rm par}B_{\rm net}\mathcal W\mathcal D\xi_n)$.
Since $R_n\lesssim\sqrt{\log(e\xi_n)}$, the entropy logarithm is bounded
by $C\mathfrak L_n$ whenever $\varepsilon\ge\delta_n$.
If $\delta_n\ge2\widetilde M$, the result follows from the envelope.
Otherwise, choosing $C_2$ sufficiently large makes the exponent in
\eqref{eq:blocking-tail-app} at most
$-c\xi_n\varepsilon^2/\widetilde M^2$ for
$\varepsilon\ge\delta_n$. Integrating the tail bound gives
\begin{align*}
\mathbb E Z_{m,R_n}
&\le
\delta_n+
16\int_{\delta_n}^{2\widetilde M}
e^{-c\xi_n\varepsilon^2/\widetilde M^2}\,\mathrm d\varepsilon
+4\widetilde M\xi_n\beta_{a_n}\\
&\le
C\widetilde M
\left[
\sqrt{\frac{\mathcal S\mathcal D\mathfrak L_n}{\xi_n}}
+\xi_n\beta_{a_n}
\right].
\end{align*}
The truncation and incomplete-block terms are absorbed by the first
term because $\mathcal S,\mathcal D,\mathfrak L_n\ge1$.
The initial supremum inequality now proves the claim.
\end{proof}

\subsection{Proof of Lemma~\ref{lem:app error}}
\begin{proof}
Since
$\mathcal T^{\pi_k}\phi\in\mathcal U_2$ for $\phi\in\mathcal G_1$
and $Q^{\pi_k}\in\mathcal U_2$, the population risk is the squared
Bellman residual. The Bellman identity and Jensen's inequality give
\begin{align*}
\mathcal E_{\mathcal G_1}
&=\inf_{\phi\in\mathcal G_1}
\|\phi-\mathcal T^{\pi_k}\phi\|_{L^2(\mu_k)}^2\\
&\le
2\inf_{\phi\in\mathcal G_1}
\left\{
\|Q^{\pi_k}-\phi\|_{L^2(\mu_k)}^2+
\gamma^2\|P^{\pi_k}(Q^{\pi_k}-\phi)\|_{L^2(\mu_k)}^2
\right\}\\
&\le
2\inf_{\phi\in\mathcal G_1}
\left\{
\|Q^{\pi_k}-\phi\|_{L^2(\mu_k)}^2+
\gamma^2\|Q^{\pi_k}-\phi\|_{L^2(\mu_kP^{\pi_k})}^2
\right\}.
\end{align*}
Writing
$\widetilde\mu_k:=(\mu_k+\gamma^2\mu_kP^{\pi_k})/(1+\gamma^2)$,
we obtain
\begin{equation}
\mathcal E_{\mathcal G_1}
\le
2(1+\gamma^2)
\inf_{\phi\in\mathcal G_1}
\|Q^{\pi_k}-\phi\|_{L^2(\widetilde\mu_k)}^2.
\label{eq:G1-mixture-final}
\end{equation}
For $R\ge1$, let $\mathcal Z_R=[-R,R]^d$.
Under the rescaling $z=2Ru-R\boldsymbol 1_d$, $u\in[0,1]^d$,
the H\"older norm of the rescaled action-value function is at most
$C_{\zeta,d}B_{\rm H}R^\zeta$. Indeed, a derivative of order $r$
contributes $(2R)^r$, and the highest-order H\"older seminorm
contributes the additional factor $(2R)^\alpha$, where
$\zeta=r_\zeta+\alpha$. %

For $\epsilon\in(0,1)$, Lemma~\ref{lem:th5-of-feng} therefore gives
a ReLU network $\overline\phi_{k,R,\epsilon}$ satisfying
\[
\sup_{z\in\mathcal Z_R}
|Q^{\pi_k}(z)-\overline\phi_{k,R,\epsilon}(z)|
\le C_{\zeta,d}B_{\rm H}R^\zeta\epsilon,
\]
with depth and size bounded by
$
\mathcal D\lesssim\log(e/\epsilon)$
and 
$
\mathcal S\lesssim\epsilon^{-d/\zeta}\log(e/\epsilon).
$
Define
\[
\phi_{k,R,\epsilon}
=
\operatorname{clip}_{[-B_{\rm H},B_{\rm H}]}
(\overline\phi_{k,R,\epsilon}),
\]
where
\[
\operatorname{clip}_{[-B_{\rm H},B_{\rm H}]}(u)
=
-B_{\rm H}+\sigma(u+B_{\rm H})-\sigma(u-B_{\rm H}).
\]
Clipping requires one additional ReLU layer and a constant number of
parameters. Since $\|Q^{\pi_k}\|_\infty\le B_{\rm H}$, it cannot
increase the approximation error and ensures
$|Q^{\pi_k}-\phi_{k,R,\epsilon}|\le2B_{\rm H}$ globally.
The network belongs to $\mathcal G_1$ when its size and parameter
bounds are chosen as stipulated in the lemma.

Assumption~\ref{assump:critic-subGaussian-tail} implies
\[
\widetilde\mu_k(\mathcal Z_R^c)
\le
e^{-c_\mu R^2}
\int_{\mathbb R^d}e^{c_\mu\|z\|_2^2}\widetilde\mu_k(\mathrm dz)
\le C_\mu e^{-c_\mu R^2}.
\]
Splitting the approximation error over $\mathcal Z_R$ and its
complement gives
\begin{equation}
\|Q^{\pi_k}-\phi_{k,R,\epsilon}\|_{L^2(\widetilde\mu_k)}^2
\le
C_{\zeta,d}B_{\rm H}^2R^{2\zeta}\epsilon^2
+4B_{\rm H}^2C_\mu e^{-c_\mu R^2}.
\label{eq:Q-truncation-bound}
\end{equation}
Choose
$
R_\epsilon
=
\max\left\{1,\,
\left(\frac{2}{c_\mu}\log\frac e\epsilon\right)^{1/2}\right\}.
$
Then
\[
e^{-c_\mu R_\epsilon^2}\le(\epsilon/e)^2,
\qquad
R_\epsilon^{2\zeta}\lesssim\bigl(\log(e/\epsilon)\bigr)^\zeta.
\]
{ 
Substituting into \eqref{eq:Q-truncation-bound} and  
\eqref{eq:G1-mixture-final} yields
}
\begin{equation}
\mathcal E_{\mathcal G_1}
\lesssim
(1+\gamma^2)B_{\rm H}^2\epsilon^2
\bigl(\log(e/\epsilon)\bigr)^\zeta.
\label{eq:G1-epsilon-rate}
\end{equation}
Bellman completeness expresses the auxiliary risk gap as a nonnegative
approximation error:
\[
\mathcal L_{k,\mathcal U_2}^{\mathrm{crit}}(\phi)
-\mathcal L_{k,\mathcal G_2}^{\mathrm{crit}}(\phi)
=
\inf_{O\in\mathcal G_2}
\|O-\mathcal T^{\pi_k}\phi\|_{L^2(\mu_k)}^2.
\]
For $v\in\mathcal U_2$ and $O\in\mathcal G_2$,
\[
\|O-v\|_{L^2(\mu_k)}^2
\le
(B_{\rm net}+B_{\rm H})\|O-v\|_{L^1(\mu_k)}.
\]
Taking the infimum over $O$ and the supremum over $\phi$, enlarging the
target set to $\mathcal U_2$, and using $B_{\rm H}\le B_{\rm net}$ give
\begin{equation}
\mathcal E_{\mathcal G_2}
\le
(4B_{\rm net}+2R_{\max})
\sup_{v\in\mathcal U_2}
\inf_{O\in\mathcal G_2}
\|O-v\|_{L^1(\mu_k)}.
\label{eq:G2-L1-reduction}
\end{equation}
For each $v\in\mathcal U_2$, the same rescaling and clipping
construction gives $O_{v,R,\epsilon}\in\mathcal G_2$ such that
\begin{align*}
\sup_{z\in\mathcal Z_R}|v(z)-O_{v,R,\epsilon}(z)|
&\le C_{\zeta,d}B_{\rm H}R^\zeta\epsilon,\\
|v-O_{v,R,\epsilon}|&\le2B_{\rm H}\quad\text{on }\mathbb R^d.
\end{align*}
Consequently,
\[
\|v-O_{v,R,\epsilon}\|_{L^1(\mu_k)}
\le
C_{\zeta,d}B_{\rm H}R^\zeta\epsilon
+2B_{\rm H}C_\mu e^{-c_\mu R^2}.
\]
Taking $R=R_\epsilon$ yields
\begin{equation*}
\sup_{v\in\mathcal U_2}
\inf_{O\in\mathcal G_2}
\|O-v\|_{L^1(\mu_k)}
\lesssim
B_{\rm H}\epsilon
\bigl(\log(e/\epsilon)\bigr)^{\zeta/2}.
\end{equation*}
Together with \eqref{eq:G2-L1-reduction}, this gives
\begin{equation}
\mathcal E_{\mathcal G_2}
\lesssim
(4B_{\rm net}+2R_{\max})B_{\rm H}\epsilon
\bigl(\log(e/\epsilon)\bigr)^{\zeta/2}.
\label{eq:G2-epsilon-rate}
\end{equation}
Finally, choosing
$
\epsilon\asymp
\left(\frac{\log(e\mathcal S)}{\mathcal S}\right)^{\zeta/d}
$
with a sufficiently large multiplicative constant meets the size
constraint above. Substituting into
\eqref{eq:G1-epsilon-rate} and \eqref{eq:G2-epsilon-rate}
proves \eqref{eq:G1-unbounded-approximation}
and \eqref{eq:G2-unbounded-approximation}.
\end{proof}
\subsection{Proof of Theorem~\ref{thm:action-value-error}}

We first prove Lemma~\ref{lem:policy-evaluation-residual}, which relates
the Bellman residual to the action-value error, and then prove
Theorem~\ref{thm:action-value-error}.
\begin{proof}[Proof of Lemma~\ref{lem:policy-evaluation-residual}]
Let $\delta=\widehat Q-\mathcal T^\pi\widehat Q$.
The Bellman identity \eqref{eq:policy-bellman-fixed-points} gives
\[
\widehat Q-Q^\pi=\delta+\gamma P^\pi(\widehat Q-Q^\pi).
\]
Since both action-value functions are bounded and $\gamma\in(0,1)$,
iteration yields the uniformly convergent series
\[
\widehat Q-Q^\pi
=
\sum_{t=0}^\infty\gamma^t(P^\pi)^t\delta.
\]
For each $t\ge0$, Jensen's inequality  %
and concentrability in Assumption~\ref{assump2}
imply
\[
\|(P^\pi)^t\delta\|_{L^2(\nu)}^2
\le
\|\delta\|_{L^2(\nu(P^\pi)^t)}^2
\le
C_{\rm eval}\|\delta\|_{L^2(\mu)}^2.
\]
The triangle inequality and the geometric-series identity therefore give
\begin{align*}
\|\widehat Q-Q^\pi\|_{L^2(\nu)}
\le
\sum_{t=0}^\infty\gamma^t\|(P^\pi)^t\delta\|_{L^2(\nu)}
\le
\frac{\sqrt{C_{\rm eval}}}{1-\gamma}
\|\widehat Q-\mathcal T^\pi\widehat Q\|_{L^2(\mu)}.
\end{align*}
Squaring proves the lemma.
\end{proof}

\begin{proof}[Proof of Theorem~\ref{thm:action-value-error}]
By Lemmas~\ref{lem:excess-risk}, \ref{lem:statistical-error}, and
\ref{lem:app error}, we have
\begin{equation*}
\begin{aligned}
\mathbb E\Bigl[
\mathcal L_{k,\mathcal U_2}^{\mathrm{crit}}(\widehat Q_k)
-\mathcal L_{k,\mathcal U_2}^{\mathrm{crit}}(Q^{\pi_k})
\Bigr]
&\lesssim
\widetilde M
\left[
\sqrt{\frac{\mathcal S\mathcal D\mathfrak L_n}{\xi_n}}
+\xi_n\beta_{a_n}
\right]\\
&\qquad+
(8B_{\rm net}+4R_{\max})B_{\rm H}
\left(\frac{\log(e\mathcal S)}{\mathcal S}\right)^{\zeta/d}
\bigl(\log(e\mathcal S)\bigr)^{\zeta/2}\\
&\qquad+
(1+\gamma^2)B_{\rm H}^2
\left(\frac{\log(e\mathcal S)}{\mathcal S}\right)^{2\zeta/d}
\bigl(\log(e\mathcal S)\bigr)^\zeta.
\end{aligned}
\end{equation*}
Choose
$\mathcal S\asymp\xi_n^{d/(d+2\zeta)}\log \xi_n $ and
$\mathcal D\asymp\log \xi_n$.
The statistical square-root term and the auxiliary approximation term
then have polynomial order $\xi_n^{-\zeta/(d+2\zeta)}$, whereas the
$\mathcal G_1$ approximation term has the smaller order
$\xi_n^{-2\zeta/(d+2\zeta)}$. The prescribed polynomial growth of
$\mathcal W$ and $B_{\rm par}$ affects only logarithmic factors.
Consequently,
\[
\mathbb E\Bigl[
\mathcal L_{k,\mathcal U_2}^{\mathrm{crit}}(\widehat Q_k)
-\mathcal L_{k,\mathcal U_2}^{\mathrm{crit}}(Q^{\pi_k})
\Bigr]
\le
\widetilde{\mathcal O}\left(
B_{\rm H}\xi_n^{-\frac{\zeta}{d+2\zeta}}
+\xi_n\beta_{a_n}\right).
\]
The population excess risk equals
$\|\widehat Q_k-\mathcal T^{\pi_k}\widehat Q_k\|_{L^2(\mu_k)}^2$.
Lemma~\ref{lem:policy-evaluation-residual} therefore yields, for every
$\nu\in\mathfrak M$,
\[
\mathbb E\|\widehat Q_k-Q^{\pi_k}\|_{L^2(\nu)}^2
\lesssim
\frac{1}{(1-\gamma)^2}
\widetilde{\mathcal O}\left(
B_{\rm H}\xi_n^{-\frac{\zeta}{d+2\zeta}}
+\xi_n\beta_{a_n}\right).
\]
Under $\beta_m\le\bar\beta e^{-bm^\eta}$ and
$a_n=\lceil(2\log n/b)^{1/\eta}\rceil$, %
\[
\xi_n\asymp\frac{n}{(\log n)^{1/\eta}},
\qquad
\xi_n\beta_{a_n}\lesssim n^{-1}.
\]
The mixing remainder is of smaller order than the leading critic
estimation term. Thus
\[
\mathbb E\|\widehat Q_k-Q^{\pi_k}\|_{L^2(\nu)}^2
\leq
\frac{1}{(1-\gamma)^2}
\widetilde{\mathcal O}\left(n^{-\frac{\zeta}{d+2\zeta}}\right).
\]
\end{proof}

\section{Proofs for the Amortized Conditional SF Actor}
\label{sec:proofs-amortized-actor}

\subsection{Finite Fisher information and the target drift}
\begin{lemma}[Gaussian final step and weak action derivatives]
\label{lem:conditional-drift-regularity-revised}
Suppose the reference policy is generated by a finite Euler scheme
with bounded drift and a nondegenerate Gaussian final increment.
Let the critic be a bounded ReLU network with finite parameters. Then the reference and tilted target have
finite relative entropy and finite relative Fisher information with
respect to $\gamma_{d_{\mathcal A}}$ at each state. These quantities
are integrable over states when the drift bound and critic action
Lipschitz constant are uniform in the state, as for the specified
finite networks. The bounds can depend on network size and policy
iteration.
\end{lemma}
\begin{proof}
Write $p_s$ for the Lebesgue density of the reference endpoint. If its
last increment has covariance $\Delta t I_{d_{\mathcal A}}$, then
$A=V+\sqrt{\Delta t}\,\varepsilon$, where $V$ is the random position
plus drift before that increment and $\varepsilon$ is independent.
Gaussian convolution and conditional Jensen's inequality give
\[
\nabla_a\log p_s(A)
=\mathbb E[-(A-V)/\Delta t\mid A,s],
\qquad
\int\|\nabla_a\log p_s(a)\|_2^2p_s(a)\,\mathrm da
\le\frac{d_{\mathcal A}}{\Delta t}.
\]
Since $\nabla_a\log f_{k,s}(a)=\nabla_a\log p_s(a)+a$,
\begin{equation}
\int\|\nabla_a\log f_{k,s}(a)\|_2^2\pi_k(\mathrm da\mid s)
\le\frac{2d_{\mathcal A}}{\Delta t}
+2\mathbb E_{\pi_k(\cdot\mid s)}\|A\|_2^2.
\label{eq:reference-fisher-finite}
\end{equation}
Bounded Euler drift implies
$\mathbb E_{\pi_k(\cdot\mid s)}\|A\|_2^2
\le2d_{\mathcal A}(1+B_{\rm net}^2)$.
The bound $p_s\le(2\pi\Delta t)^{-d_{\mathcal A}/2}$ and the second
moment also imply finite Gaussian relative entropy.

A finite ReLU network is globally Lipschitz and weakly differentiable.
Since the normalizer depends only on $s$,
\[
\nabla_a\log\bar f_{k+1,s}
=\nabla_a\log f_{k,s}+\lambda^{-1}\nabla_a\widehat Q_k
\quad\text{almost everywhere}.
\]
Since $\|\widehat Q_k\|_\infty\le B_{\rm net}$, the normalized endpoint
tilt is between $e^{-2B_{\rm net}/\lambda}$ and
$e^{2B_{\rm net}/\lambda}$. Hence
\begin{align}
&\int\|\nabla_a\log\bar f_{k+1,s}\|_2^2\,\mathrm d\bar\pi_{k+1}
\notag\\
&\quad\le 2e^{2B_{\rm net}/\lambda}
\int\|\nabla_a\log f_{k,s}\|_2^2\,\mathrm d\pi_k
+\frac2{\lambda^2}
\int\|\nabla_a\widehat Q_k\|_2^2\,\mathrm d\bar\pi_{k+1}<\infty.
\label{eq:target-fisher-finite}
\end{align}
The tilted density remains bounded above and has finite second moment,
so its Gaussian relative entropy is finite as well. The weak gradient of a finite ReLU network is bounded by its action
Lipschitz constant. For fixed networks, these estimates are uniform in
$s$. The bound in \eqref{eq:reference-fisher-finite} scales with the
inverse variance of the final step, and the critic gradient bound may
grow with network size. The next lemma gives a bound uniform over
grids for fixed networks.
\end{proof}

\subsection{A Fisher-information bound independent of the Euler grid}
\begin{lemma}[Shift bound for the frozen Euler policy]
\label{lem:actor-fisher-grid-uniform}
For the frozen policy generated by $b_{\phi_k}$ over diffusion time
$[0,1]$, let $\operatorname{Lip}_y(b_{\phi_k})$ denote its Lipschitz
constant in the diffusion position, uniformly in state and grid time.
Then
\begin{align}
\int\|\nabla_a\log f_{k,s}(a)\|_2^2\pi_k(\mathrm da\mid s)
\le{}&2d_{\mathcal A}
\{1+\operatorname{Lip}_y(b_{\phi_k})\}^2
+2\mathbb E_{\pi_k(\cdot\mid s)}\|A\|_2^2.
\label{eq:reference-fisher-grid-uniform}
\end{align}
Consequently,
\begin{align}
&\int\|\nabla_a\log\bar f_{k+1,s}(a)\|_2^2
\bar\pi_{k+1}(\mathrm da\mid s)\notag\\
&\quad\le 2e^{2\|\widehat Q_k\|_\infty/\lambda}
\left[2d_{\mathcal A}\{1+\operatorname{Lip}_y(b_{\phi_k})\}^2
+2\mathbb E_{\pi_k(\cdot\mid s)}\|A\|_2^2\right]
+\frac{2\operatorname{Lip}_a(\widehat Q_k)^2}{\lambda^2}.
\label{eq:target-fisher-grid-uniform}
\end{align}
The bounds use the drift Lipschitz constant in place of the inverse final-step variance.
\end{lemma}
\begin{proof}
Fix $s$ and a translation
$v\in\mathbb R^{d_{\mathcal A}}$. If $X_j$ is the reference Euler
chain on a grid of mesh $h$, define $Y_j=X_j+jhv$. This chain starts
at zero, ends at $X_{1/h}+v$, and has conditional mean increment
\[
h\{b_{\phi_k}(s,Y_j-jhv,jh)+v\}.
\]
The reference chain at $Y_j$ has mean increment
$h b_{\phi_k}(s,Y_j,jh)$, with the same Gaussian covariance $hI$.
The chain rule for relative entropy and terminal data processing give
\begin{align*}
\mathrm{KL}(\pi_k(\cdot-v\mid s)\|\pi_k(\cdot\mid s))
&\le\frac h2\sum_j\mathbb E
\|v+b_{\phi_k}(s,X_j,jh)-b_{\phi_k}(s,X_j+jhv,jh)\|_2^2\\
&\le\tfrac12\{1+\operatorname{Lip}_y(b_{\phi_k})\}^2\|v\|_2^2.
\end{align*}
Lipschitz continuity suffices for this argument. To obtain the
directional Fisher bound under the same condition, apply the
variational formula for KL to the test function $\varepsilon g$, with $g$ smooth and
compactly supported and $v=\varepsilon e$ for a unit vector $e$.
Divide by $\varepsilon^2$ and let $\varepsilon\to0$ to obtain
\[
\mathbb E_{\pi_k}\partial_e g
-\tfrac12\operatorname{Var}_{\pi_k}(g)
\le\tfrac12\{1+\operatorname{Lip}_y(b_{\phi_k})\}^2.
\]
Rescaling $g$ gives
$|\mathbb E_{\pi_k}\partial_e g|^2
\le\{1+\operatorname{Lip}_y(b_{\phi_k})\}^2
\operatorname{Var}_{\pi_k}(g)$.
Integration by parts and $L^2$ duality bound each directional score
second moment by this constant. Summing over an orthonormal basis
bounds the Lebesgue Fisher information by
$d_{\mathcal A}\{1+\operatorname{Lip}_y(b_{\phi_k})\}^2$.
The density is smooth by the final Gaussian increment; its finite
score second moment was also established in
Lemma~\ref{lem:conditional-drift-regularity-revised}.
Using $\nabla\log f_{k,s}=\nabla\log p_s+a$ proves
\eqref{eq:reference-fisher-grid-uniform}.

Boundedness of the critic and its normalizer gives
$\mathrm d\bar\pi_{k+1}/\mathrm d\pi_k
\le e^{2\|\widehat Q_k\|_\infty/\lambda}$.
Apply the weak-gradient identity in
\eqref{eq:target-fisher-finite} and bound
$\|\nabla_a\widehat Q_k\|_2$ by its action Lipschitz constant.
This proves \eqref{eq:target-fisher-grid-uniform}, uniformly over
grids for fixed network norms. The bound may grow as the network
norms change across policy iterations.
\end{proof}

\subsection{Proof of Lemma~\ref{lem:kl-time-discretization}}
\begin{proof}
Fix $s$ and let $\mathbb P^s$ be Wiener measure on $[0,1]$, with
coordinate process $Z_t^s$. The target path law $\overline{\mathbb P}^s$
has density $\bar f_{k+1,s}(Z_1^s)$ with respect to $\mathbb P^s$.
Finite entropy gives its F\"ollmer drift and the entropy--energy
identity \citep{lehec2011representationformulaentropyfunctional}
Gaussian integration by parts with weak derivatives gives
\begin{equation}
\bar b_{k+1,s}(Z_t^s,t)
=\mathbb E_{\overline{\mathbb P}^s}
[\nabla_a\log\bar f_{k+1,s}(Z_1^s)\mid\mathcal F_t],\qquad t<1.
\label{eq:target-drift-martingale}
\end{equation}
Differentiate $Q_{1-t}\bar f$ weakly and divide by
$Q_{1-t}\bar f$ to obtain the conditional expectation. Finite Fisher
information gives a square-integrable terminal variable and justifies
the weak approximation and conditional-expectation limit. The process
in \eqref{eq:target-drift-martingale} is therefore an $L^2$ martingale,
with terminal value $\nabla_a\log\bar f_{k+1,s}(Z_1^s)$.

For $u\ge t$, orthogonality of martingale increments implies
\begin{align*}
&\mathbb E_{\overline{\mathbb P}^s}
\|\bar b_{k+1,s}(Z_u^s,u)-\bar b_{k+1,s}(Z_t^s,t)\|_2^2=
\mathbb E_{\overline{\mathbb P}^s}\|\bar b_{k+1,s}(Z_u^s,u)\|_2^2
-\mathbb E_{\overline{\mathbb P}^s}\|\bar b_{k+1,s}(Z_t^s,t)\|_2^2.
\end{align*}
The second moments are nondecreasing, so integration over the grid
intervals followed by telescoping gives
\begin{align}
&\sum_j\int_{t_j}^{t_{j+1}}
\mathbb E_{\overline{\mathbb P}^s}
\|\bar b_{k+1,s}(Z_t^s,t)-\bar b_{k+1,s}(Z_{t_j}^s,t_j)\|_2^2
\,\mathrm dt \le\Delta t\int
\|\nabla_a\log\bar f_{k+1,s}(a)\|_2^2
\bar\pi_{k+1}(\mathrm da\mid s).
\label{eq:martingale-time-error}
\end{align}

Let $\widetilde{\mathbb P}^s$ be the path law of
\eqref{eq:continuous version  of sampling process}. Since the learned
drift is bounded, its Girsanov exponential is a true martingale and
$\widetilde{\mathbb P}^s$ is equivalent to Wiener measure. Under the
target law, the entropy--energy identity and subtraction of this log
Girsanov density give
\begin{align}
\mathrm{KL}(\overline{\mathbb P}^s\|\widetilde{\mathbb P}^s)
=\frac12\mathbb E_{\overline{\mathbb P}^s}\int_0^1
\|\bar b_{k+1,s}(Z_t^s,t)
-b_{\phi_{k+1}}(s,Z_{\underline t}^s,\underline t)\|_2^2
\,\mathrm dt.
\label{eq:reverse-path-kl}
\end{align}
To justify the subtraction, write
$\mathrm dZ_t^s=\bar b_{k+1,s}(Z_t^s,t)\,\mathrm dt+\mathrm dB_t$
under the target law. The stochastic integral of the bounded learned
drift against $B$ has mean zero, while the drift cross term is
integrable by finite target energy. No Novikov condition on an
unbounded target drift is needed.

For each $t\in[t_j,t_{j+1})$, the quantity
$\bar b_{k+1,s}(Z_{t_j}^s,t_j)-b_{\phi_{k+1}}(s,Z_{t_j}^s,t_j)$
is $\mathcal F_{t_j}$-measurable. Its inner product with the target
martingale increment has zero expectation. Thus the right-hand side
of \eqref{eq:reverse-path-kl} is exactly the sum of one half of the
time-increment energy and
\[
\frac{\Delta t}{2}\sum_j\mathbb E_{\overline{\mathbb P}^s}
\|\bar b_{k+1,s}(Z_{t_j}^s,t_j)
-b_{\phi_{k+1}}(s,Z_{t_j}^s,t_j)\|_2^2.
\]
Apply data processing under the terminal-coordinate map to obtain
\eqref{eq:actor-drift-time-decomposition}. For the time term,
\eqref{eq:martingale-time-error} and
Lemma~\ref{lem:actor-fisher-grid-uniform}, averaged over
$S\sim d^{\pi^\star}$, give $C/T$ and prove
\eqref{eq:kl-time-discretization}. Bounded frozen drift gives a uniform
endpoint second moment, and the bounded finite ReLU critic has finite
action Lipschitz constant, so this $C$ is independent of the new grid.
The argument includes the terminal interval and does not assume bounded
spatial derivatives of the target drift.
\end{proof}

\subsection{Exact target bridge and regression error}
\begin{lemma}[Weighted transfer to the exact target path]
\label{lem:exact-path-coverage}
For any measurable vector field $b$ with finite excess risk,
\begin{align}
&\Delta t\sum_j\mathbb E_{S\sim d^{\pi^\star}}
\mathbb E_{\overline{\mathbb P}^S}
\|b(S,Z_{t_j}^S,t_j)-\bar b_{k+1,S}(Z_{t_j}^S,t_j)\|_2^2
\notag\\
&\quad\le d_{\mathcal A}e^{\|\widehat Q_k\|_\infty/\lambda}
\left\|\frac{\mathrm d d^{\pi^\star}}
{\mathrm d\mu_{k,\mathcal X}}\right\|_\infty
\{\mathcal L_{\rm actor}^k(b)
-\mathcal L_{\rm actor}^k(\bar b_{k+1})\}.
\label{eq:exact-target-risk-transfer}
\end{align}
\end{lemma}
\begin{proof}
The reference and target SF paths mix the same Brownian bridges
over endpoint laws $\pi_k$ and $\bar\pi_{k+1}$. Their endpoint density
ratio is $\omega_k(S,A)/C(S)$, so averaging over uniform grid time
expresses the left-hand side as
\[
P\left[\frac{\mathrm d d^{\pi^\star}}
{\mathrm d\mu_{k,\mathcal X}}(S)
\frac{\omega_k(S,A)}{C(S)}\|b(X)-\bar b_{k+1}(X)\|_2^2\right].
\]
Use $C(S)=\mathbb E_{A\sim\pi_k}\omega_k(S,A)
\ge e^{-\|\widehat Q_k\|_\infty/\lambda}$ and the weighted excess-risk identity.
At $t_0=0$, the bridge is a point mass at zero and the same argument
applies. The comparison uses the two exact bridge laws, independently
of the learned Euler marginal.
\end{proof}

\subsection{Proof of Lemma~\ref{lem:bridge-label-moment}}
\begin{proof}
The bridge construction gives
$U_k=A-\sqrt{t/(1-t)}\,\varepsilon$.
Conditioning on $(S,A,t)$ removes the cross term. Since
\[
\frac1T\sum_{j=0}^{T-1}\frac{t_j}{1-t_j}
=\frac1T\sum_{j=0}^{T-1}\frac{j}{T-j}=H_T-1,
\]
averaging proves the identity. Bounded endpoint weights give a finite
weighted label second moment. Conditional Jensen's inequality then
makes $\bar b_{k+1}(X)$ square integrable under the weighted design,
which justifies the squared-loss decomposition and centered excess losses.
\end{proof}

\subsection{Proof of Lemma~\ref{lem:learned-drift-error}}
\begin{proof}
For every $b\in\mathcal G_3$, subtract the empirical losses and cancel
their common label-only term to obtain
\begin{align*}
\mathcal L_{\rm actor}^k(b_{\phi_{k+1}})-\mathcal L_{\rm actor}^k(b)
={}&\widehat{\mathcal L}_{\rm actor}^k(b_{\phi_{k+1}})
-\widehat{\mathcal L}_{\rm actor}^k(b)+(P-P_n)\ell_{b_{\phi_{k+1}}}-(P-P_n)\ell_b.
\end{align*}
Empirical risk minimization makes the empirical difference
nonpositive. The sum of the two deviations is at most twice their
uniform absolute supremum. Add
$\mathcal L_{\rm actor}^k(b)-\mathcal L_{\rm actor}^k(\bar b_{k+1})$
and take the infimum over $b$. Use
\eqref{eq:bridge-regression-drift-error} to express both population
excess losses as squared design norms. 
Taking expectations gives \eqref{eqe1}. The deterministic inequality
also holds for dependent observations; the population infimum need
not be attained.
\end{proof}

\subsection{A finite-sample statistical bound}
\begin{lemma}[Independent regression observations]
\label{lem:actor-statistical-error}
If the $n$ observations are conditionally independent draws from the
frozen design $P$, then
\begin{equation}
\mathbb E\sup_{b\in\mathcal G_3}|(P_n-P)\ell_b|
\le2\mathbb E\sup_{b\in\mathcal G_3}
\left|\frac1n\sum_{i=1}^n\sigma_i\ell_b(X_i,A_i)\right|,
\label{eq:actor-statistical-symmetrization}
\end{equation}
where the independent signs $\sigma_i\in\{-1,1\}$ have equal
probabilities and are independent of the observations.
For $\eta>0$, let $\mathcal N_n(\eta)$ be the covering number of
$\mathcal G_3$ for the sample distance
$\max_{i\le n}\|b(X_i)-b'(X_i)\|_2/\sqrt{d_{\mathcal A}}$,
with covering centers in $\mathcal G_3$. The right-hand side of
\eqref{eq:actor-statistical-symmetrization} is at most
\begin{align}
2e^{\|\widehat Q_k\|_\infty/\lambda}\mathbb E\Biggl[
&\sqrt{\frac{2\log(2\mathcal N_n(\eta))}{n}}
\left\{\frac1n\sum_{i=1}^n
\left(B_{\rm net}^2+
\frac{2B_{\rm net}\|U_{k,i}\|_2}{\sqrt{d_{\mathcal A}}}\right)^2
\right\}^{1/2}\notag\\
&+2\eta\left(B_{\rm net}
+\frac1n\sum_{i=1}^n\frac{\|U_{k,i}\|_2}{\sqrt{d_{\mathcal A}}}\right)
\Biggr].
\label{eq:actor-statistical-cover}
\end{align}
\end{lemma}
\begin{proof}
The integrable envelope
\[
|\ell_b(X,A)|\le e^{\|\widehat Q_k\|_\infty/\lambda}
\left(B_{\rm net}^2+
\frac{2B_{\rm net}\|U_k\|_2}{\sqrt{d_{\mathcal A}}}\right)
\]
allows the usual independent ghost-sample argument: Jensen's inequality
replaces $P\ell_b$ by the ghost empirical mean, sign symmetry exchanges
each paired observation, and the triangle inequality separates the two
samples. This proves \eqref{eq:actor-statistical-symmetrization}.

Condition on the observed sample and choose an $\eta$-net in the stated
distance. For a network and its net representative,
\[
|\ell_b(X_i,A_i)-\ell_{b'}(X_i,A_i)|
\le2e^{\|\widehat Q_k\|_\infty/\lambda}\eta
\left(B_{\rm net}+\frac{\|U_{k,i}\|_2}{\sqrt{d_{\mathcal A}}}\right).
\]
For each of the $\mathcal N_n(\eta)$ net vectors of loss values, its
signed sample mean is sub-Gaussian, by
$\cosh(x)\le e^{x^2/2}$, with variance proxy bounded by $1/n$ times
the sample average of the squared envelope. The exponential maximum
bound over the net and its negatives gives the first term in
\eqref{eq:actor-statistical-cover}; the net approximation gives the
second term. Finally average over the observed sample.
\end{proof}

A fixed finite architecture with bounded parameters has finite
sample covering numbers. Lemma~\ref{lem:bridge-label-moment}
bounds the squared envelope, with logarithmic grid dependence through
$H_T$. Smoothness enters separately in the approximation bound.
When observations are dependent or replay states are reused, the
general regression bound uses the empirical discrepancy; repeated
buffer draws do not count as independent population states.

\subsection{Proof of Theorem~\ref{thm:conditional-sfs-sampling-error}}
\label{app:proof of actor}
\begin{proof}
The path-space argument above gives the distribution-error decomposition
\eqref{eq:actor-drift-time-decomposition}. Its time term is at most
$C/T$ by Lemma~\ref{lem:kl-time-discretization}.
Lemma~\ref{lem:exact-path-coverage} and
\eqref{eq:bridge-regression-drift-error} bound the grid drift term by
$C\|b_{\phi_{k+1}}-\bar b_{k+1}\|_{\chi_k}^2$.
Thus, for each fitted bounded actor,
\begin{equation}
\mathbb E_{S\sim d^{\pi^\star}}
\mathrm{KL}(\bar\pi_{k+1}(\cdot\mid S)\|\pi_{k+1}(\cdot\mid S))
\le C\|b_{\phi_{k+1}}-\bar b_{k+1}\|_{\chi_k}^2+\frac{C}{T}.
\label{eq:actor-distribution-from-drift}
\end{equation}
For an empirical minimizer, apply \eqref{eqe1} to the expected squared
drift error. Under the stated independent design and parameter bounds,
the covering-number argument leading to
\eqref{eq:bridge-statistical-explicit} controls its statistical term.
Sub-Gaussian tails place the sample in a box of radius
$C\sqrt{\log n}$ outside an event whose probability decays at an
arbitrarily large polynomial rate. Parameter covers have log cardinality
$\widetilde{\mathcal O}(\mathcal S\mathcal D)$, and the centered-loss
envelope has second moment controlled by
$[B_{\rm net}^2+B_{\rm net}\sqrt{1+\log T}]^2$.
Lemma~\ref{lem:actor-statistical-error} gives the statistical term.
The polynomial grid and parameter bounds make the contribution
outside the box negligible; the details follow below. The remaining
infimum is the approximation error.
Combining the three estimates proves the claim.
Under Assumption~\ref{assump:actor-target-regularity},
Lemma~\ref{lem:actor-approximation-error} supplies
\eqref{eq:sketch-actor-approximation-bound}. With fixed $B_{\rm net}$,
balancing $\sqrt{\mathcal S/n}$ against
$\mathcal S^{-2\zeta/p}$, up to logarithmic factors, gives the
stated class size and \eqref{eq:actor-holder-sample-rate}.
\end{proof}

\subsection{Network complexity and the statistical term}
Bounded reference drift gives uniform conditional Gaussian tails for
its endpoints: the endpoint is a standard Brownian terminal value plus
a bounded integral. The bridge construction and the state-tail
condition therefore imply
\[
P(\|(S,Y)\|_\infty>R)\le Ce^{-cR^2},
\]
uniformly over the bridge grid. In addition to
Lemma~\ref{lem:bridge-label-moment}, the identity
$U_k=A-\sqrt{t/(1-t)}\varepsilon$ gives
\begin{equation}
P\|U_k\|_2^4\le C(1+T),\qquad
\frac1T\sum_{j=0}^{T-1}\left(\frac{j}{T-j}\right)^2\le CT.
\label{eq:bridge-fourth-moment}
\end{equation}
On a sample whose inputs lie in a box of radius $R$, a parameter net
for the bounded-parameter sparse ReLU class gives
\[
\log\mathcal N_n(\eta)
\le C\mathcal S\left[
\log(1/\eta)+\mathcal D\log\{C(1+R)\mathcal W B_{\rm par}\}
\right].
\]
Count supports of at most $\mathcal S$ parameters and discretize
the active parameters. Layerwise propagation bounds the
effect of a parameter perturbation by
$C(1+R)\mathcal D(C\mathcal W B_{\rm par})^{\mathcal D+1}$
times its magnitude. Taking logarithms gives the displayed entropy
bound, with the support count absorbed into the same expression.
Use $R$ of order $\sqrt{\log n}$, a polynomially small net radius,
and Lemma~\ref{lem:actor-statistical-error}. The complement of this
sample event has arbitrarily small polynomial probability by choosing
the constant in $R$ sufficiently large. The centered-loss second
moment controls its contribution by Cauchy--Schwarz. Consequently,
\begin{equation}
\mathbb E\sup_{b\in\mathcal G_3}|(P_n-P)\ell_b|
\le C\{B_{\rm net}^2+B_{\rm net}\sqrt{1+\log T}\}
\widetilde{\mathcal O}\left(\sqrt{\frac{\mathcal S\mathcal D}{n}}\right).
\label{eq:bridge-statistical-explicit}
\end{equation}
The logarithmic factors include the polynomial parameter and width
bounds. The endpoint-weight bound uses the frozen
$\|\widehat Q_k\|_\infty$.

\subsection{Proof of Lemma~\ref{lem:actor-approximation-error}}
\label{app:actor-approximation-size-proof}
\begin{proof}
We bound the target drift's regularity, approximate it on an expanding
box, and control the error outside the box using the design tails.

\paragraph{Regularity of the target drift.}
Write $u_s(y,t)=Q_{1-t}\bar f_{k+1,s}(y)$. Under
Assumption~\ref{assump:actor-target-regularity}, positivity preservation
gives $u_s\ge\xi_f$. For derivatives within the assumed smoothness
range,
\[
\partial_t^j\partial_s^\alpha\partial_y^\beta u_s
=(-1/2)^jQ_{1-t}
\Delta_a^j\partial_s^\alpha\partial_a^\beta\bar f_{k+1,s}.
\]
Since $\bar b_{k+1,s}=\nabla_yu_s/u_s$, each action derivative
requires one derivative of the density ratio, and each time derivative
requires two.
The $\mathcal H^{2\zeta+2}$ bound, the heat-semigroup increment
estimate, and the quotient rule therefore give a coordinatewise
$\mathcal H^\zeta$ bound on $(s,y,t)\in\mathbb R^{p-1}\times[0,1]$,
including $t=1$. Let $B_{\rm H}$ bound these H\"older norms
uniformly and fix $B_{\rm net}\ge B_{\rm H}$.
The bound depends on $B_f,\xi_f,\zeta$ and the dimensions, but not
on $k$ or $T$.

\paragraph{Approximation on an expanding box.}
Set
\[
\delta_{\mathcal S}
:=\left(\frac{\log(e\mathcal S)}{\mathcal S}\right)^{\zeta/p},
\qquad R_{\mathcal S}:=c_R\sqrt{\log(e\mathcal S)},
\]
and let
\[
\mathcal Z_{R_{\mathcal S}}
:=\{(s,y,t):\|(s,y)\|_\infty\le R_{\mathcal S},\ t\in[0,1]\}.
\]
Rescaling this box to $[0,1]^p$ multiplies the H\"older bound by
at most $C R_{\mathcal S}^{\zeta}$. Apply
Lemma~\ref{lem:th5-of-feng} coordinatewise, assigning a fixed fraction
of the total parameter budget to each coordinate. A sufficiently
large constant multiple of $\delta_{\mathcal S}$ gives a bounded-output
ReLU network $b_{\mathcal S}\in\mathcal G_3$ with size at most
$\mathcal S$, depth $\mathcal D\lesssim\log(e\mathcal S)$, and
polynomially bounded widths and parameters, such that
\begin{equation}
\sup_{x\in\mathcal Z_{R_{\mathcal S}}}
\|b_{\mathcal S}(x)-\bar b_{k+1}(x)\|_2^2
\lesssim d_{\mathcal A}B_{\rm H}^2
R_{\mathcal S}^{2\zeta}\delta_{\mathcal S}^2.
\label{eq:actor-interior-approximation}
\end{equation}
The bounded-output realization preserves the size and depth orders.

\paragraph{Weighted error and the tail.}
Bounded endpoint weights give
\begin{align*}
\|b_{\mathcal S}-\bar b_{k+1}\|_{\chi_k}^2
&=\mathbb E\left[\frac{\omega_k(S,A)}{d_{\mathcal A}}
\|b_{\mathcal S}(X)-\bar b_{k+1}(X)\|_2^2\right]\\
&\le\frac{e^{\|\widehat Q_k\|_\infty/\lambda}}{d_{\mathcal A}}
\mathbb E\|b_{\mathcal S}(X)-\bar b_{k+1}(X)\|_2^2.
\end{align*}
Inside the box use \eqref{eq:actor-interior-approximation}.
Outside it, the coordinatewise bounds and the sub-Gaussian design
tail give
\begin{equation}
\mathbb E\left[\frac{\omega_k(S,A)}{d_{\mathcal A}}
\|b_{\mathcal S}(X)-\bar b_{k+1}(X)\|_2^2
\mathbf 1_{\{\|(S,Y)\|_\infty>R_{\mathcal S}\}}\right]
\le C(B_{\rm net}^2+B_{\rm H}^2)e^{-cR_{\mathcal S}^2}.
\label{eq:actor-exterior-approximation}
\end{equation}
Here $C$ includes the bounded endpoint-weight factor.
Choose $c_R$ large enough that this tail is at most the interior
approximation order. Since
$R_{\mathcal S}^{2\zeta}\asymp(\log(e\mathcal S))^\zeta$,
combining the two regions proves
\eqref{eq:sketch-actor-approximation-bound}; the factor
$d_{\mathcal A}$ gives a possibly loose upper bound.
This argument is conditional on the frozen design. Taking expectations
under the same uniform bounds proves
\eqref{eq:sketch-actor-approximation-bound}.
\end{proof}

\subsection{Finite-sample rate for a frozen improvement step}
\label{app:bridge-explicit-rate}
\begin{corollary}[Finite-sample rate for weighted bridge regression]
\label{cor:bridge-explicit-rate}
Under the conditions of
Theorem~\ref{thm:conditional-sfs-sampling-error} and
Assumption~\ref{assump:actor-target-regularity}, take
\[
\mathcal S\asymp n^{p/(p+4\zeta)}\log n,\qquad
\mathcal D\asymp\log n,\qquad
T\asymp n^{2\zeta/(p+4\zeta)},
\]
and a fixed output bound $B_{\rm net}\ge B_{\rm H}$, with the
polynomial width and parameter bounds of the approximation lemma.
For the empirical minimizer,
\begin{equation}
\mathbb E
\left[
\mathbb E_{S\sim d^{\pi^\star}}
\mathrm{KL}(\bar\pi_{k+1}(\cdot\mid S)\|\pi_{k+1}(\cdot\mid S))
\right ]
\le\widetilde{\mathcal O}\left(n^{-2\zeta/(p+4\zeta)}\right).
\label{eq:bridge-explicit-rate}
\end{equation}
The expected state-averaged total variation is consequently
$\widetilde{\mathcal O}(n^{-\zeta/(p+4\zeta)})$.
 
\end{corollary}
\begin{proof}
With these choices, the statistical and approximation terms both
have polynomial order $n^{-2\zeta/(p+4\zeta)}$. The bridge-label
factor $\sqrt{1+\log T}$ contributes only a logarithmic factor,
and $T^{-1}$ has the same polynomial order. Apply
\eqref{eq:actor-holder-sample-rate}. Pinsker's inequality and Jensen's
inequality give the total-variation rate.
\end{proof}

If state comparison also holds with $d^{\bar\pi_{k+1}}$ in place of
$d^{\pi^\star}$, the same argument and bounded rewards give
\begin{equation}
\mathbb E|J(\bar\pi_{k+1})-J(\pi_{k+1})|
\le\frac{\sqrt2 R_{\max}}{(1-\gamma)^2}
\left\{\mathbb E\mathbb E_{S\sim d^{\bar\pi_{k+1}}}
\mathrm{KL}(\bar\pi_{k+1}\|\pi_{k+1})\right\}^{1/2}.
\label{eq:bridge-local-return-rate}
\end{equation}
The local return difference is therefore
$\widetilde{\mathcal O}(n^{-\zeta/(p+4\zeta)})$.

\paragraph{Multi-round consequence.}
Under the common bounds in Theorem~\ref{thm:overall-rate}, substitute
the single-step actor rate into \eqref{eq:overall-amortized-rate}.
This gives the explicit rate and choice of $K$ in
\eqref{eq:overall-explicit-rate}--\eqref{eq:overall-final-rate}.

\paragraph{Fresh mixing observations.}
If the frozen regression design is stationary with
$\beta_r\le\bar\beta e^{-br^\eta}$, the statistical proof extends
with $n$ replaced by the block sample size
$\xi_n\asymp n/(\log n)^{1/\eta}$. Indeed, take blocks of length
$a_n=\lceil(c\log n/b)^{1/\eta}\rceil$ with a sufficiently large
fixed $c$ and $\xi_n=\lfloor n/(2a_n)\rfloor$.
Couple odd and even block families separately to independent blocks.
The block-average centered loss has second moment bounded by the
single-observation second moment by Jensen's inequality, and one
parameter net covers all inputs in the blocks. The coupling
probability is at most $\xi_n\beta_{a_n}$, whose contribution is at
most
$C(B_{\rm net}^2+B_{\rm net}\sqrt{1+\log T})
\sqrt{\xi_n\beta_{a_n}}$ by Cauchy--Schwarz.
The discarded incomplete block contributes at most
$C(a_n/n)(B_{\rm net}^2+B_{\rm net}\sqrt{1+\log T})$.
These terms are smaller than the displayed statistical rate.
Choosing the network class and grid with $\xi_n$ in place of $n$
therefore gives the same polynomial rate up to logarithmic factors.
The effective sample size is determined by the fresh observation
process. Repeated draws from a fixed buffer do not increase it.

\paragraph{Offline initialization.}
The same conditional-mean identity applies with unit weights when
endpoints are drawn from the offline data-generating conditional action
law. Pretraining targets the SF drift of the data-generating law, while
the learned Euler policy may differ from that law. Pretraining alone
therefore does not establish $D_0<\infty$ or an initial performance
bound. The overall theorem depends explicitly on $D_0$, and online
updates use the current policy as the endpoint proposal.

\section{Proof of the Overall Performance Bound}
\label{sec:proof-overall-revised}
\begin{lemma}[KL projection identity]
\label{lem:kl-projection}
For fixed $s$ and any $\pi(\cdot\mid s)\ll\pi_k(\cdot\mid s)$ with
finite KL divergence,
\begin{align*}
\langle\pi,\widehat Q_k\rangle-\lambda\mathrm{KL}(\pi\|\pi_k)
={}&\langle\bar\pi_{k+1},\widehat Q_k\rangle
-\lambda\mathrm{KL}(\bar\pi_{k+1}\|\pi_k)
-\lambda\mathrm{KL}(\pi\|\bar\pi_{k+1}).
\end{align*}
\end{lemma}
\begin{proof}
Integrate
$\log(\mathrm d\bar\pi_{k+1}/\mathrm d\pi_k)
=\widehat Q_k/\lambda-\log C$
under $\pi$ and $\bar\pi_{k+1}$ and subtract.
\end{proof}

\subsection{Proof of Theorem~\ref{thm:overall-rate}}
\label{app:proof_overall_rate}
\begin{proof}
Fix a realized sequence and its coefficient $C_K$.
Write $\Delta_k=\widehat Q_k-Q^{\pi_k}$,
$c_k(s)=\langle\pi_k,\widehat Q_k\rangle
=\langle\pi_k,Q^{\pi_k}\rangle+\langle\pi_k,\Delta_k\rangle$, and
$g_k(s)=\langle\bar\pi_{k+1},\widehat Q_k\rangle-c_k(s)$.
The centering quantity $c_k$ is used only in the proof.
Optimality of the exponential tilt implies
$g_k\ge\lambda\mathrm{KL}(\bar\pi_{k+1}\|\pi_k)\ge0$.
The performance-difference lemma and the KL projection identity give
\begin{align*}
(1-\gamma)\{J(\pi^\star)-J(\pi_k)\}
\le{}&\mathbb E_{d^{\pi^\star}}g_k
+\lambda\{D(\pi_k)-D(\bar\pi_{k+1})\}\\
&+\mathbb E_{d^{\pi^\star}\otimes\pi_k}\Delta_k
-\mathbb E_{d^{\pi^\star}\otimes\pi^\star}\Delta_k.
\end{align*}
Since $g_k\ge0$, the occupancy comparison gives
$\mathbb E_{d^{\pi^\star}}g_k
\le C_K\mathbb E_{d^{\bar\pi_{k+1}}}g_k$. Also,
\begin{align*}
\mathbb E_{d^{\bar\pi_{k+1}}}g_k
={}&(1-\gamma)\{J(\bar\pi_{k+1})-J(\pi_k)\}\\
&+\mathbb E_{d^{\bar\pi_{k+1}}\otimes\bar\pi_{k+1}}\Delta_k
-\mathbb E_{d^{\bar\pi_{k+1}}\otimes\pi_k}\Delta_k.
\end{align*}
Cauchy--Schwarz bounds each of the four critic-error integrals by
$\mathcal E_k$. This gives \eqref{eq:decpo} with coefficient
$2(C_K+1)/(1-\gamma)$, fixed throughout the sum over policy iterations.

For the first policy difference, bounded rewards imply
$\|Q^{\pi_{k+1}}\|_\infty\le R_{\max}/(1-\gamma)$.
Use the performance-difference lemma in the direction with occupancy
$d^{\bar\pi_{k+1}}$, and subtract the expectation under $\pi_{k+1}$:
\begin{align*}
|J(\bar\pi_{k+1})-J(\pi_{k+1})|
&\le\frac{2R_{\max}}{(1-\gamma)^2}
\mathbb E_{d^{\bar\pi_{k+1}}}
\mathrm{TV}(\bar\pi_{k+1},\pi_{k+1})\\
&\le\frac{2R_{\max}C_K}{(1-\gamma)^2}
\mathbb E_{d^{\pi^\star}}
\sqrt{\tfrac12\mathrm{KL}(\bar\pi_{k+1}\|\pi_{k+1})}\\
&\le\frac{\sqrt2 R_{\max}C_K}{(1-\gamma)^2}
\sqrt{\mathcal A_k}.
\end{align*}
Pinsker's and Jensen's inequalities give the last two steps;
bounded rewards suffice without action-value smoothness.

For the second difference, let $p=\bar\pi_{k+1}(\cdot\mid s)$ and
$q=\pi_{k+1}(\cdot\mid s)$. On $\{q>p\}$,
$p\log(q/p)\le q-p$, so
\[
\int p\,\bigl(\log(p/q)\bigr)_+
\le\mathrm{KL}(p\|q)+\mathrm{TV}(p,q).
\]
Comparator action coverage and Pinsker's inequality give
\begin{align*}
D(\pi_{k+1})-D(\bar\pi_{k+1})
&=\mathbb E_{d^{\pi^\star}}\int\pi^\star\log(p/q)\\
&\le C_K\mathbb E_{d^{\pi^\star}}
\{\mathrm{KL}(p\|q)+\mathrm{TV}(p,q)\}\\
&\le C_K\left(\mathcal A_k+\sqrt{\mathcal A_k/2}\right).
\end{align*}
The calculation is justified by first truncating the positive log parts.
The bounded log tilt implies $D(\bar\pi_{k+1})<\infty$ whenever
$D(\pi_k)<\infty$. The positive-log bound just proved then implies
$D(\pi_{k+1})<\infty$. Starting from $D_0<\infty$, induction justifies every KL difference.
Finiteness of $\mathcal A_k$ alone gives no uniform upper bound,
so the pathwise inequality includes its linear term.

Insert $\pi_{k+1}$ into both policy differences in \eqref{eq:decpo}.
The preceding estimates yield
\begin{align}
J(\pi^\star)-J(\pi_k)
\le{}&C_K\{J(\pi_{k+1})-J(\pi_k)\}
+\frac{\lambda}{1-\gamma}\{D(\pi_k)-D(\pi_{k+1})\}
\notag\\
&+\frac{2(C_K+1)}{1-\gamma}\mathcal E_k
+\frac{\sqrt2 R_{\max}C_K^2}{(1-\gamma)^2}\sqrt{\mathcal A_k}
\notag\\
&+\frac{\lambda C_K}{1-\gamma}
\left(\mathcal A_k+\sqrt{\mathcal A_k/2}\right).
\label{eq:overall-one-step}
\end{align}
Summing with the same realized $C_K$ for all $k$ telescopes both
differences. Use $J(\pi_K)-J(\pi_0)\le2R_{\max}/(1-\gamma)$ and
$D(\pi_0)-D(\pi_K)\le D_0$, then divide by $K$.
This proves \eqref{eq:overall-amortized-rate} pathwise. Taking
expectations of the entire expression proves its expected version.

For the explicit rate, apply the deterministic bound on $C_K$
and the common constants assumed in the theorem.
The critic estimate controls the finitely many admissible norms in
$\mathcal E_k$. Uniform coverage allows the comparison distribution
to depend on the fitted policies, since the density comparison holds
for every admissible distribution before expectations are taken.
Theorem~\ref{thm:action-value-error}, the inequality
$\max_i x_i^2\le\sum_i x_i^2$, and Cauchy--Schwarz consequently give
\[
\mathbb E\mathcal E_k
\le\frac1{1-\gamma}
\widetilde{\mathcal O}\!\left(n^{-\zeta/(2d+4\zeta)}\right).
\]
With $p=d+1$ and the network sizes of
Theorem~\ref{thm:conditional-sfs-sampling-error}, choose
$T\asymp n^{2\zeta/(d+1+4\zeta)}$. The actor theorem and Jensen's
inequality yield
\[
\mathbb E\mathcal A_k
\le\widetilde{\mathcal O}\!\left(n^{-2\zeta/(d+1+4\zeta)}\right),
\qquad
\mathbb E\sqrt{\mathcal A_k}
\le(\mathbb E\mathcal A_k)^{1/2}
\le\widetilde{\mathcal O}\!\left(n^{-\zeta/(d+1+4\zeta)}\right).
\]
Insert these bounds into the expected version of
\eqref{eq:overall-amortized-rate}. The linear actor-error term has
the faster exponent $2\zeta/(d+1+4\zeta)$ and is absorbed into the
actor term for $n\ge2$. The constants absorb the fixed
$R_{\max},\lambda$ and comparison bounds, and
$(1-\gamma)^{-1}\le(1-\gamma)^{-2}$. This proves
\eqref{eq:overall-explicit-rate}.

Finally, $d+1\le2d$ for $d\ge1$, whence
$\zeta/(d+1+4\zeta)\ge\zeta/(2d+4\zeta)$.
For $K\asymp n^{\zeta/(2d+4\zeta)}$, the iteration term has the
critic's polynomial order and the actor term is no larger, up to
logarithmic factors. This proves \eqref{eq:overall-final-rate}.
Without a uniform comparison bound,
\eqref{eq:overall-amortized-rate} still holds pathwise, with the random
$C_K$ inside expectations.
\end{proof}

\section{Auxiliary Results}
\label{app:auxiliary_lemmas}
We collect the covering, concentration, approximation, and
change-of-measure results used above. The neural network results are
stated on compact domains. Their use on larger boxes requires rescaling;
unbounded designs additionally require tail control.

\begin{lemma}[Lemma 20 of \cite{feng2024deep}]
\label{lemma:covering number upperbound}
Let $\mathcal F$ be a class of ReLU DNNs on $[0,1]^d$ with width
$\mathcal W$, depth $\mathcal D$, size $\mathcal S$, and parameters bounded
in absolute value by $B_{\rm par}$. For sufficiently small $\delta>0$,
\begin{equation*}
    \log \mathcal{N}(\mathcal{F}, \delta, \|\cdot\|_{\infty})
    \lesssim \mathcal S\mathcal D
    \log\!\left(e+B_{\rm par}\mathcal W\mathcal D/\delta\right).
\end{equation*}
\end{lemma}
 
\begin{lemma}[Lemma 5 of \cite{antos2008learning}]
\label{lem:5 of antos}
Let $(Z_t)$ be a strictly stationary $\beta$-mixing sequence with mixing
coefficients $(\beta_m)_{m\ge1}$ and let $n=2a_n\xi_n$, where
$a_n,\xi_n$ are positive integers. Let
$\{\tilde Z_i\}_{i\in H}$ denote its block-independent counterpart,
with $H$ indexing the odd blocks in the independent-block construction.
If $\mathcal F_{\overline M}$ is a
measurable function class whose elements are bounded in absolute value by
$\overline M$, then
\begin{align*}
&\mathbb{P}\Bigg(
\sup_{f \in \mathcal{F}_{\overline M}} \Bigg| \frac{1}{n} \sum_{t=1}^n f(Z_t) - \mathbb{E}[f(Z_1)] \Bigg| > \varepsilon
\Bigg) \notag \\
&\qquad \leq 16 \, \mathbb{E}\Bigl[ \mathcal{N}\Bigl( \varepsilon/8, \mathcal{F}_{\overline{M}}, d_{\{\tilde{Z}_i\}_{i \in H},1} \Bigr) \Bigr] 
\exp\left( -\frac{\xi_n  \varepsilon^2}{128 \overline{M}^2} \right) + 2 \xi_n \beta_{a_n}\\
&\qquad \leq 16   \mathcal{N}\Bigl( \varepsilon/8, \mathcal{F}_{\overline{M}}, \|\cdot\|_{\infty} \Bigr)  
\exp\left( -\frac{\xi_n  \varepsilon^2}{128 \overline{M}^2} \right) + 2 \xi_n \beta_{a_n},
\end{align*}
where the empirical semimetric is
\[
d_{\{x_1,\dots,x_n\},1}(f,g) := \frac{1}{n} \sum_{i=1}^n |f(x_i) - g(x_i)|.
\]
\end{lemma}
\begin{lemma}[Theorem 5 of \cite{feng2024deep}]\label{lem:th5-of-feng}
Suppose that $f \in \mathcal{H}^{\varsigma}([0,1]^d,B)$. For any
$\epsilon \in (0,1)$, there exists a ReLU DNN $\psi$ with
depth
$
D \leq O\!\left(\log(1/\epsilon)\right),
$
size
$
S \leq O\!\left(\epsilon^{-d/\varsigma}\log(1/\epsilon)\right),
$
and parameters bounded in absolute value by
$
O\!\left(B\epsilon^{-d/\varsigma}\right),
$
such that
\[
\lVert f-\psi\rVert_{\infty} \leq B\epsilon.
\]
\end{lemma}

\begin{theorem}[Girsanov's theorem]
\label{thm:girsanov}
Let $(B_t)_{0\le t\le T}$ be a $d_{\mathcal A}$-dimensional Brownian
motion on a filtered probability space
$(\Omega,\mathcal F,(\mathcal F_t)_{0\le t\le T},\mathbb P)$.
Suppose the progressively measurable process $u_t\in\mathbb R^{d_{\mathcal A}}$
satisfies Novikov's condition
\begin{align*}
    \mathbb E_{\mathbb P}
    \left[
        \exp\left(
            \frac12\int_0^T \|u_t\|_2^2\,\mathrm dt
        \right)
    \right]
    <\infty.
\end{align*}
Then
\begin{align*}
    M_t
    :=
    \exp\left(
        \int_0^t u_r^\top \mathrm dB_r
        -
        \frac12\int_0^t \|u_r\|_2^2\,\mathrm dr
    \right),\qquad 0\le t\le T,
\end{align*}
is a martingale with $\mathbb E_{\mathbb P}M_t=1$.
Under the measure $\mathbb Q$ defined by
$\mathrm d\mathbb Q/\mathrm d\mathbb P=M_T$, the process
\begin{align*}
    \widetilde B_t
    :=
    B_t-\int_0^t u_r\,\mathrm dr,
    \qquad 0\le t\le T,
\end{align*}
is a $d_{\mathcal{A}}$-dimensional Brownian motion under $\mathbb Q$.
\end{theorem}

\section{Offline Pretraining Algorithm}\label{sec:offline_alg}
The algorithm estimates the reference SF drift by regressing the scaled
Brownian-bridge terminal displacement on the bridge inputs.
\begin{algorithm}[H]
\caption{Offline Pretraining of the SF Drift}
\label{alg:offline-pretraining}
\begin{algorithmic}[1]
\STATE \textbf{Input:} offline dataset
$\mathcal D_{\mathrm{off}}
=\{(s_i,a_i)\}_{i=1}^{n}$,
drift network $b_\phi$, Euler-step count $T$, minibatch size $B$

\FOR{each training step}
    \STATE Sample a minibatch
    $\mathcal I\subset\{1,\ldots,n \}$,
    $|\mathcal I|=B$.
    \FOR{$i\in\mathcal I$}
        \STATE Sample
        $J_i\sim\mathrm{Unif}\{0,\ldots,T-1\}$ and
        $\varepsilon_i\sim\mathcal N(0,I_{d_{\mathcal A}})$,
        and set $t_i=J_i/T$.
        \STATE Construct
        \[
        y_i=t_i a_i+\sqrt{t_i(1-t_i)}\,\varepsilon_i,
        \qquad
        u_i=\frac{a_i-y_i}{1-t_i}.
        \]
    \ENDFOR
    \STATE Update $\phi$ using
    \[
    \frac{1}{B}\sum_{i\in\mathcal I}
    \left\|
    b_\phi(s_i,y_i,t_i)-u_i
    \right\|_2^2.
    \]
\ENDFOR
\STATE \textbf{Output:} pretrained drift
$b_{\phi_k}\leftarrow b_\phi$.
\end{algorithmic}
\end{algorithm}

\bibliographystyle{spbasic}  %
\bibliography{ref_bib}   %

\end{document}